\documentclass[lettersize,journal]{IEEEtran}
\usepackage{amsmath,amsthm,amssymb,mathtools}
\usepackage{bm}
\usepackage{acronym}
\usepackage{graphicx}
\usepackage{color, colortbl} 
\usepackage[dvipsnames,svgnames,x11names]{xcolor}
\usepackage{algpseudocode,algorithm}
\usepackage{tabularx, boldline, multirow}
\usepackage{enumerate}
\usepackage[font=footnotesize]{caption,subcaption}
\usepackage{cite}
\usepackage{url,comment}
\usepackage[normalem]{ulem}
\usepackage[shortlabels]{enumitem}
\usepackage{balance}
\usepackage{etoolbox}
\usepackage{setspace}
\usepackage{pgfplots}
\usepackage{tikz}
\usetikzlibrary{fit,calc,bending,arrows.meta}

\makeatletter
\patchcmd{\@makecaption}
  {\scshape}
  {}
  {}
  {}
\makeatother
\acrodef{3d}[3D]{three-dimensional}
\acrodef{bs}[BS]{base station}
\acrodef{bp}[BP]{belief propagation}
\acrodef{d}[D]{dimensional}
\acrodef{ue}[UE]{user equipment}
\acrodef{los}[LOS]{line-of-sight}
\acrodef{aoa}[AoA]{angle-of-arrival}
\acrodef{aod}[AoD]{angle-of-departure}
\acrodef{cae}[CAE]{cumulative absolute error}
\acrodef{toa}[ToA]{time-of-arrival}
\acrodef{tdoa}[TDoA]{time-difference-of-arrival}
\acrodef{rfs}[RFS]{random finite set}
\acrodefplural{rfs}[RFSs]{random finite sets}
\acrodef{ris}[RIS]{reconfigurable intelligent surface}
\acrodefplural{ris}[RISs]{reconfigurable intelligent surfaces}
\acrodef{tx}[Tx]{transmitter}
\acrodef{rx}[Rx]{receiver}
\acrodefplural{rx}[Rxs]{receivers}
\acrodef{crb}[CRB]{Cram\'er-Rao lower bounds}
\acrodef{rss}[RSS]{received signal strength}
\acrodef{los}[LOS]{line-of-sight}
\acrodef{nlos}[NLOS]{non line-of-sight}
\acrodef{dft}[DFT]{discrete Fourier transform}
\acrodef{fft}[FFT]{fast Fourier transform}
\acrodef{fov}[FOV]{field-of-view}
\acrodef{fim}[FIM]{Fisher information matrix}
\acrodef{upa}[UPA]{uniform planar array}
\acrodefplural{upa}[UPAs]{uniform planar arrays}
\acrodef{peb}[PEB]{position error bound}
\acrodef{snr}[SNR]{signal-to-noise ratio}
\acrodef{sre}[SRE]{smart radio environment}
\acrodefplural{sre}[SRE]{smart radio environments}
\acrodef{mimo}[MIMO]{multiple-input  multiple-output}
\acrodef{mtt}[MTT]{multiple-target tracking}
\acrodef{rfid}[RFID]{radio-frequency identification}
\acrodef{siso}[SISO]{single-input single-output}
\acrodef{miso}[MISO]{multiple-input single-output}
\acrodef{ici}[ICI]{inter-carrier interference}
\acrodef{iid}[iid]{independent and identically distributed}
\acrodef{ppp}[PPP]{Poisson point process}
\acrodef{phd}[PHD]{probability hypothesis density}
\acrodef{cphd}[CPHD]{cardinalized probability hypothesis density}
\acrodef{mb}[MB]{multi-Bernoulli}
\acrodef{mbm}[MBM]{multi-Bernoulli mixture}
\acrodef{pmb}[PMB]{Poisson multi-Bernoulli}
\acrodef{pmbm}[PMBM]{Poisson multi-Bernoulli mixture}
\acrodef{lmb}[LMB]{labeled multi-Bernoulli}
\acrodef{glmb}[$\delta$-GLMB]{$\delta$-generalized labeled multi-Bernoulli}
\acrodef{ml}[ML]{maximum likelihood}
\acrodef{kld}[KLD]{Kullback-Leibler divergence}
\acrodef{cdf}[CDF]{cumulative distribution function}
\acrodef{ofdm}[OFDM]{orthogonal frequency-division multiplexing}
\acrodef{pmf}[PMF]{probability mass function}
\acrodef{pdf}[PDF]{probability density function}
\acrodef{qos}[QoS]{Quality of Service}
\acrodef{sp}[SP]{scattering point}
\acrodefplural{sp}[SPs]{scattering points}
\acrodef{ula}[ULA]{uniform linear array}
\acrodef{va}[VA]{virtual anchor}
\acrodefplural{va}[VAs]{virtual anchors}
\acrodef{slam}[SLAM]{simultaneous localization and mapping}
\acrodef{slat}[SLAT]{simultaneous localization and tracking}
\acrodef{gospa}[GOSPA]{generalized optimal subpattern assignment}
\acrodef{rmse}[RMSE]{root mean square error}

\newcommand{\appropto}{\mathrel{\vcenter{
  \offinterlineskip\halign{\hfil$##$\cr
    \propto\cr\noalign{\kern2pt}\sim\cr\noalign{\kern-2pt}}}}}

\usepackage{xfp}
\newcommand{\beginsupplement}{%
        \setcounter{table}{0}
        \renewcommand{\thetable}{S\arabic{table}}%
        \setcounter{figure}{0}
        \renewcommand{\thefigure}{S\arabic{figure}}%
       \setcounter{equation}{0} \def\theequation{S\arabic{equation}}
     }
\newtheorem{lemma}{Lemma}
\newtheorem{prop}{Proposition}
\newtheorem{rem}{Remark}
\newtheorem{define}{Definition}
\newtheorem{theorem}{Theorem}
\newtheorem{corollary}{Corollary}
\newtheorem{example}{Example}

\usepackage{framed} 
\usepackage{multicol} 
\usepackage{nomencl} 
\makenomenclature
\usepackage{cuted}

\renewcommand*\nompreamble{\begin{multicols}{3}}
\renewcommand*\nompostamble{\end{multicols}}
\newcolumntype{P}[1]{>{\centering\arraybackslash}m{#1}}
\newcolumntype{Q}[1]{>{\arraybackslash}m{#1}}
\nomenclature{$k$}{time index}
\nomenclature{$p$}{particle index}
\nomenclature{$i$}{target index}
\nomenclature{$j$}{measurement index}
\nomenclature{$\mathbf{s}_k$}{vehicle state at time $k$}
\nomenclature{$\mathbf{z}_k$}{measurement at time $k$}
\nomenclature{$\mathbf{x}$}{target location}
\nomenclature{$m$}{target type}
\nomenclature{$f_k(\mathbf{s}_k)$}{pdf of vehicle state}
\pgfplotsset{compat=1.18}
\begin{document}
\title{Set-Type Belief Propagation with Applications to Poisson Multi-Bernoulli SLAM}
\author{
    \IEEEauthorblockN{Hyowon Kim, \IEEEmembership{Member, IEEE}, {\'{A}}ngel F. Garc{\'{i}}a-Fern{\'{a}}ndez, 
    Yu Ge, \IEEEmembership{Student Member, IEEE},\\
    Yuxuan Xia, \IEEEmembership{Member, IEEE},
    Lennart Svensson, \IEEEmembership{Senior Member, IEEE},
    and Henk Wymeersch, 
    \IEEEmembership{Fellow, IEEE}
    }

    \thanks{H. Kim is with the Department of Electronics Engineering, Chungnam National University, 34134 Daejeon, South Korea (email: hyowon.kim@cnu.ac.kr).}
    \thanks{{\'{A}}. F. Garc{\'{i}}a-Fern{\'{a}}ndez is with the Department of Electrical Engineering and Electronics, University of Liverpool, Liverpool L69 3GJ, United Kingdom, and also with the ARIES Research Centre, Universidad Antonio de Nebrija,  Madrid, Spain (e-mail: angel.garcia-fernandez@liverpool.ac.uk).}    
    \thanks{Y. Ge, L. Svensson, and H. Wymeersch are with the Department of Electrical Engineering, Chalmers University of Technology, 412 58 Gothenburg, Sweden (email: 
    \{yuge; lennart.svensson; and henkw\}@chalmers.se).}
    \thanks{Y. Xia is with the Department of Electrical Engineering, Link\"{o}ping University, 581 83 Link\"{o}ping, Sweden (email: yuxuan.xia@liu.se).}
    
}

\maketitle
\begin{abstract}
    \Ac{bp} is a useful probabilistic inference algorithm for efficiently computing approximate marginal probability densities of random variables.
    However, in its standard form, \ac{bp} is only applicable to the vector-type random variables with a fixed and known number of vector elements, while certain applications rely on 
    \acp{rfs} with an unknown number of vector elements. 
    In this paper, we develop \ac{bp} rules for factor graphs defined on sequences of 
    \acp{rfs} where each 
    \ac{rfs} has an unknown number of elements, with the intention of 
    deriving novel
    inference methods for \acp{rfs}.
    Furthermore, we show that vector-type BP is a special case of set-type BP, where each \ac{rfs} follows the Bernoulli process.
    To demonstrate the validity of developed set-type \ac{bp}, we apply it to the \ac{pmb} filter
    for \ac{slam}, which naturally 
    leads to a set-type \ac{bp} \ac{pmb}-\ac{slam} method, which is analogous to a vector type \ac{slam} method, subject to minor modifications.
\end{abstract}

\acresetall
\IEEEpeerreviewmaketitle

\section{Introduction}
    \Ac{bp} is a powerful statistical inference algorithm, applicable to different fields including signal processing, robotics, autonomous vehicles, image processing, and artificial intelligence.
    In particular, \ac{bp} is attractive for computing probability densities in Bayesian networks~\cite{Frank_SPA_TIT2001,Hans_SPA_SPM2004,Yedidia2005BetheFree,Henk_IRD_2007}, where one can efficiently compute the marginal probability densities of random variables of interest given the observed data.
    \Ac{bp}  
    has been actively applied in localization~\cite{Ihler_CL_2005,Henk_CL_Proc2009}, mapping~\cite{Callum_BPMapping_IROS2022}, \ac{mtt}~\cite{Florian_MTT_TSP2017,Florian_BP-MTT_Proc2018}, \ac{slam}~\cite{HenkGlobecom2018,HyowonAsilomar2018,Erik_BPSLAM_TWC2019,Erik_AOABPSLAM_ICC2019,Erik_BPSLAM_2022}, and \ac{slat}~\cite{Florian_NCO_TSIP2015}.
    Mapping and \ac{mtt} respectively involve mapping landmarks and tracking targets, based on the noisy sensor measurements.
    Simultaneously estimating the unknown sensor state as well as landmarks and targets leads to \ac{slam} and \ac{slat}, respectively.
    The main challenge in this research stems from the data association uncertainty between the unknown number of targets and imperfect measurements by missed detections
    and clutter. 
    To effectively
    address this challenge, several approaches have been developed as follows. 
     Classical \ac{slam} methods~\cite{FastSLAM,GraphSLAM,Durrant2006SLAM1,Durrant2006SLAM2,Randall_EKFSLAM} consider the data associations outside the Bayesian filter, while the filter operates on vector random variables. 
     With the introduction of 
     \acp{rfs} in~\cite{mahler_book_2007,Mahler_book2014}, a theoretically sound tool for modeling the unknown number of targets and for handling the data association between targets and measurements became available.
    By modeling targets with vector-type random variables, the marginal densities for targets are efficiently computed by \ac{bp}, relying on several ad-hoc modifications~\cite{Erik_BPSLAM_TWC2019,Erik_AOABPSLAM_ICC2019,Erik_BPSLAM_2022}.
    Between them, the \ac{rfs}- and \ac{bp}-based methods have rigorously handled the main challenge of \ac{mtt} and \ac{slam} by starting with the formulation of joint posterior density of targets and data association.
    There are several \ac{rfs}-based methods for \ac{mtt} and \ac{slam}. 
    The \ac{phd} filter \cite{mahler_AES_2003_PHD,mahler_book_2007,Mahler_book2014,Hyowon_TWC2020} propagates the first-order statistical moment of the RFS, which is its intensity. The \ac{cphd} filter propagates the intensity and the cardinality distribution \cite{mahler_book_2007,Mahler_book2014,Angle_CPHD_TSP2015}.
    A more accurate solution can be obtained with
    RFS filters based on multi-object conjugate priors, which are inherently closed under prediction and update~\cite{Jason_PMB_TAES2015,Angel_PMBM_TAES2018,Hyowon_MPMB_TVT2022,Markus_PMB_TIV2019}.
    The \ac{pmbm}~\cite{Angel_PMBM_TAES2018,Hyowon_MPMB_TVT2022} adopts the Poisson birth model and \ac{mbm}.
    For standard multi-object models with Poisson birth~\cite{mahler_book_2007}, the Bayesian optimal solutions to \ac{mtt} and \ac{slam} are given by the \ac{pmbm} conjugate prior~\cite{Jason_PMB_TAES2015,Hyowon_MPMB_TVT2022}.
    With the \ac{mb} birth model instead of Poisson, the conjugate prior is an \ac{mbm}~\cite{Angel_PMBM_TAES2018}. The \ac{mbm} filter corresponds to a \ac{pmbm} filter where the Poisson intensity equals to zero, and the birth Bernoulli components are added into the \ac{mbm} in the prediction step. If we expand the number of global hypotheses in the \ac{mbm} filter such that each Bernoulli has deterministic target existence, we obtain the MBM$_{01}$ filter~\cite{Angel_PMBM_TAES2018}. Both MBM and MBM$_{01}$ filters can be labeled without changing the filtering recursion. The MBM$_{01}$ filter recursion is analogous to the \ac{glmb} filtering recursion~\cite{Diluka_GLMB-SLAM_ICCAIS2018,Diluka_GLMB-SLAM_Sensor2019,Deusch2015,DeuschRD:2015,vo2013LMB}.
    Some \ac{rfs}-based methods rely on  \ac{bp} to compute marginal association probabilities~\cite{Jason_PMB_TAES2015}, which converts the \ac{pmbm}~\cite{Angel_PMBM_TAES2018,Hyowon_MPMB_TVT2022} to a \ac{pmb}~\cite{Jason_PMB_TAES2015,Yu_EK-PMBM_JSAC2021,Markus_PMB_TIV2019,Hyowon_MPMB_TVT2022} after every update, by marginalizing out the association variables.

    In addition to \ac{rfs}-based methods, vector-type \ac{bp} methods have also been applied to tracking and mapping problems. 
    In \cite{Florian_MTT_TSP2017,Florian_BP-MTT_Proc2018}, the \ac{mtt} problem is tackled by running vector-type \ac{bp} on the factor graph representation of the joint distribution of target states and data association variables.
    Here targets are modeled by augmented vectors including the binary target existence indicators, and message passing between targets and data association variables is adopted from~\cite{Florian_BP-MTT_Proc2018}. 
    The \ac{slam} problems with radio signal propagation are formulated in \cite{HenkGlobecom2018,HyowonAsilomar2018,Erik_BPSLAM_TWC2019,Erik_AOABPSLAM_ICC2019,Erik_BPSLAM_2022}.
    The authors of \cite{HenkGlobecom2018,HyowonAsilomar2018} formulate the factor graph and introduce an efficient message passing scheduling, among the unknown sensor variables consisting of position, orientation, and clock bias, and multiple-types of targets are further studied in \cite{HyowonAsilomar2018}.
    With the augmented target vectors and the factorized joint distribution of targets and data association variables from \cite{Florian_BP-MTT_Proc2018}, the message passing methods for joint \ac{slam} and data association are developed \cite{Erik_BPSLAM_TWC2019,Erik_AOABPSLAM_ICC2019,Erik_BPSLAM_2022}. 
    These vector-type \ac{bp}-based approaches \cite{Erik_BPSLAM_TWC2019,Erik_AOABPSLAM_ICC2019,Erik_BPSLAM_2022} can be explained as the track-oriented \ac{pmb} filter~\cite{Hyowon_MPMB_TVT2022} and yield competitive performance.
    To handle undetected targets and newly detected targets, the auxiliary \ac{phd} is adopted in~\cite{Erik_BPSLAM_TWC2019,Erik_AOABPSLAM_ICC2019,Erik_BPSLAM_2022}, which was initially considered in~\cite{Paul_TrackIniPHD_2011} for vector-type \ac{mtt} and in \cite{Jason_TrackIniPMB_2012} for set-type \ac{mtt}.
    The propagation of the auxiliary PHD filter is explicitly introduced in~\cite{Erik_AOABPSLAM_ICC2019}, but external to the factor graph representation.
    Therefore, undetected targets and their connections to newly detected targets are not represented in the formulated factor graph.
    The modeling of undetected targets in vector-based methods is treated in~\cite{Stefano_Undetected_TAES2014} but the distinction between eventually and never detected targets is required.
    Moreover, it is not explicitly revealed how to apply \ac{bp} in the multiple hypotheses tracking formalism.
    Hence, while \ac{bp}-based methods have benefits in algorithm implementation, they require certain ad-hoc modifications.
    
    
    In this paper, we aim to bridge the gap between \ac{rfs} theory and \ac{bp}, by developing a novel \ac{bp} algorithm, running on factor graphs defined on the sequence of \acp{rfs}. Like conventional \ac{bp}, this opens the door to automated inference over factor graphs, once the \ac{rfs} density is factorized. This approach can avoid the need for heuristics and approximations.
    Related work has been done in \cite{Yuxuan_BPTrajectory_2022}, where \ac{bp} is applied to only the update step, without a systematic treatment to the entire filtering recursion, except for the modeling of undetected targets and newly detected targets.
    The goal of this paper is thus to formalize set-type \ac{bp} and the corresponding factor graphs from the \ac{rfs} densities. From the newly proposed set-type \ac{bp}, we derive the \ac{pmb} filter using the developed set-type \ac{bp}.
    The contributions of this paper are summarized as follows:
\begin{itemize}
    \item \textbf{The specification of set-type \ac{bp}:} We derive set-type \ac{bp} and demonstrate that vector-type \ac{bp} is a special case of set-type \ac{bp}.
    We also show that as in vector-type \ac{bp}, the interior stationary points of the constrained Bethe free energy are set-type \ac{bp} fixed points.
    \item \textbf{The introduction of novel factors for set-type BP:} We devise a partition and merging factor, which partitions a single set into multiple sets and merges multiple sets into a single set, useful for handling sets with unknown cardinalities.
    We also propose a conversion factor for sets augmented with auxiliary vectors such as unique marks.
    \item \textbf{Derivation of \ac{pmb} and \ac{mb} filters with set-type BP for the related problems of mapping, \ac{mtt}, \ac{slam}, \ac{slat}:}
    With the developed set-type \ac{bp}, we revisit the \ac{pmb}- and \ac{mb}-\ac{slam} filters. We first introduce auxiliary variables to factorize  their joint \ac{slam} and data association distribution,
    formulating a factor graph from the factorized density, and running set-type \ac{bp} on the factor graph.
    This work can also lead to set-type \ac{bp} \ac{pmb}-mapping, \ac{mtt}, and \ac{slat} filters.
    The resulting method bears a close resemblance to the PMB-SLAM filter~\cite{Hyowon_MPMB_TVT2022} that computes the updated sensor state density by ignoring the sensor state information generated from newly detected targets and selecting the most likely association determined by the Murty’s algorithm~\cite{murthy1968algorithm}.
    
    \item \textbf{Relation to vector-type BP-SLAM:} We clearly show the connections between the proposed set-type \ac{bp} \ac{pmb}-\ac{slam} and vector-type \ac{bp}-\ac{slam} filters~\cite{Erik_BPSLAM_TWC2019,Erik_AOABPSLAM_ICC2019,Erik_BPSLAM_2022}.
    The methods are algorithmically equivalent, except for minor details described in Section \ref{sec:Connections}, even though the proposed filter is derived directly from the \ac{rfs} extension and developed set-type \ac{bp} framework.
    However, 
    vector-type \ac{bp}-\ac{slam} filters~\cite{Erik_BPSLAM_TWC2019,Erik_AOABPSLAM_ICC2019,Erik_BPSLAM_2022} requires certain heuristics as part of the algorithm development, which are avoided in the proposed set-type BP PMB-SLAM filter.
    The simulation results show that the proposed set-type \ac{bp} \ac{pmb}-\ac{slam} filter outperforms the vector-type \ac{bp}-\ac{slam} filter~\cite{Erik_BPSLAM_TWC2019,Erik_AOABPSLAM_ICC2019,Erik_BPSLAM_2022}, in scenarios with informative \ac{ppp} birth.
\end{itemize}
    This paper is organized as follows. Section~\ref{sec:Background} provides the background of vector-type \ac{bp} and \acp{rfs}.
    In Section~\ref{sec:FG_SetBP}, the set-type \ac{bp} rules and set-type factor nodes are proposed.
    Proposed set-type \ac{bp} is applied to the \ac{pmb} filter, and the connections between set-type and vector-type \ac{bp}-\ac{slam} filters are analyzed in Section~\ref{sec:SetBP-SLAM}. The numerical results and discussions are reported in Section~\ref{sec:NumericalResults}, and conclusions are drawn in Section~\ref{sec:Conclusions}.


\allowdisplaybreaks
\section{Background} \label{sec:Background}
    In this section, we review factor graphs and belief propagation, and we recall the RFS approaches.

\emph{Notation:} 
    Scalars are denoted by italic font, vectors and matrices are respectively indicated by bold lowercase and uppercase letters, and sets are displayed in calligraphic font, e.g., $x$, $\mathbf{x}$, $\mathbf{X}$, and $\mathcal{X}$.
    The set of finite subsets of a space $\mathbb{R}$ is denoted by $\mathcal{F}(\mathbb{R})$.
    The vector consisting of a sequence of vectors $\mathbf{x}^i$ is denoted by $\underline{\mathbf{x}}$, and the sequence of multiple sets $\mathcal{X}$ is denoted by $\underline{\mathcal{X}}$.
    
\subsection{Vector-Type Factor Graph and Belief Propagation}
\subsubsection{Joint Density Factorization and Factor Graph}
    
    Let $\mathbf{x}^i \in \mathbb{R}^{n_{\mathbf{x}}}$ denote a single state vector and $\underline{\mathbf{x}}=[(\mathbf{x}^1)^\top,\dots,(\mathbf{x}^N)^\top]^\top$ denote augment single state vectors representing $N$ states.
    Then, we denote a joint 
    probability
    density of the hidden variables $\underline{\mathbf{x}}$
    by $f(\underline{\mathbf{x}})$,
    which can be factorized as~\cite{Frank_SPA_TIT2001,Henk_IRD_2007} 
\begin{align}
    f(\underline{\mathbf{x}}) \propto \prod_{a} f_a(\underline{\mathbf{x}}^{a}),
    \label{eq:vectorFac}
\end{align}
    where $f_a(\underline{\mathbf{x}}^{a})$ denotes a 
    nonnegative function,
    and $\underline{\mathbf{x}}^{a}$ denotes the argument vector of the function $f_a(\cdot)$. 

    The factorized functions in~\eqref{eq:vectorFac}
    and corresponding argument vectors can be represented by a factor graph.
    The factor graph with the general graphical model~\cite{Frank_SPA_TIT2001} consists of nodes for the different factors and variables, illustrated by squares for the factors, $f_a(\cdot)$, and circles for the variables, $\mathbf{x}^i$, respectively, with edge connections between the factors and their argument variables.
    
    
\begin{example}
    Suppose we have a 
    probability
    density $f(\underline{\mathbf{x}})$, such that $\underline{\mathbf{x}} = [(\mathbf{x}^1)^\top,(\mathbf{x}^2)^\top]^\top$, $\underline{\mathbf{x}}^A = \mathbf{x}^1$, $\underline{\mathbf{x}}^B=\mathbf{x}^2$, $\underline{\mathbf{x}}^C=[(\mathbf{x}^1)^\top,(\mathbf{x}^2)^\top]^\top$, which can be factorized as
\begin{align}
    f(\underline{\mathbf{x}})
    \propto f_A(\mathbf{x}^1)f_B(\mathbf{x}^2)f_C(\mathbf{x}^1,\mathbf{x}^2).
    \label{eq:vector-factorized}
\end{align}
    The corresponding factor graph is illustrated in Fig.~\ref{Fig:VecFG}.
\end{example}

\begin{figure}[h!]
\begin{centering}
	\includegraphics[width=1\columnwidth]{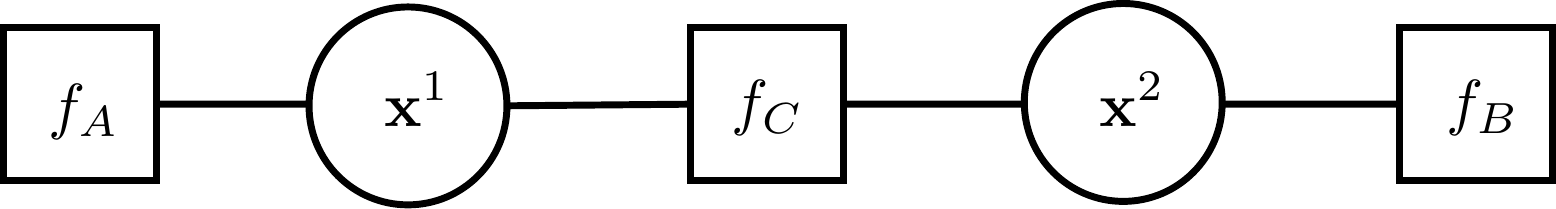}
	\caption{Factor graph representation of the factorized vector density~\eqref{eq:vector-factorized}.}
	\label{Fig:VecFG}
\end{centering}
\end{figure}

\subsubsection{Belief Propagation}

    \Ac{bp} is an efficient approach for estimating the marginal densities of the variables from a joint probability density function~\cite{Frank_SPA_TIT2001}.
    By running \ac{bp} on the factor graph, we compute the messages passed between factors and variables for all links.
    We denote the propagated messages from factor $a$ to variable $i$ and from variable $i$ to factor $a$ by $\mathtt{m}_{a \to i}(\mathbf{x}^i)$ and $\mathtt{n}_{i \to a}(\mathbf{x}^i)$, respectively.
    The messages are updated with the following rules:
\begin{align}
    \mathtt{m}_{a \to i}(\mathbf{x}^i) &= \int f_a(\underline{\mathbf{x}}^a) \prod_{j\in \mathcal{N}(a) \setminus\{i\}}  \mathtt{n}_{j \to a}(\mathbf{x}^j) \mathrm{d} \mathbf{x}^{\sim i}, \\
    \mathtt{n}_{i \to a}(\mathbf{x}^i)&=\prod_{b \in \mathcal{M}(i)\setminus \{a\}} \mathtt{m}_{b \to i}(\mathbf{x}^i),
\end{align}
    where $\mathcal{N}(a)$ denotes the set of indices $i$ of neighboring vectors linked to the factor $f_a(\underline{\mathbf{x}}^{a})$, such that $i \in \mathcal{N}(a)$ if and only if $\mathbf{x}^i$ is an argument of $f_a(\cdot)$, $\mathcal{M}(i)$ denotes the set of indices $a$ of neighboring factors linked to the variable $\mathbf{x}^i$, such that $a \in \mathcal{M}(i)$ if and only if $\mathbf{x}^i$ is an argument of $f_a(\cdot)$, and $\int \dots \mathrm{d}\mathbf{x}^{\sim  i}$ 
    denotes integration with respect to all vectors $\mathbf{x}^{j}$ except $\mathbf{x}^{i}$.
    The messages can be used to compute beliefs that approximate the marginalized posterior densities. 
    The beliefs at variable $i$ and factor $a$ are denoted $\mathtt{b}(\mathbf{x}^i)$ and $\mathtt{b}(\underline{\mathbf{x}}^a)$, respectively, and are updated using the following rules:
\begin{align}
    \mathtt{b}(\mathbf{x}^i) & \propto \prod_{a \in \mathcal{M}(i)} \mathtt{m}_{a \to i}(\mathbf{x}^i),\label{eq:vecvBel}\\
    \mathtt{b}(\underline{\mathbf{x}}^a) & \propto f_a(\underline{\mathbf{x}}^a)\prod_{i \in \mathcal{N}(a)} \mathtt{n}_{i \to a}(\mathbf{x}^i).\label{eq:vecfBel} 
\end{align}

\subsection{Random Finite Sets} \label{sec:RFSdefs}

\subsubsection{Set-Variables, Density, and Integral}
    Let us denote an \ac{rfs} by 
    $\mathcal{X}=\{\mathbf{x}^1,\dots,\mathbf{x}^{n}\}\in \mathcal{F}(\mathbb{R}^{n_\mathbf{x}})$, where both vector $\mathbf{x}^i$, $i\in \{1,\dots,n\}$ and cardinality ${n}=\lvert \mathcal{X} \rvert $ are random.
    We define a set-density 
    $f({\mathcal{X}})$
    as~\cite{Jason_PMB_TAES2015}
\begin{align}
    f({\mathcal{X}}) =
    p(n)\sum_\pi f^n({\mathbf{x}}^{\pi(1)},\dots,{\mathbf{x}}^{\pi(n)}),
    \label{eq:setDensity}
\end{align}
    in which $p(n)=\mathrm{Pr}(\lvert {\mathcal{X}} \rvert = n)$ denotes the probability mass function of the set cardinality, $\pi$ denotes a possible permutation of the set $\mathcal{N} = \{1,\dots,n\}$ with $\pi(i) \in \mathcal{N}$, and $f^n(\cdot)$ is the joint probability density function of the vector with $n$ elements, evaluated for permutation $\pi$.
    Given the set-valued function $g(\cdot)$, we define the set integral as~\cite[eq. (3.11)]{Mahler_book2014}
\begin{align}
    \int g({\mathcal{X}}) \delta {\mathcal{X}} = g(\emptyset) + \sum_{n=1}^\infty \dfrac{1}{n !} 
    \int g(\{{\mathbf{x}}^1,\dots, {\mathbf{x}}^n\}) \mathrm{d}{\mathbf{x}}^1 \dots \mathrm{d}{\mathbf{x}}^n. \label{eq:SetIntegral}
\end{align}

\subsubsection{Poisson, Bernoulli, and PMB Densities}\label{sec:PMB_PMBM}
We will follow the definitions from~\cite{mahler_book_2007}. 
Suppose we have a set $\mathcal{X}^\mathrm{U}$ that follows the Poisson process. The density is given by~\cite{Jason_PMB_TAES2015}
\begin{align} \label{eq:PoissonDensity}
    f^{\mathrm{PPP}}(\mathcal{X}^\mathrm{U}) = e^{-\int \lambda(\mathbf{x})\mathrm{d}\mathbf{x}}\prod_{\mathbf{x}\in\mathcal{X}^\mathrm{U}}\lambda(\mathbf{x}),
\end{align}    
    where $\lambda(\mathbf{x})$ denotes the Poisson intensity function.
    A Bernoulli density  $f^{\mathrm{B}}(\mathcal{X})$ is given by
\begin{align} \label{eq:Bernoulli}
    f^{\mathrm{B}}(\mathcal{X}) =
    \begin{cases}
    1-r,& \mathcal{X}=\emptyset\\
    r\,f(\mathbf{x}), & \mathcal{X}=\{\mathbf{x}\}\\
    0, & \lvert \mathcal{X} \rvert > 1,
    \end{cases}
\end{align}
    where $f(\mathbf{x})$ and $r\in [0,1]$ denote the spatial density and the existence probability, respectively.
    An \ac{mb} density with $n$ Bernoulli components is given by 
\begin{align} \label{eq:MBdensity}
    f^{\mathrm{MB}}(\mathcal{X}^\mathrm{D})
    &= \sum_{\uplus_{i=1}^{n}\mathcal{X}^{i}=\mathcal{X}^\mathrm{D}}\,\,\,\prod_{i=1}^{n} f^{i}(\mathcal{X}^{i}),
\end{align}
    where $f^{i}(\mathcal{X}^{i})$ is a Bernoulli density, and $\uplus$ stands for disjoint set union~\cite[pp.~24]{armstrong2013basic}.

    We are now ready to introduce a \ac{pmb} density, defined as follows.
    Suppose we have two independent \acp{rfs} $\mathcal{X}^{\mathrm{U}}$ and $\mathcal{X}^{\mathrm{D}}$ such that $\mathcal{X} = \mathcal{X}^{\mathrm{U}} \uplus \mathcal{X}^{\mathrm{D}}$, where $\mathcal{X}^{\mathrm{D}}$ 
    follows a \ac{mb} process and $\mathcal{X}^{\mathrm{U}}$ 
    follows a \ac{ppp}.
    Using the convolution formula for independent RFSs~\cite{mahler_book_2007}, a \ac{pmb} density $f(\mathcal{X})$ is
  
\begin{align}\label{eq:PMBdensity}
    f^{\mathrm{PMB}}(\mathcal{X}) \propto \sum_{\uplus_{i=1}^{n}\mathcal{X}^{i}\uplus \mathcal{X}^\mathrm{U} =\mathcal{X}} \,\,\prod_{\mathbf{x}\in \mathcal{X}^\mathrm{U}}\lambda(\mathbf{x}) \prod_{i=1}^{n}f^{i}(\mathcal{X}^i).
\end{align}

\subsubsection{Auxiliary Variables}\label{sec:AuxV}
    The derivation of set-type BP \ac{pmb} filters will require us to introduce auxiliary variables to remove the summation in~\eqref{eq:PMBdensity}. 
    In particular, we introduce
    $u \in \mathbb{U}$ in the \ac{pmb} density~\eqref{eq:PMBdensity},  
    where $\mathbb{U}=\{0,1,\dots,n\}$ ~\cite{Angel_Trajectory_TSP2020,Yuxuan_BPTrajectory_2022}. 
    We thus extend the single state space, such that $({u},\mathbf{x}) \in \mathbb{U} \times \mathbb{R}^{n_{\mathbf{x}}}$, and denote a set of target states with auxiliary variables by $\tilde{\mathcal{X}}
    \in \mathcal{F}(\mathbb{U} \times \mathbb{R}^{n_{\mathbf{x}}})$.
    The set with the auxiliary variable ${u}=0$ follows a \ac{ppp} and indicates that the targets have not been detected, denoted by $\tilde{\mathcal{X}}^\mathrm{U}=\{(u,\mathbf{x}) \in \tilde{\mathcal{X}}: u=0 \}$. Similarly, the set with ${u}=i$ follows a Bernoulli process and indicates that the single target has previously been detected, denoted by $\tilde{\mathcal{X}}^i=\{(u,\mathbf{x}) \in \tilde{\mathcal{X}}: u=i \}$.
    For the set of targets with auxiliary variables $\tilde{\mathcal{X}}$, the \ac{pmb} density is~\cite[Definition~1]{Angel_Trajectory_TSP2020}
\begin{align}
    \tilde{f}^{\mathrm{PMB}}(\tilde{\mathcal{X}})&=\tilde{f}^{\mathrm{U}}(\tilde{\mathcal{X}}^\mathrm{U})\prod_{i=1}^{n}\tilde{f}^{i}(\tilde{\mathcal{X}}^i).
    \label{eq:PMB_Aux}
\end{align}
    Here  
    $\tilde{f}^{\mathrm{U}}(\tilde{\mathcal{X}}^\mathrm{U})$ and $\tilde{f}^{i}(\tilde{\mathcal{X}}^i)$ are given by
\begin{align}
    \tilde{f}^{\mathrm{U}}(\tilde{\mathcal{X}}^\mathrm{U})&
    = \exp\left(-\int \lambda(\mathbf{x})\mathrm{d}\mathbf{x}\right) \prod_{(u,\mathbf{x})\in \tilde{\mathcal{X}}^\mathrm{U}}\delta_0[u]\lambda(\mathbf{x}),
    \label{eq:PMB_Aux_Poi}
    \\
    \tilde{f}^{i}(\tilde{\mathcal{X}}^i)&=
    \begin{cases}
        1-r^{i}, & \tilde{\mathcal{X}}^i=\emptyset\\
        r^{i} f^{i}(\mathbf{x})\delta_{i}[u], & \tilde{\mathcal{X}}^i=(u,\mathbf{x})\\
        0, & \text{otherwise}
    \end{cases},
    \label{eq:PMB_Aux_Ber}
\end{align}
    where $\delta_i[\cdot]$ denotes the Kronecker delta function.
    For notational simplicity, $\tilde{\cdot}$ on sets with auxiliary variables will be omitted, when possible.

\section{Factor Graph and Belief Propagation for Random Finite Sets}\label{sec:FG_SetBP}

    We describe the proposed set-type \ac{bp} update rules and special factors for \acp{rfs}. We reveal that set-type \ac{bp} is a generalization of standard vector-type \ac{bp} since a vector can be represented as a set with a single element $p(n=1)=1$ of~\eqref{eq:setDensity}. 
    
\subsection{Factor Graphs and BP over a Sequence of RFSs}
\label{sec:SetFG}

    
    Suppose we have $n$ \acp{rfs} ${\mathcal{X}}^1,\dots,{\mathcal{X}}^n$, with the joint density $f({\mathcal{X}}^1,\dots,{\mathcal{X}}^n)$.
    In the general formulation of set-type \ac{bp}, ${\mathcal{X}}^1,\dots,{\mathcal{X}}^n$ may or may not contain auxiliary variables, and the number of \acp{rfs} in the sequence, $n$, is known.

    
\begin{define} [Factorization of Set-Density and Factor Graph] \label{def:SetFG}
    Let us denote by $f_a(\cdot)$ the set-factor $a$, which is 
    a nonnegative function; 
    by $\mathcal{N}(a)$ the set of neighboring set-variable indices linked to the set-factor $a$; by $\mathcal{X}^i$ the set-variable $i$; by $\mathcal{M}(i)$ the set of neighboring set-factor indices linked to the set-variable $\mathcal{X}^i$; and by $\underline{\mathcal{X}}^a$ the arguments of the set-factor $a$, represented by the sequence of all \acp{rfs} $\mathcal{X}^i$ for $i\in \mathcal{N}(a)$.
    Suppose the joint density $f({\mathcal{X}}^1,\dots,{\mathcal{X}}^n)$ is factorized as follows:
\begin{align}
    f({\mathcal{X}}^1,\dots,{\mathcal{X}}^n) \propto 
    \prod_{a} f_a(\underline{\mathcal{X}}^a).
    \label{eq:SetFactor}
\end{align}
    The factorized density can then be represented by a factor graph, consisting of the set-variables $\mathcal{X}^i$ and set-factors $f_a(\underline{\mathcal{X}}^a)$, which can be represented by circles and squares, and edge connections between $\mathcal{X}^i$ and $f_a(\underline{\mathcal{X}}^a)$.
    It should be noted that $f_a(\underline{\mathcal{X}}^a)$ has adequate units for each cardinality of $\underline{\mathcal{X}}^a$ such that we can integrate \eqref{eq:SetFactor} using set integrals
    over ${\mathcal{X}}^1,\dots,{\mathcal{X}}^n$~\cite[Sec.~3.2.4]{Mahler_book2014}.
\end{define}

\begin{example}\label{ex:set-fact}
    Consider a set-density $f(\mathcal{X}^1,\mathcal{X}^2)$, such that $\underline{\mathcal{X}}^{A}=\mathcal{X}^{1}$, $\underline{\mathcal{X}}^{B}=\mathcal{X}^{2}$, and $\underline{\mathcal{X}}^{C}=(\mathcal{X}^{1}, \mathcal{X}^{2})$, which can be factorized as
\begin{align}
    f({\mathcal{X}}^1,{\mathcal{X}}^2) \propto 
    f_A(\mathcal{X}^{1})f_B(\mathcal{X}^{2})f_C(\mathcal{X}^{1}, \mathcal{X}^{2}).
    \label{eq:exSetfac}
\end{align}
    Using the set-variables and set-factors, the factor graph corresponding to~\eqref{eq:exSetfac} is 
    shown in Fig.~\ref{Fig:SetFG}.
\end{example}    
\begin{figure}[t!]
\begin{centering}
	\includegraphics[width=1\columnwidth]{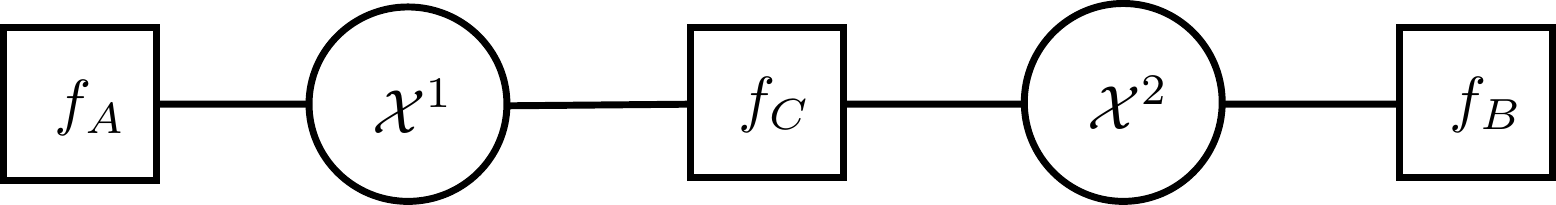}
	\caption{Factor graph representation of set-density of~\eqref{eq:exSetfac}.}
	\label{Fig:SetFG}
\end{centering}
\end{figure}
    
\begin{define}[Set-Type \ac{bp} Update Rules]

    The set-messages from the set-factor $a$ to the set-variable $i$ are denoted by 
    $\mathtt{m}_{a \to i}(\mathcal{X}^i)$, while those  from the set-variable $i$ to the set-factor $a$ are denoted by 
    $\mathtt{n}_{i \to a}(\mathcal{X}^i)$.
    The beliefs at the set-variable $i$ and set-factor $a$ are denoted by $\mathtt{b}(\mathcal{X}^i)$ and $\mathtt{b}(\underline{\mathcal{X}}^a)$, respectively.
    The set-messages are updated with the following rules:
\begin{align}
    \mathtt{m}_{a \to i}(\mathcal{X}^i)
    & = \int 
    f_a(\underline{\mathcal{X}}^a)
    \prod_{j \in \mathcal{N}(a) \setminus \{i\}} \mathtt{n}_{j \to a}(\mathcal{X}^j)
    \delta \mathcal{X}^{\sim i},\label{eq:setM_fv}\\
    \mathtt{n}_{i 
    \to a} (\mathcal{X}^i) & = \prod_{b \in \mathcal{M}(i) \setminus \{a\}} \mathtt{m}_{b \to i}(\mathcal{X}^i),\label{eq:setM_vf}
\end{align}
    where $\setminus$ denotes the set difference, and $\int \dots \delta \mathcal{X}^{\sim  i}$ 
    denotes integration with respect to all sets $\mathcal{X}^{j}$ except $\mathcal{X}^{i}$, i.e., with respect to all $\mathcal{X}^{j}$ for which $j \in \mathcal{N}(a) \setminus \{i\}$.
    The beliefs at the set-variable $i$ and set-factor $a$ are updated with the following rules:
\begin{align}
    \mathtt{b}(\mathcal{X}^i) &\propto \prod_{a \in \mathcal{M}(i)} \mathtt{m}_{a \to i}(\mathcal{X}^i)\label{eq:setVbp}\\
    \mathtt{b}(\underline{\mathcal{X}}^a)  &\propto 
    f_a (\underline{\mathcal{X}}^a)
    \prod_{i \in \mathcal{N}(a)} \prod_{b \in \mathcal{M}(i) \setminus \{a\}} \mathtt{m}_{b \to i}(\mathcal{X}^{j}).\label{eq:setFbp}
\end{align}
    
\end{define}
    The optimality of the set-type \ac{bp} update rules in Definition~\ref{def:SetFG} is described by Theorem~\ref{theo:DerSetBP} and~Corollary~\ref{coro:setBelief_nocycle}, provided next.
\begin{theorem}
    \label{theo:DerSetBP}
    The interior stationary points of the constrained Bethe free energy 
    are set-type \ac{bp} fixed points with positive set-beliefs and vice versa.
\end{theorem}
\begin{proof}
    See Appendix~\ref{app:Optimal_SetBP}.
\end{proof}

\begin{corollary}
    \label{coro:setBelief_nocycle}
    The set-beliefs obtained by running set-type \ac{bp} on a factor graph that has no cycles,  represent the exact marginal probability densities. 
\end{corollary}
\begin{proof}
   See Appendix~\ref{app:setBelief_nocycle}.
\end{proof}

\begin{example}
    Consider the factor graph in Fig.~\ref{Fig:SetFG}, representing the factorized density in Example~\ref{ex:set-fact}.
    Using set-type BP, the beliefs at set-variable 1 and factor $C$ are obtained as follows: $\mathtt{m}_{A \rightarrow 1}(\mathcal{X}^{1}) = f_A(\mathcal{X}^{1})$, 
    $\mathtt{m}_{B \rightarrow 2}(\mathcal{X}^{2}) = f_B(\mathcal{X}^{2})$,
    $\mathtt{n}_{1 \rightarrow C}(\mathcal{X}^{1}) = \mathtt{m}_{A \rightarrow 1}(\mathcal{X}^{1})$,
    $\mathtt{n}_{2 \rightarrow C}(\mathcal{X}^{2}) = \mathtt{m}_{B \rightarrow 2}(\mathcal{X}^{2})$,
\begin{subequations}
\begin{align}
    \mathtt{m}_{C \rightarrow 1}(\mathcal{X}^{1}) &= \int \mathtt{n}_{2 \rightarrow C}(\mathcal{X}^{2})
    f_C(\mathcal{X}^{1} , \mathcal{X}^{2}) \delta \mathcal{X}^{2},\label{eq:Mes_Xi3-X1}\\
    \mathtt{b}(\mathcal{X}^{1}) & \propto \mathtt{m}_{A \rightarrow 1}(\mathcal{X}^{1})\mathtt{m}_{C \rightarrow 1}(\mathcal{X}^{1})\label{eq:Bel_X1},\\
    \mathtt{b}(\mathcal{X}^C)&\propto f_C(\mathcal{X}^1,\mathcal{X}^2)
    \mathtt{m}_{A \rightarrow 1}(\mathcal{X}^1)\mathtt{m}_{B \rightarrow 2}(\mathcal{X}^2).
\end{align}
\end{subequations}
\end{example}

\begin{rem}[Vector-Type BP is a Special Case of Set-Type BP]
     Note that the vector-type \ac{bp} message passing rules and beliefs can be obtained from the set-type \ac{bp} expressions when considering sets whose cardinality is 1 with probability 1, i.e., $p(\lvert \mathcal{X}^i\rvert=1)=1,~\forall~i$.\footnote{The number of elements of the sets is set deterministically to 1. 
     For example, suppose we have $\mathcal{X}^1=\{\mathbf{x}^1\},\mathcal{X}^2=\{\mathbf{x}^2\},\mathcal{X}^3=\{\mathbf{x}^3\}$ in the \ac{rfs} representation and $\mathbf{x}^1,\mathbf{x}^2,\mathbf{x}^3$ in the vector representation. Then, both representations are equivalent.}
     Note that this property implies that we can directly apply \ac{bp} to joint set and vector densities using the set- and vector-integrals for set- and vector-variables, respectively. 
     A vector-type factor graph~\eqref{eq:vectorFac} can be written as a set-type factor graph in~\eqref{eq:SetFactor} with the property that each factor requires sets with cardinality 1.
     Then, for factors that only consider sets with cardinality 1, the set integral is equivalent to a vector integral, and the set-type BP message~\eqref{eq:setVbp}-\eqref{eq:setFbp} becomes equivalent to those in~\eqref{eq:vecvBel}-\eqref{eq:vecfBel}.
\end{rem}

\begin{figure}[t!]
\begin{centering}
    \includegraphics[width=.7\columnwidth]{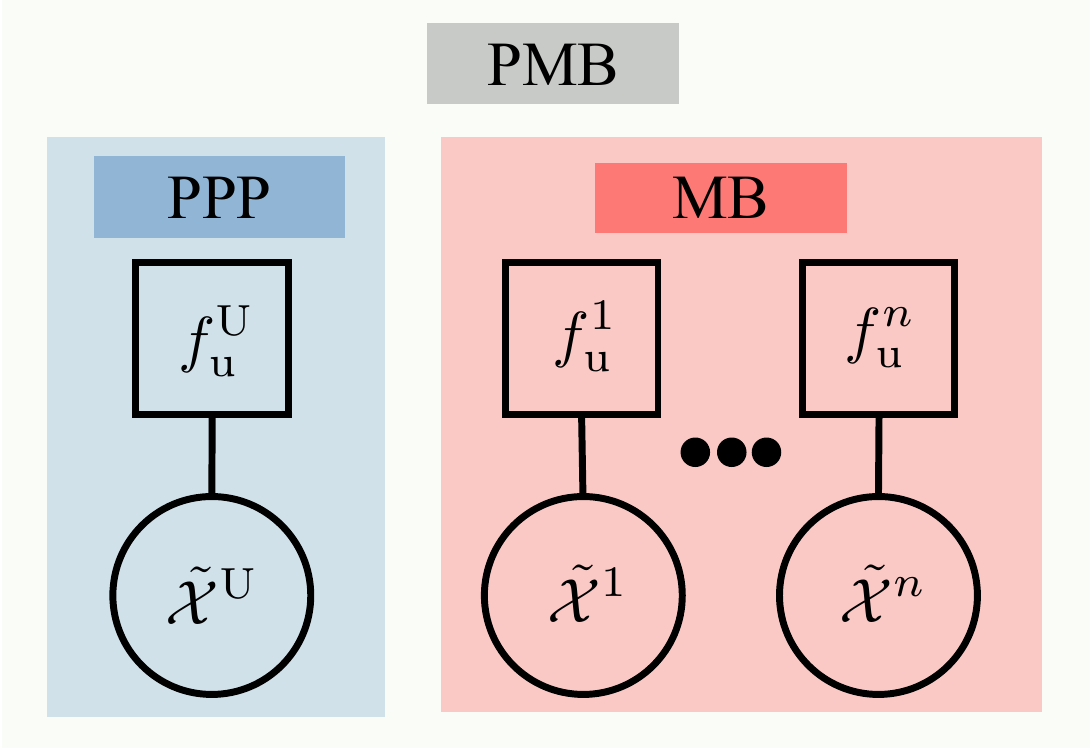}
	\caption{Factor graph of a PMB density with auxiliary variables, see~\eqref{eq:PMB_Aux}, the product of a PPP and $n$ Bernoulli densities.}
	\label{Fig:FG_PMB}
\end{centering}
\end{figure}

\subsection{Examples and Special Factors for Set Densities}\label{sec:SetSpecial}


\subsubsection{Factor Graph of a PMB Density}

We have so far explained how set-type BP can be applied to a joint density over a sequence of RFSs. We now proceed to explain how to obtain this type of density from a PMB density \eqref{eq:PMBdensity}.

A set-density defined in a single-target space that is the disjoint union of different sub-spaces can be used to define a density over a sequence of sets~\cite[Eq.~3.52]{Mahler_book2014}. This type of single-target space was obtained when we introduced auxiliary variables to the \ac{pmb} in~\eqref{eq:PMBdensity}, resulting in \eqref{eq:PMB_Aux}. Therefore, applying this result to the \ac{pmb} density
of the form~\eqref{eq:PMB_Aux} yields
\begin{align}    
    \tilde{f}^{\mathrm{PMB}}(\tilde{\mathcal{X}}^\mathrm{U},\tilde{\mathcal{X}}^1,\dots,\tilde{\mathcal{X}}^n)=\tilde{f}^{\mathrm{PMB}}(\tilde{\mathcal{X}}^\mathrm{U} \uplus \tilde{\mathcal{X}}^1 \uplus \dots \uplus \tilde{\mathcal{X}}^n).
\end{align}
We can now use this joint density over a sequence of RFSs to apply set-type BP.
    The factor graph of a PMB density \eqref{eq:PMB_Aux} with auxiliary variables is then shown in Fig.~\ref{Fig:FG_PMB}. Since the density is fully factorized, it appears as a collection of disjoint factors. We can also directly obtain the factor graph of a PPP and an MB density with auxiliary variables by removing the required factors and variables in Fig.~\ref{Fig:FG_PMB}.



 
\subsubsection{Partitioning and Merging Factor}
    Unions of \acp{rfs} are common in the literature~\cite[Sec.~3.5.3]{Mahler_book2014}. To represent unions of \acp{rfs} in a factor graph, we introduce what we refer to as partitioning and merging factors, defined as follows.
\begin{define}[Partitioning and Merging Factor] \label{def:PartitionFactor}


    We define a partitioning and merging set-factor as
\begin{align}
    f_a(\underline{\mathcal{X}}^a) 
    &= \delta_{\uplus_{i \in \mathcal{N}(a) \setminus \{j\}} \mathcal{X}^{i}}(\mathcal{X}^j),
\end{align}
    where $\delta_\mathcal{X}(\cdot)$ denotes a set Dirac delta centered at set $\mathcal{X}$, defined in~\cite[Sec.~11.3.4.3]{mahler_book_2007}
    and $\mathcal{X}^j$ is the union $\mathcal{X}^j =\uplus_{i \in \mathcal{N}(a) \setminus \{j\}} \mathcal{X}^{i}$. 
    This factor partitions a single set $\mathcal{X}^j$ into $\lvert \mathcal{N}(a) \rvert-1$ subsets, i.e., $\mathcal{X}^i$ for $i \in \mathcal{N}(a) \setminus \{j\}$, and merges $\lvert \mathcal{N}(a) \rvert-1$ sets, i.e., $\mathcal{X}^i$ for $i \in \mathcal{N}(a) \setminus \{j\}$, into a single set $\mathcal{X}^j$.
\end{define}

\begin{figure}[t!]
\begin{centering}
	\includegraphics[width=1\columnwidth]{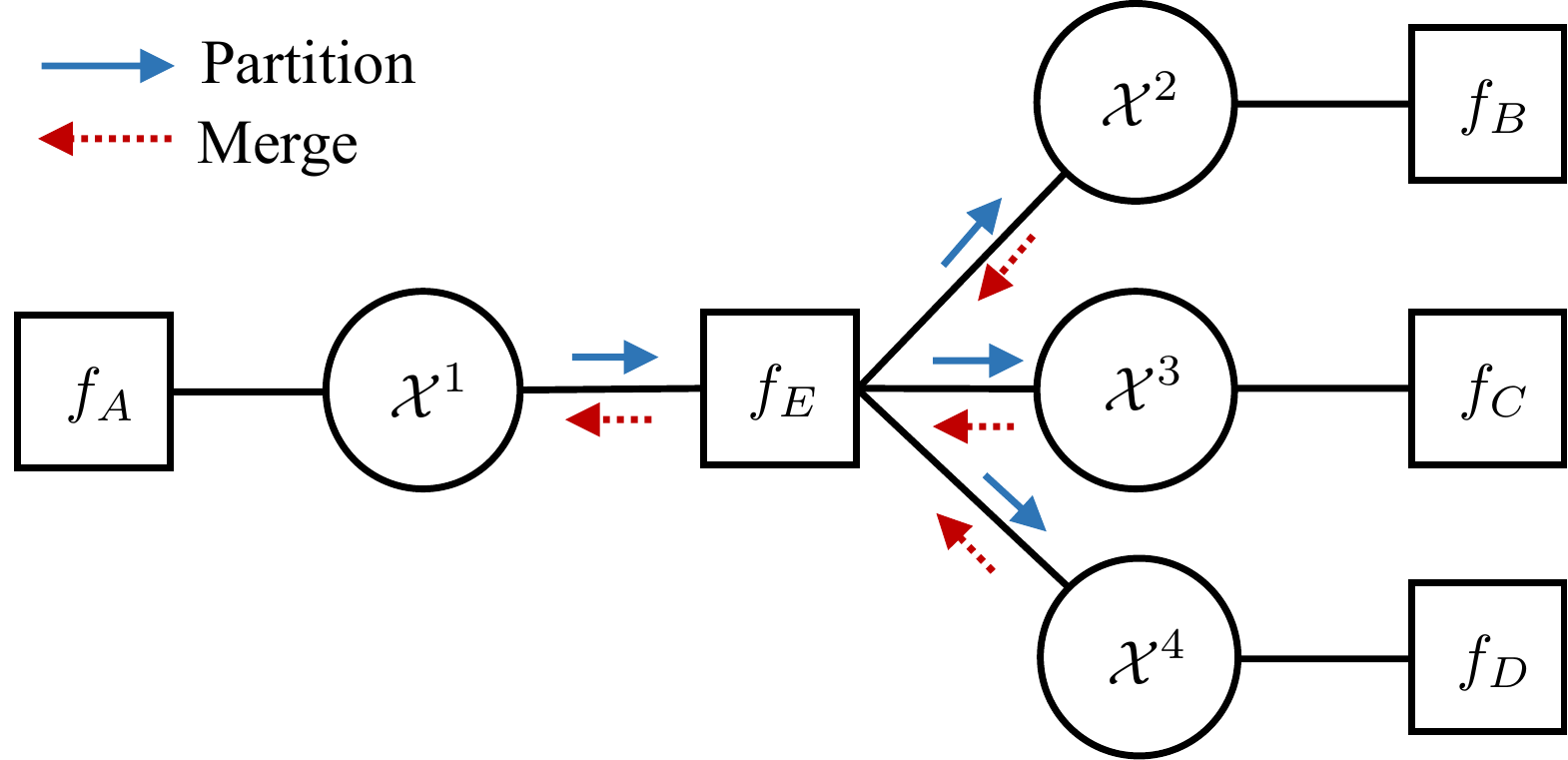}
	\caption{Factor graph representation with a partitioning and merging factor: (left-right) one incoming message is partitioned into multiple messages; and (right-left) multiple messages are merged into one message.}
	\label{Fig:SpeFactor}
\end{centering}
\end{figure}
    It is useful to understand how this factor affects the set-messages. 
    Suppose we have incoming messages $\mathtt{n}_{j \rightarrow a}(\mathcal{X}^{j})$ and $\mathtt{n}_{q \rightarrow a}(\mathcal{X}^q)=1$ for $q \in \mathcal{N}(a) \setminus \{j,i\}$.
    By following the set-type BP update rules, the outgoing messages from the set-factor $a$ to set-variable $i$ for $i \in \mathcal{N}(a) \setminus \{j\}$ are
\begin{align}
    \mathtt{m}_{a \to i}(\mathcal{X}^i) 
    &= 
    \int 
    \mathtt{n}_{j \to a}(\mathcal{X}^j)
    \delta_{\uplus_{q \in \mathcal{N}(a) \setminus \{j\}} \mathcal{X}^{q}} (\mathcal{X}^j)
    \delta \mathcal{X}^{\sim i}
    \\
    &= \int \mathtt{n}_{j \to a}(\uplus_{q \in \mathcal{N}(a) \setminus \{j\}} \mathcal{X}^q)
    \delta \mathcal{X}^{\sim (i,j)} \label{eq:SMFactor_ForOut}\\
    & =  \int \mathtt{n}_{j \to a}(\mathcal{X}^i \uplus \mathcal{X})
    \delta \mathcal{X} \label{eq:SMFactor_ForOut2},
\end{align}
    where $\int \delta \mathcal{X}^{\sim(i,j)}$ indicates the integration with respect to all sets $\mathcal{X}^q$ except $\mathcal{X}^i$ and $\mathcal{X}^j$, and $\mathcal{X}$ is a dummy variable that is used to integrate over all possible sets.
    It indicates that $\mathcal{X}=\uplus_{q\in\mathcal{N}(a)\setminus \{j,i\}} \mathcal{X}^q$, and
    the single set $\mathcal{X}^j$ with the incoming message $\mathtt{n}_{j \to a}(\mathcal{X}^j)$ is partitioned into $\lvert \mathcal{N}(a) \rvert-1$ subsets $\mathcal{X}^i$ with the outgoing messages $\mathtt{m}_{a \rightarrow i}(\mathcal{X}^i)$, for $i \in \mathcal{N}(a) \setminus \{j\}$.

    Conversely,
    suppose we have incoming messages $\mathtt{n}_{i \to a}(\mathcal{X}^i)$ for $i \in \mathcal{N}(a) \setminus \{j\}$, then the outgoing message from the set-factor $a$ to the set-variable $j$ is
\begin{align}
    &\mathtt{m}_{a \rightarrow j}(\mathcal{X}^j) \notag \\
    &= 
    \int 
    \prod_{i \in \mathcal{N}(a) \setminus \{j\}} \mathtt{n}_{i \to a}(\mathcal{X}^i)
    \delta_{\uplus_{i \in \mathcal{N}(a) \setminus \{j\}} \mathcal{X}^{i}} (\mathcal{X}^j)
    \delta \mathcal{X}^{\sim j}\label{eq:SMFactor_BackConv}
    \\
    &= \int 
    \sum_{\uplus_{i \in \mathcal{N}(a) \setminus \{j\}} \mathcal{W}^{i}=\mathcal{X}^j} 
    \prod_{i \in \mathcal{N}(a) \setminus \{j\}} \mathtt{n}_{i \to a}(\mathcal{X}^{i})
    \delta_{\mathcal{W}^i}(\mathcal{X}^{i})
    \delta \mathcal{X}^{\sim j} \\
    & = 
    \sum_{\uplus_{i \in \mathcal{N}(a) \setminus \{j\}} \mathcal{W}^{i}=\mathcal{X}^j}
    \prod_{i \in \mathcal{N}(a) \setminus \{j\}} 
    \mathtt{n}_{i \to a}(\mathcal{W}^{i}),
    \label{eq:SMFactor_BackOut}
\end{align}
    where in step \eqref{eq:SMFactor_BackConv}, we have used that $\delta_{\uplus_{i \in \mathcal{N}(a) \setminus \{j\}} \mathcal{X}^i} (\mathcal{X}^j)= \sum_{\uplus_{i \in \mathcal{N}(a) \setminus \{j\}} \mathcal{W}^{i}=\mathcal{X}^j} \prod_{i\in \mathcal{N}(a)\setminus \{j\}} \delta_{\mathcal{W}^{i}}(\mathcal{X}^{i}) $ by the convolution formula.
    Then, \eqref{eq:SMFactor_BackOut} indicates that the sets $\mathcal{X}^i$ with the incoming messages $\mathtt{n}_{i \to a}(\mathcal{X}^i)$ for $i \in \mathcal{N}(a)\setminus \{j\}$ are merged into the single set $\mathcal{X}^j$ with the outgoing message $\mathtt{m}_{a \to j}(\mathcal{X}^{j})$.
    
    

\begin{example}
    Given the factor graph shown in Fig.~\ref{Fig:SpeFactor} such that $\mathcal{X}^{1} = \mathcal{X}^{2} \uplus \mathcal{X}^{3} \uplus \mathcal{X}^{4}$, the argument at the set-factor $E$ is represented as the sequence of sets  as
    $\underline{\mathcal{X}}^E=(\mathcal{X}^{1} , \mathcal{X}^{2} , \mathcal{X}^{3} , \mathcal{X}^{4})$
    by Definition~\ref{def:SetFG}. 
    By Definition~\ref{def:PartitionFactor}, the partitioning and merging factor is given by $f_E(\underline{\mathcal{X}}^E) 
    = \delta_{\mathcal{X}^{2} \uplus \mathcal{X}^{3} \uplus \mathcal{X}^{4}}(\mathcal{X}^1)$.  
    Suppose we have incoming messages $\mathtt{n}_{1 \rightarrow E}(\mathcal{X}^1)$, $\mathtt{n}_{i \rightarrow E}(\mathcal{X}^{i})=1$, for $i \in \{2,3,4\}$. 
    From~\eqref{eq:SMFactor_ForOut}, the partitioned message $\mathtt{m}_{E \rightarrow i}(\mathcal{X}^i)$ is computed as
\begin{align}
    \mathtt{m}_{E \rightarrow i}(\mathcal{X}^i) 
    &= \int \mathtt{n}_{1 \rightarrow E}(\mathcal{X}^{2} \uplus \mathcal{X}^{3} \uplus \mathcal{X}^{4})
    \delta \mathcal{X}^{\sim(i,1)},
\end{align}
    for $i=2,3,4$.
    Conversely, suppose we have incoming messages $\mathtt{n}_{j \rightarrow E}(\mathcal{X}^j)$ for $j \in \{2,3,4\}$, then the outgoing message from the set-factor $E$ to set-variable 1 is computed as
\begin{align}
    &\mathtt{m}_{E \rightarrow 1}(\mathcal{X}^1) \notag \\
    & = 
    \sum_{\mathcal{W}^{2} \uplus \mathcal{W}^{3} \uplus \mathcal{W}^{4}=\mathcal{X}^1}
    \mathtt{n}_{2 \rightarrow E}(\mathcal{W}^{2})\
    \mathtt{n}_{3 \rightarrow E}(\mathcal{W}^{3})
    \mathtt{n}_{4 \rightarrow E}(\mathcal{W}^{4}).
\end{align}
    That is, the outgoing message is the convolution of the three incoming messages, representing the  union of three independent \acp{rfs}.
\end{example}

\begin{prop}[PPP Partitioning and Merging] \label{prop:Poisson}
    Suppose we have a set factor $f_a(\underline{\mathcal{X}}^a)$ and an \ac{rfs} $\mathcal{X}^j=\uplus_{i\in \mathcal{N}(a) \setminus \{j\}}\mathcal{X}^i$.
    Let the \acp{rfs} $\mathcal{X}^i$ for $i \in \mathcal{N}(a)$ follow a Poisson process.
    From~\eqref{eq:SMFactor_ForOut2}, the partitioning messages from $f_a(\underline{\mathcal{X}}^a)$ to $\mathcal{X}^i$ with $i \in \mathcal{N}(a) \setminus \{j\}$ are
\begin{align}
    \mathtt{m}_{a \rightarrow i}(\mathcal{X}^i) 
    \propto f^{\mathrm{PPP}}(\mathcal{X}^{i}),
\end{align}
    From~\eqref{eq:SMFactor_BackOut}, the merging messages from $f_a(\underline{\mathcal{X}}^a)$ to $\mathcal{X}^j$ is
\begin{align}
    \mathtt{m}_{a \rightarrow j}(\mathcal{X}^j) 
    & = 
    \sum_{\uplus_{i \in \mathcal{N}(a) \setminus \{j\}} \mathcal{W}^{i}=\mathcal{X}^j}
    \prod_{i \in \mathcal{N}(a) \setminus \{j\}} 
    f^\mathrm{PPP}(\mathcal{W}^{i}), \label{eq:PPPMerg}
\end{align}
    which follows a \ac{ppp}.
\end{prop}
\begin{proof}
    See Appendix~\ref{app:ProofPoi}.
\end{proof}

\subsubsection{Auxiliary Variable Shifting Factor}
    We introduce a factor to change the auxiliary variables. 

\begin{define} [Auxiliary Variable Shifting Function] \label{def:AuxShift}
    Given an arbitrary integer $L$ and a single-target space with auxiliary variables, such that $({u},\mathbf{x})\in \mathbb{U} \times \mathbb{R}^{n_{\mathbf{x}}}$, we define the function
\begin{align}
    \mathtt{h}_L(u,\mathbf{x}) = (u+L,\mathbf{x}),
\end{align}
    which shifts the auxiliary variables units by $L$. The function $\mathtt{h}_L(\cdot)$ can be extended to a set $\tilde{\mathcal{X}} \in \mathcal{F}(\mathbb{U}\times \mathbb{R}^{n_\mathbf{x}})$ such that
\begin{align}
    \mathtt{h}_L(\tilde{\mathcal{X}}) =
    \uplus_{(u,\mathbf{x})\in \tilde{\mathcal{X}}} \{ (u+L,\mathbf{x}) \}.
    \label{eq:ConvfuncAux}
\end{align}
\end{define}
\begin{define} [Conversion Factor for Auxiliary Variables] \label{def:ConvertFactor}
    We have two sets $\tilde{\mathcal{X}}^i=\{({u}^1,\mathbf{x}^1),\dots,({u}^n,\mathbf{x}^n)\}$ and  $\tilde{\mathcal{X}}^j=\{({u}^1+L,\mathbf{x}^1),\dots,({u}^n+L,\mathbf{x}^n)\}$, where $L$ is an integer.
    By Definition~\ref{def:AuxShift}, 
    we define the conversion factor for auxiliary variables 
    as
\begin{align}
    f_a(\tilde{\mathcal{X}}^i,\tilde{\mathcal{X}}^j) = \delta_{\mathtt{h}_{-L}(\tilde{\mathcal{X}}^j)}(\tilde{\mathcal{X}^i}),
\end{align}
    where $\mathtt{h}_{-L}(\cdot)$ is given by \eqref{eq:ConvfuncAux}.
\end{define}
    It is useful to understand how this factor affects the set-messages. 
    Suppose we have an incoming message $\mathtt{n}_{i \rightarrow a}(\tilde{\mathcal{X}}^i)$.
    Then, the outgoing message from $f_a(\tilde{\mathcal{X}}^i,\tilde{\mathcal{X}}^j)$ to $\tilde{\mathcal{X}}^j$ is
\begin{align}
    \mathtt{m}_{a \rightarrow j}(\tilde{\mathcal{X}^j}) 
    &= \int \mathtt{n}_{i \rightarrow a}(\tilde{\mathcal{X}}^i)
    \delta_{\mathtt{h}_{-L}(\tilde{\mathcal{X}}^j)}(\tilde{\mathcal{X}}^i)
    \delta \tilde{\mathcal{X}}^i \\
    &
    = \mathtt{n}_{i \rightarrow a}(\mathtt{h}_{-L}(\tilde{\mathcal{X}}^j))=\mathtt{n}_{i \rightarrow a}(\tilde{\mathcal{X}}^i).
\end{align}
    That is, the outgoing message has the same form as the incoming message with the difference that the auxiliary variables of the incoming message have been converted by the auxiliary variable shifting functions.

\section{Application of Set-Type Belief Propagation}\label{sec:SetBP-SLAM}
    The aim of this section is to propose an application of the developed set-type \ac{bp} update rules and special factors for \acp{rfs}.
    In particular, we derive set-type \ac{bp} \ac{pmb} and set-type \ac{mb} filters for \ac{slam}, where the targets and measurements are modeled by \acp{rfs}.

\subsection{Problem Formulation}
\label{sec:JointSLAMProb}

\subsubsection{Objective}
    Our objective is to compute the marginal densities $f(\mathbf{s}_{k}|\mathcal{Z}_{1:k})$ and $f(\tilde{\mathcal{X}_{k}}|\mathcal{Z}_{1:k})$ at discrete time $k$, where the random vector $\mathbf{s}_k$ denotes a sensor state, the \ac{rfs} $\tilde{\mathcal{X}_{k}}$ denotes the set of target states with  auxiliary variables, modeled by a \ac{pmb}.
    To compute the marginal densities, we adopt the sequential Bayesian framework consisting of prediction and update steps, which will be detailed in Section~\ref{sec:JDF-Pred} and~\ref{sec:JDF-UP}.
\subsubsection{Multi-Target Dynamics}
    Each target $\mathbf{x}_{k-1}\in \tilde{\mathcal{X}}_{k-1}$ at time $k-1$ survives with  probability $p_\mathrm{S}(\mathbf{x}_{k-1})$ or dies with probability $1-p_\mathrm{S}(\mathbf{x}_{k-1})$. 
    The surviving targets evolve with a transition density $f(\cdot|\mathbf{x}_{k-1})$ but may be static (in which case they are landmarks)
    or mobile (in this case they are targets, which is the terminology we will adopt here).
    The set of targets at time step $k$, $\tilde{\mathcal{X}}_k$, is the union of surviving and evolving targets and new targets, where target birth follows a \ac{ppp} with the intensity $\lambda_\mathrm{B}(\cdot)$.
    The sensor may have an unknown state (in the case of \ac{slat} and \ac{slam}) or a known state (in the case of \ac{mtt} and mapping).
\subsubsection{Measurements}
    The targets ${\mathcal{X}}_k$ are observed at the sensor state $\mathbf{s}_k$, and the observations are denoted using a measurement set $\mathcal{Z}_k$.
    Each target $\mathbf{x}_{k}\in \tilde{\mathcal{X}}_{k}$ is detected with probability $p_\mathrm{D}(\mathbf{s}_k,\mathbf{x}_k)$, and if detected, it generates a single measurement with 
    the single target measurement likelihood function $g(\cdot|\mathbf{s}_k,\mathbf{x}_k)$.
    The measurement $\mathcal{Z}_k$ is the union of target measurements and \ac{ppp}
    with the intensity $c(\cdot)$.

\begin{figure*}
\begin{centering}
	\includegraphics[width=2\columnwidth]{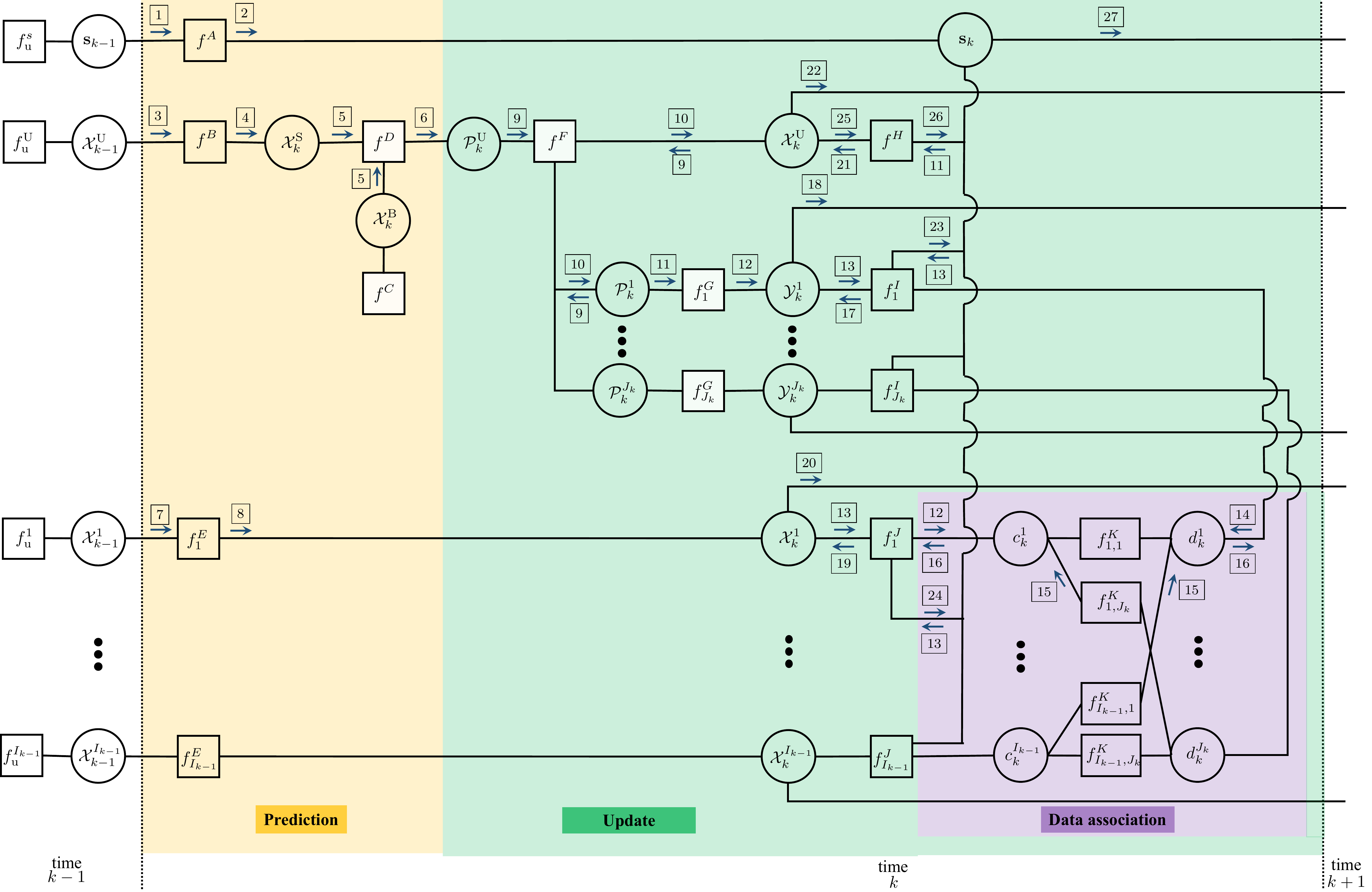}
    \caption{Concatenated factor graph of the joint densities of 
    \eqref{eq:P_factorizedDensity} and \eqref{eq:factorizedDensity}.
    The special factors for set densities, discussed in Section~\ref{sec:SetSpecial}, are represented by white squares.  The abbreviated notations for the factor nodes are represented by dropping the arguments and $\tilde{(\cdot)}$ on sets, as follows: 
    $f_\mathsf{u}^s \triangleq \mathtt{b}(\mathbf{s}_{k-1})$,
    $f_\mathsf{u}^\mathrm{U} \triangleq \mathtt{b}(\mathcal{X}_{k-1}^\mathrm{U})$,
    $f_\mathsf{u}^i \triangleq \mathtt{b}(\mathcal{X}_{k-1}^i)$,
    $f^A \triangleq f^A(\mathbf{s}_k,\mathbf{s}_{k-1})$,
    $f^B \triangleq f^B({\mathcal{X}}_{k}^\mathrm{S},{\mathcal{X}}_{k-1}^\mathrm{U})$,
    $f^C \triangleq f^C(\mathcal{X}_k^\mathrm{B})$,
    $f^D \triangleq f^D({\mathcal{X}}_{k}^\mathrm{S},{\mathcal{X}}_{k}^\mathrm{B},{\mathcal{P}}_{k}^\mathrm{U})$,
    $f_i^E  \triangleq f_i^E(\mathcal{X}_{k}^i,\mathcal{X}_{k-1}^i)$,
    $f^F \triangleq f^F( {\mathcal{X}}_{k}^\mathrm{U}, {\mathcal{P}}_k^1,\cdots ,{\mathcal{P}}_k^{J_k})$,
    $f_j^G  \triangleq f_j^G({\mathcal{P}}_k^j,{\mathcal{Y}}_k^j)$,
    $f^H \triangleq f^H(\mathbf{s}_k,\mathcal{X}_k^\mathrm{U})$,
    $f_j^I \triangleq f_j^I(\mathbf{s}_k,\mathcal{Y}_k^j,d_k^j)$,
    $f_i^J \triangleq f_i^J(\mathbf{s}_k,\mathcal{X}_k^i,c_k^i)$, and
    $f_{i,j}^K \triangleq f_{i,j}^K(c_{k}^i,d_{k}^j)$.
    }
	\label{Fig:FG}
\end{centering}
\end{figure*}

\subsection{Prediction with Joint Density and Factor Graph}
\label{sec:JDF-Pred}
    Without loss of generality, we consider two time steps, $k-1$ and $k$, as part of an iterative Bayesian filter. 
    For the factor graph formulation, 
    the joint density for all variables (including auxiliary variables) in the prediction is factorized as
\begin{subequations}
\begin{align}
    &f(\mathbf{s}_{k-1:k},\tilde{\mathcal{X}}_{k-1}^\mathrm{U},\tilde{\mathcal{X}}_{k}^\mathrm{S},\tilde{\mathcal{X}}_{k}^\mathrm{B},\tilde{\mathcal{P}}_{k}^\mathrm{U},\tilde{\mathcal{X}}_{k-1:k}^{1:I_{k-1}})\notag\\
  & \propto f_{\mathtt{u}}(\mathbf{s}_{k-1})f_{\mathtt{u}}^{\mathrm{U}}(\tilde{\mathcal{X}}_{k-1}^{\mathrm{U}})\prod_{i\in\mathcal{I}_{k-1}}f_{\mathtt{u}}^{i}(\tilde{\mathcal{X}}_{k-1}^{i}) \label{eq:P_factorized1}\\
 & \,\, \times f(\mathbf{s}_{k}|\mathbf{s}_{k-1}) 
 f^\mathrm{P}(\tilde{\mathcal{X}}_{k}^{\mathrm{S}}|\tilde{\mathcal{X}}_{k-1}^{\mathrm{U}}) f(\tilde{\mathcal{X}}_{k}^{i}|\tilde{\mathcal{X}}_{k-1}^{i})\label{eq:P_factorized2}\\
 & \,\, \times f^{\mathrm{U}}(\tilde{\mathcal{X}}_{k}^{\mathrm{B}})\delta_{\tilde{\mathcal{X}}_{k}^{\mathrm{S}}\uplus\tilde{\mathcal{X}}_{k}^{\mathrm{B}}}(\tilde{\mathcal{P}}_{k}^{\mathrm{U}}),
 \label{eq:P_factorized3}
\end{align}
\label{eq:P_factorizedDensity}%
\end{subequations}
where $\tilde{\mathcal{X}}_{k-1:k}^{1:I_{k-1}}$ represents the sequence of sets $\tilde{\mathcal{X}}^{i}$ for $i=1,\dots,I_{k-1}$ at time $k$ and $k-1$, and $I_{k-1}$ is the number of Bernoullis at time $k-1$.
For completeness, the meaning of each line in~\eqref{eq:P_factorizedDensity} is described as follows.
\begin{itemize}
    \item \emph{Posterior at time $k-1$ \eqref{eq:P_factorized1}}: This line describes 
    the posterior at time step $k-1$, which is assumed to be $f_{\mathrm{u}}(\mathbf{s}_{k-1},\mathcal{\tilde{X}}_{k-1})=f_{\mathrm{u}}(\mathbf{s}_{k-1}) f_{\mathrm{u}}(\mathcal{\tilde{X}}_{k-1})$, where $f_{\mathrm{u}}(\mathbf{s}_{k-1})$ is the sensor state posterior, and $f_{\mathrm{u}}(\tilde{\mathcal{X}}_{k-1}^{\mathrm{U}})$ is the target set posterior. Here  $\tilde{\mathcal{X}}_{k-1}$ follows a \ac{pmb}, 
    endowed with auxiliary variables (see Section~\ref{sec:AuxV}).
    Due to the auxiliary variables, the \ac{pmb} posterior at time $k-1$, $f_{\mathrm{u}}(\mathcal{\tilde{X}}_{k-1})$, can be factorized as the product of $f_{\mathrm{u}}(\tilde{\mathcal{X}}_{k-1}^{\mathrm{U}})$ and $f_{\mathrm{u}}(\tilde{\mathcal{X}}_{k-1}^{i})$ for $i\in\mathcal{I}_{k-1}$, where $\mathcal{I}_{k-1}=\{1,\dots,I_{k-1}\}$.
    Here $\tilde{\mathcal{X}}_{k-1}^{\mathrm{U}}$ is the set of undetected targets (modeled as a PPP), $\tilde{\mathcal{X}}_{k-1}^{i}$ is the set of detected target $i$ (each modeled as a Bernoulli).
    The subsets $\tilde{\mathcal{X}}_{k-1}^\mathrm{U}$ and $\tilde{\mathcal{X}}_{k-1}^i $ of $\tilde{\mathcal{X}}_{k-1}^{\mathrm{U}}$ are defined as
\begin{align}
    \tilde{\mathcal{X}}_{k-1}^\mathrm{U}& =  \{(u,\mathbf{x}) \in \tilde{\mathcal{X}}_{k-1}: u = 0\} ,\\
    \tilde{\mathcal{X}}_{k-1}^i &= \{(u,\mathbf{x}) \in \tilde{\mathcal{X}}_{k-1}: u = i\},
    ~i\in \mathcal{I}_{k-1}.
\end{align}
    The set densities for targets are given by the form of~\eqref{eq:PMB_Aux}--\eqref{eq:PMB_Aux_Ber}.

\item \emph{Transition densities \eqref{eq:P_factorized2}}: 
    This line describes the transition densities, where $f(\mathbf{s}_{k}|\mathbf{s}_{k-1})$, $f(\tilde{\mathcal{X}}_{k}^{i}|\tilde{\mathcal{X}}_{k-1}^{i})$, and $f^\mathrm{P}(\tilde{\mathcal{X}}_{k}^{\mathrm{S}}|\tilde{\mathcal{X}}_{k-1}^{\mathrm{U}})$ correspond to the sensor state, previously detected targets, and undetected targets.
    Here $\tilde{\mathcal{X}}_{k}^{\mathrm{S}}$ denotes a surviving set from $\tilde{\mathcal{X}}_{k-1}^{\mathrm{U}}$, modeled as a \ac{ppp} with auxiliary variable $u=0$.
    The transition density for Bernoullis with auxiliary variables is 
\begin{align}
    &f(\tilde{\mathcal{X}}_k|\tilde{\mathcal{Y}}_k) \label{eq:Tran_Ber}\\
    &=
    \begin{cases}
        p_\mathrm{S}(\mathbf{y}_k)f(\mathbf{x}_k|\mathbf{y}_k)\delta_{u_k'}[u_k], & \tilde{\mathcal{X}}_k=\{(u_k,\mathbf{x}_k)\},\\
        & \tilde{\mathcal{Y}_k}=\{(u_k',\mathbf{y}_k)\}\\
        1-p_\mathrm{S}(\mathbf{y}_k), & \tilde{\mathcal{X}}_k=\emptyset, \\
        & \tilde{\mathcal{Y}}_k=\{(u_k',\mathbf{y}_k)\}\\
        1, & \tilde{\mathcal{X}}_k=\emptyset,\tilde{\mathcal{Y}}_k=\emptyset\\
        0, & \text{otherwise}
    \end{cases},\notag
\end{align}
    and the transition density for the \ac{ppp} is 
\begin{align}
    &f^\mathrm{P}(\tilde{\mathcal{X}}_k|\tilde{\mathcal{Y}}_k) 
    \\
    &= 
    \begin{cases}
        \sum\limits_{\uplus_{i=1}^n \tilde{\mathcal{X}}_k^i=\tilde{\mathcal{X}}_k}\prod\limits_{i=1}^n f(\tilde{\mathcal{X}}_k^i |\{\tilde{\mathbf{y}}_k^i\}), & \tilde{\mathcal{Y}}_k=\{\tilde{\mathbf{y}}_k^1,\dots,\tilde{\mathbf{y}}_k^n\} \\
        1, & \tilde{\mathcal{X}}_k=\emptyset, \tilde{\mathcal{Y}}_k=\emptyset\\
        0, & \text{otherwise}
    \end{cases}. \notag
\end{align}
    where $f(\tilde{\mathcal{X}}_k^i |\{\tilde{\mathbf{y}}_k^i\})$ follows~\eqref{eq:Tran_Ber} and $\tilde{\mathbf{y}}_k^i=(0,\mathbf{y}_k^i)$.
\item \emph{Undetected targets \eqref{eq:P_factorized3}}:
    This last line describes the set-density of newborn targets at time $k$, represented by $f(\tilde{\mathcal{X}}_{k}^{\mathrm{B}})$, where $\tilde{\mathcal{X}}_{k}^{\mathrm{B}}$ denotes a set of newborn targets that follows a PPP, and the set-factor $\delta_{\tilde{\mathcal{X}}_{k}^{\mathrm{S}}\uplus\tilde{\mathcal{X}}_{k}^{\mathrm{B}}}(\tilde{\mathcal{P}}_{k}^{\mathrm{U}})$ merges the set of surviving targets $\tilde{\mathcal{X}}_{k}^{\mathrm{S}}$ and the set of newborn targets $\tilde{\mathcal{X}}_{k}^{\mathrm{B}}$ into the set $\tilde{\mathcal{P}}_{k}^{\mathrm{U}}$ (see also Definition~\ref{def:PartitionFactor}).
    Both $\tilde{\mathcal{X}}_{k}^{\mathrm{B}}$ and $\tilde{\mathcal{P}}_{k}^{\mathrm{U}}$ have auxiliary variable $u=0$.
    
\end{itemize}

    Using the factorized density in~\eqref{eq:P_factorizedDensity}, we depict the factor graph for prediction, as shown in  Fig.~\ref{Fig:FG} (yellow area).
    Here we use the following notations for the factor nodes: 
\begin{subequations}
\begin{align}
    &f^A(\mathbf{s}_k,\mathbf{s}_{k-1}) \triangleq f(\mathbf{s}_k|\mathbf{s}_{k-1}), \\
    &f^B(\tilde{\mathcal{X}}_{k}^\mathrm{S},\tilde{\mathcal{X}}_{k-1}^\mathrm{U}) 
    \triangleq 
    f(\tilde{\mathcal{X}}_{k}^\mathrm{S}|\tilde{\mathcal{X}}_{k-1}^\mathrm{U}),\\
    &f^C(\tilde{\mathcal{X}}_k^\mathrm{B}) \triangleq f^\mathrm{U}(\tilde{\mathcal{X}}_k^\mathrm{B})\\
    &f^D(\tilde{\mathcal{X}}_{k}^\mathrm{S}, \tilde{\mathcal{X}}_{k}^\mathrm{B},\tilde{\mathcal{P}}_{k}^\mathrm{U}) 
    \triangleq 
    \delta_{\tilde{\mathcal{X}}_{k}^\mathrm{S} \uplus \tilde{\mathcal{X}}_{k}^\mathrm{B}}(\tilde{\mathcal{P}}_{k}^\mathrm{U}),\\
    &f_i^E(\tilde{\mathcal{X}}_{k}^i,\tilde{\mathcal{X}}_{k-1}^i)  \triangleq 
    f(\tilde{\mathcal{X}}_{k}^i|\tilde{\mathcal{X}}_{k-1}^i).
\end{align}
\end{subequations}
    We compute the messages on the prediction factor graph in Fig.~\ref{Fig:FG}, which has no cycle, by running developed set-type \ac{bp}.\footnote{This step is a straightforward application of the message passing rules introduced in Section \ref{sec:FG_SetBP}.
    For completeness, the messages are detailed in Appendix~\ref{SM:Messages} of the supplementary material.}
    Note that the messages \fbox{\footnotesize{2}}, \fbox{\footnotesize{6}}, and \fbox{\footnotesize{8}} are equivalent to the marginal densities in prediction of  $\mathbf{s}_k$, $\tilde{\mathcal{P}}_k^\mathrm{U}$, and $\tilde{\mathcal{X}}_k^{1:I_{k-1}}$, corresponding to prediction of the standard \ac{pmb} filter~\cite{Jason_PMB_TAES2015}, represented as $f_\mathtt{p}(\mathbf{s}_k)$, $f_\mathtt{p}^{\mathrm{U}}(\tilde{\mathcal{P}}_k^\mathrm{U})$, and $f_\mathtt{p}^i(\tilde{\mathcal{X}}_k^i)$ for $i\in\mathcal{I}_{k-1}$.

\subsection{Update with Joint Density and its Factor Graph}
\label{sec:JDF-UP}
    Without any loss of generality, we consider the update step at time $k$ after the prediction step.
    The joint density for all variables in the update is
\begin{subequations}
\begin{align}
    &f(\mathbf{s}_{k},\tilde{\mathcal{P}}_{k}^\mathrm{U},\tilde{\mathcal{P}}_{k}^{1:J_k},\tilde{\mathcal{X}}_{k}^\mathrm{U},\tilde{\mathcal{Y}}_{k}^{1:J_k},\tilde{\mathcal{X}}_{k}^{1:I_{k-1}},\mathbf{c}_{k},\mathbf{d}_{k}|\mathcal{Z}_{k})\notag \\
    & \propto f_{\mathtt{p}}(\mathbf{s}_{k})f_{\mathtt{p}}^{\mathrm{U}}(\tilde{\mathcal{X}}_{k}^{\mathrm{U}})\prod_{i\in\mathcal{I}_{k-1}}f_{\mathtt{p}}^{i}(\tilde{\mathcal{X}}_{k}^{i}) \label{eq:factorized1}\\
    & \times \delta_{\tilde{\mathcal{X}}_{k}^{\mathrm{U}}\uplus_{j=1}^{J_{k}}\tilde{\mathcal{P}}_{k}^{j}}(\tilde{\mathcal{P}}_{k}^{\mathrm{U}})
    \prod_{j\in\mathcal{J}_{k}} \psi(c_{k}^{i},d_{k}^{j}) \delta_{\mathtt{h}_{-(I_{k-1}+j)}(\tilde{\mathcal{Y}}_{k}^{j})}(\tilde{\mathcal{P}}_{k}^{j})
 \label{eq:factorized3}\\
    & \times \tilde{l}(\mathbf{z}_{k}^{j}|\mathbf{s}_{k},\tilde{\mathcal{Y}}_{k}^{j},d_{k}^{j})t(\mathcal{Z}_{k}^{i}|\mathbf{s}_{k},\tilde{\mathcal{X}}_{k}^{i},c_{k}^{i}) [1-p_\mathrm{D}(\mathbf{s}_k,\cdot)]^{\tilde{\mathcal{X}}_{k}^{\mathrm{U}}}. \label{eq:factorized4}
\end{align}
\label{eq:factorizedDensity}%
\end{subequations}
    The proof of~\eqref{eq:factorizedDensity} is an extension of Appendix~E in~\cite{Hyowon_MPMB_TVT2022} and~\cite{Angel_PMBM_TAES2018}, found in Appendix~\ref{SM:JointUp}.
    For completeness, we describe the meaning of each line in \eqref{eq:factorizedDensity} as follows.
\begin{itemize}
    \item \emph{Prediction at time $k-1$~\eqref{eq:factorized1}}: This line describes the predicted densities at time step $k$, discussed at the end of Section~\ref{sec:JDF-Pred}.
    \item \emph{Consistency constraints~\eqref{eq:factorized3}}: This line ensures consistency among the introduced hidden variables.
    A set of measurements $\mathcal{Z}_{k}$ is provided at time $k$ with index set $\mathcal{J}_{k}=\{1,\dots,J_k\}$. 
    We introduce data association variables $\mathbf{c}_{k}\in\mathbb{N}^{I_{k-1}}$, $c_{k}^{i}\in\{0,\dots,J_{k}\}$, and $\mathbf{d}_{k}\in\mathbb{N}^{J_{k-1}}$, $d_{k}^{j}\in\{0,\dots,I_{k-1}\}$, where, in target-oriented data association, $c_{k}^{i}=j$ implies that target $i$ is associated with measurement $j$, and $c_{k}^{i}=0$ implies that target $i$ is not detected.
    In measurement-oriented data association, $d_{k}^{j}=i>0$ implies that measurement $j$ is associated with target $i$, and $d_{k}^{j}=0$ implies that measurement $j$ originates from either a target that has never been detected or clutter. 
    In particular, we introduced $\Psi(\mathbf{c}_{k},\mathbf{d}_{k})=\prod_{i=1}^{I_{k-1}}\prod_{j=1}^{J_{k}}\psi(c_{k}^{i},d_{k}^{j})$, which ensures mutual consistency between $\mathbf{c}_{k}$ and $\mathbf{d}_{k}$, with $\psi(c_{k}^{i},d_{k}^{j})=0$ when $c_{k}^{i}=j,d_{k}^{j}\neq i~\text{or}~c_{k}^{i}\neq j,d_{k}^{j}=i$, and 1 otherwise. 
    The factor $\delta_{\tilde{\mathcal{X}}_{k}^{\mathrm{U}}\uplus_{j=1}^{J_{k}}\tilde{\mathcal{P}}_{k}^{j}}(\tilde{\mathcal{P}}_{k}^{\mathrm{U}})$ partitions the set $\tilde{\mathcal{P}}_{k}^{\mathrm{U}}$ into $J_k +1 $ sets,
    such that $\tilde{\mathcal{P}}_{k}^{\mathrm{U}}= \tilde{\mathcal{X}}_{k}^{\mathrm{U}} \uplus_{j=1}^{J_k} \tilde{\mathcal{P}}_{k}^{j}$, where $\tilde{\mathcal{X}}_{k}^{\mathrm{U}}$ is the set of targets that remain undetected (as indicated by $[1-p_\mathrm{D}(\mathbf{s}_k,\cdot)]^{\tilde{\mathcal{X}}_{k}^{\mathrm{U}}}$ in the likelihood part), and $\tilde{\mathcal{P}}_{k}^{j}$ for $j\in\mathcal{J}_k$ represent surviving undetected or newborn targets that are first detected at time $k$. The factor $\delta_{\mathtt{h}_{-(I_{k-1}+j)}(\tilde{\mathcal{Y}}_{k}^{j})}(\tilde{\mathcal{P}}_{k}^{j})$ (see Definition~\ref{def:ConvertFactor}) converts the auxiliary variable of $\tilde{\mathcal{P}}_{k}^{j}$ from 0 to $I_{k-1}+j$, 
    turning the set $\tilde{\mathcal{P}}_{k}^{j}$ into a set $\tilde{\mathcal{Y}}_{k}^{j}$ corresponding to a newly detected target (or clutter) arising from measurement $j$. 
    The subsets $\tilde{\mathcal{X}}_k^\mathrm{U}$, $\tilde{\mathcal{X}}_k^i$, and $\tilde{\mathcal{Y}}_k^j$ are defined as
\begin{align}
    \tilde{\mathcal{X}}_k^\mathrm{U}& =  \{(u,\mathbf{x}) \in \tilde{\mathcal{X}}_k: u = 0\},\\
	\tilde{\mathcal{X}}_k^i &= \{(u,\mathbf{x}) \in \tilde{\mathcal{X}}_k: u = i\},
    ~i\in \mathcal{I}_{k-1},\\
    \tilde{\mathcal{Y}}_k^j &= \{(u,\mathbf{x}) \in \tilde{\mathcal{X}}_k: u = I_{k-1} + j\},
    ~j \in \mathcal{J}_k.
\end{align}
    \item \emph{Set-likelihoods~\eqref{eq:factorized4}}: This last line describes the likelihood, conditioned on the data association $\mathbf{c}_{k}$
(or equivalently $\mathbf{d}_{k}$).
The factor $\tilde{l}(\mathbf{z}_{k}^{j}|\mathbf{s}_{k},\tilde{\mathcal{Y}}_{k}^{j},d_{k}^{j})$
considers potential new targets or clutter, while $t(\mathcal{Z}_{k}^{i}|\mathbf{s}_{k},\tilde{\mathcal{X}}_{k}^{i},c_{k}^{i})$
considers detections and missed detections of the previously detected targets. 
    Both likelihoods are defined as follows. We consider the likelihood when the $j$-th measurement $\mathbf{z}_k^j$ is associated with a newly detected target or with clutter conditioned on measurement-oriented data association variable $d_k^j$: 
$\tilde{l}(\mathbf{z}_k^j|\mathbf{s}_k,\tilde{\mathcal{Y}_k}^{j},d_k^j)$ is given by 
\begin{align}
    &\tilde{l}(\mathbf{z}_k^j|\mathbf{s}_k,\tilde{\mathcal{Y}_k}^{j},d_k^j) \label{eq:likelihoodPoi} \\
    &=
    \begin{cases} 
        p_\mathrm{D}(\mathbf{s}_k,\mathbf{x}_k)g(\mathbf{z}_k^j|\mathbf{s}_k,\mathbf{x}_k)\delta_0[u_k], & \tilde{\mathcal{Y}_k}^{j} = \{(u_k,\mathbf{x}_k)\},\\
        &d_k^j=0 \\
        c(\mathbf{z}_k), & \tilde{\mathcal{Y}_k}^{j} = \emptyset,~d_k^j=0\\
        1, & \tilde{\mathcal{Y}_k}^{j}=\emptyset ,~d_k^j \neq 0\\
        0, & \mathrm{otherwise}.
    \end{cases}\notag 
\end{align}
Here $p_\mathrm{D}(\cdot)$ accounts for the fact that a target may be detected, while $c(\mathbf{z}_k)$ represents the clutter intensity (when a measurement is generated by clutter and not a newly detected target). The function $g(\mathbf{z}_k^j|\mathbf{s}_k,\mathbf{x}_k)$ is a classical likelihood function. 
Then, the likelihood $t(\mathcal{Z}_k^i|\mathbf{s}_k,\tilde{\mathcal{X}}_k^i,c_k^i)$ considers the case when a set  $\mathcal{Z}_k^i$ is associated to the $i$-th previously detected target conditioned on target-oriented data association variable $c_k^i$:
\begin{align}
    &t(\mathcal{Z}_k^i|\mathbf{s}_k,\tilde{\mathcal{X}}_k^i,c_k^i) \label{eq:likelihoodBer} \\
    & =
    \begin{cases}
        p_\mathrm{D}(\mathbf{s}_k,\mathbf{x}_k^i)g(\mathbf{z}_k^j|\mathbf{s}_k,\mathbf{x}_k^i)\delta_i[u_k], & \mathcal{Z}_k^i=\{\mathbf{z}_k^j \},~c_k^i=j,\\
        & \tilde{\mathcal{X}}_k^i=\{(u_k,\mathbf{x}_k^i)\}\\
        (1-\mathtt{p}_\mathrm{D}(\mathbf{s}_k,\mathbf{x}_k^i))\delta_i[u_k], & \mathcal{Z}_k^i = \emptyset,~c_k^i = 0,\\
        &\tilde{\mathcal{X}}_k^i=\{(u_k,\mathbf{x}_k^i)\}\\
        1, & \mathcal{Z}_k^i = \emptyset,~c_k^i =0,\\
        & \tilde{\mathcal{X}}_k^i=\emptyset\\
        0, & \mathrm{otherwise}.
    \end{cases}\notag
\end{align}

The factor $[1-p_\mathrm{D}(\mathbf{s}_k,\cdot)]^{\tilde{\mathcal{X}}_{k}^{\mathrm{U}}}$ describes the set of targets that remain undetected.
\end{itemize}

    Using the factorized density in~\eqref{eq:factorizedDensity}, we depict the factor graph for update, as shown in  Fig.~\ref{Fig:FG} (green area).
    Here we use the following notations for the factor nodes:
\begin{subequations}
\begin{align}
    &f^F( \tilde{\mathcal{X}}_{k}^\mathrm{U}, \tilde{\mathcal{P}}_k^1,\cdots ,\tilde{\mathcal{P}}_k^{J_k})  \triangleq \delta_{ \tilde{\mathcal{X}}_{k}^\mathrm{U} \uplus_{j=1}^{J_k} \tilde{\mathcal{P}}_k^j }(\tilde{\mathcal{P}}_{k}^\mathrm{U}),\\
    &f_j^G(\tilde{\mathcal{P}}_k^j,\tilde{\mathcal{Y}}_k^j)  \triangleq \delta_{\mathtt{h}_{-(I_{k-1}+j)}(\tilde{\mathcal{Y}}_k^j)}(\tilde{\mathcal{P}}_k^j), \\
    &f^H(\mathbf{s}_k,\tilde{\mathcal{X}}_k^\mathrm{U}) \triangleq [1-p_\mathrm{D}(\mathbf{s}_k,\cdot)]^{\tilde{\mathcal{X}}_k^\mathrm{U} },\\
    &f_j^I(\mathbf{s}_k,\tilde{\mathcal{Y}}_k^j,d_k^j)  \triangleq \tilde{l}(\mathbf{z}_k^j|\mathbf{s}_k,\tilde{\mathcal{Y}}_k^j,d_k^j),\\
    &f_i^J(\mathbf{s}_k,\tilde{\mathcal{X}}_k^i,c_k^i)  \triangleq t(\mathcal{Z}_k^i|\mathbf{s}_k,\tilde{\mathcal{X}}_k^i,c_k^i),\\ 
    &f_{i,j}^K(c_{k}^i,d_{k}^j)  \triangleq \psi(c_{k}^i,d_{k}^j).
\end{align}
\end{subequations}
We compute the messages and beliefs on the update factor graph in Fig.~\ref{Fig:FG} by running developed set-type \ac{bp}, detailed in Appendix~\ref{SM:Messages} of the supplementary material.
    Messages are not sent backward in time. The iterative message is only performed for the data association step, whereas for other steps a single message passing is performed.
As a final result, the message \fbox{\footnotesize{27}} represents the posterior of the sensor state at the end of time step $k$, \fbox{\footnotesize{22}} represents the posterior PPP of undetected targets, \fbox{\footnotesize{18}} represents the posteriors of the Bernoullis of each newly detected target, and \fbox{\footnotesize{20}} represents the posteriors of the Bernoullis of each previously detected target. These posteriors will be used when creating the factor graph connecting time step $k$ with time step $k+1$. 
Note that the number of detected targets grows over time. 

\subsection{Connecting the Factor Graphs in Prediction and Update}

    So far, we have discussed the prediction and update factor graphs independently. We now proceed to explain how these factor graphs are connected. The factor graph in the prediction~(see, yellow part in Fig.~\ref{Fig:FG}) includes variable nodes $\mathbf{s}_k,\tilde{\mathcal{P}}_k^\mathrm{U},\tilde{\mathcal{X}}_k^{1},\dots,\tilde{\mathcal{X}}_k^{I_{k-1}}$
while the factor graph in the update (see, green and purple parts in Fig.~\ref{Fig:FG})
contains variables $\mathbf{s}_k,\tilde{\mathcal{X}}_k^\mathrm{U},\tilde{\mathcal{P}}_k^{1},\dots,\tilde{\mathcal{P}}_k^{J_k},\tilde{\mathcal{X}}_k^{1},\dots,\tilde{\mathcal{X}}_k^{I_{k-1}}$.
We note that the variables $\mathbf{s}_k,\tilde{\mathcal{X}}_k^{1},\dots,\tilde{\mathcal{X}}_k^{I_{k-1}}$ are the same. To connect the other variable nodes in the prediction and update factor graphs, we can use a merging and partitioning factor~(see, Definition~\ref{def:PartitionFactor}) as well as a conversion factor for auxiliary variables~(see, Definition~\ref{def:ConvertFactor})
for newly detected Bernoullis. 
    That is, we first adopt the factor $f^F(\tilde{\mathcal{P}}_k^\mathrm{U},\tilde{\mathcal{X}}_k^\mathrm{U},\tilde{\mathcal{P}}_k^{1},\dots,\tilde{\mathcal{P}}_k^{J_k})$,
    which partitions the set $\tilde{\mathcal{P}}_k^\mathrm{U}$ into subsets $\tilde{\mathcal{X}}_k^\mathrm{U},\tilde{\mathcal{Y}}_k^{1},\dots,\tilde{\mathcal{Y}}_k^{J_k}$ (all with auxiliary variable 0). 
    Then, to each of the newly detected Bernoullis, we apply the conversion factor for auxiliary variables $f_j^G(\tilde{\mathcal{P}}_k^j,\tilde{\mathcal{Y}}_k^{j})$.
    Considering the set-variables that are linked to the set-factors $f^F(\cdot)$ and $f_j^G(\cdot)$, see Fig.~\ref{Fig:FG},
    we know that the set-variable $\tilde{\mathcal{Y}}_k^{j}$ and ${\mathcal{X}}_k^\mathrm{U}$ follow a Poisson process with $u = I_{k-1}+j$ and $u=0$, respectively.
    Finally, the factor graphs for prediction and update are connected by the factors~($f^A(\cdot)$, $f^D(\cdot)$, and $f_i^E(\cdot)$) and their linked variables.

\begin{rem}[Factor Graphs for Smoothing at the Previous Time Step]
    Fig.~\ref{Fig:FG} shows the concatenation of factor graphs for the prediction and update steps at time $k-1$ and $k$. 
    It is possible to derive a joint prediction and update factor graph (instead of its concatenation). 
    This factor graph would require the use of \acp{pmb} for sets of trajectories between time $k-1$ and $k$~\cite{Angel_Trajectory_TSP2020,Angel_MTTT_TAES2020}.
    Furthermore, this factor graph would enable us to infer the state at time $k-1$ of a Bernoulli created at time $k$, via smoothing. 
\end{rem}

    \begin{rem}[Special Cases and Generalizations]
    The factorization~(\eqref{eq:P_factorizedDensity} and~\eqref{eq:factorizedDensity}) and the corresponding factor graph in Fig.~\ref{Fig:FG} can be specialized and generalized to cover a variety of applications. Some special cases include: 
    \begin{itemize}
        \item Mapping and SLAM are obtained when the targets have no mobility, which is obtained when the transition density is $f(\mathbf{x}_k|\mathbf{x}_{k-1})=\delta(\mathbf{x}_k-\mathbf{x}_{k-1})$.
        \item Mapping and MTT are obtained when then sensor state is known at all times, so that $f_{\mathtt{u}}(\mathbf{s}_{k-1})$ and $f(\mathbf{s}_{k}|\mathbf{s}_{k-1})$ are removed and $\bm{s}_{k-1}$ and $\bm{s}_{k}$ no longer explicitly appear as vertices in the factor graph (instead they are absorbed as parameters in the likelihood functions).        
    \end{itemize}
    \end{rem}

    \begin{rem}[Set-Type BP MB Filter]
        To obtain a set-type \ac{bp} MB filter (whether labeled or not),
        we use the same modeling assumptions as in the set-type \ac{bp} \ac{pmb} filter except that the birth model is \ac{mb} instead of Poisson. Then, we set the intensity of \ac{ppp} to zero and add the Bernoulli components of the birth process in the prediction step~\cite{Angel_MBM_Fusion2019}.  
        The corresponding factor graph is equivalent to the one in Fig.~\ref{Fig:FG} but removing all the \ac{ppp} variables ($\mathcal{X}_{k-1}^\mathrm{U},\mathcal{X}_{k}^\mathrm{S},\mathcal{X}_{k}^\mathrm{B},\mathcal{P}_{k}^\mathrm{U},\mathcal{P}_{k}^1,\dots,\mathcal{P}_{k}^{J_k},\mathcal{X}_{k}^\mathrm{U}$), their connected factors ($f_\mathsf{u}^\mathrm{U},f^B,f^C,f^D,f^F,f_1^G,\dots,f_{J_k}^G$), 
        and newly detected targets ($\mathcal{Y}_k^1,\dots,\mathcal{Y}_k^{J_k}$), 
        and also adding the new birth Bernoulli variables and factors in the prediction step.
    \end{rem}


\begin{rem}
[On the Use of Set-Type BP with Other RFS Filters]
It should be noted that set-type BP can be used to derive \ac{rfs} filters based on computing marginal distributions for targets, such as \ac{pmb} and \ac{mb} filters, for general measurement and clutter models, including coexisting point and extended targets \cite{10130623}.
It is not suitable to derive \ac{phd} and \ac{cphd} filters as these filters do not calculate marginals but approximate the posterior as a \ac{ppp} or as an independent and identically distributed cluster process~\cite{Mahler_book2014}. Set-type BP is also not suitable for implementing
filters based on conjugate priors such as \ac{pmbm} and \ac{glmb} filters.
\end{rem}

    
\subsection{Approximate KLD Minimization of Set-Type BP PMB for SLAM}

    Assuming that the prior is a \ac{pmb}, the prediction of the set-type \ac{bp} is represented in closed-form. As a direct application of the \ac{pmbm} update~\cite[Sec.~IV-B]{Hyowon_MPMB_TVT2022}, the updated density is $\tilde{f}_\mathsf{u}^\mathrm{PMBM}(\tilde{\mathcal{X}}_k|\mathbf{s}_k)f_\mathsf{u}(\mathbf{s}_k)$, where $\tilde{f}_\mathsf{u}^\mathrm{PMBM}(\tilde{\mathcal{X}}_k|\mathbf{s}_k)$ corresponds to the \ac{pmbm}
    with the auxiliary variables~\cite[Sec.~III-A]{Angel_MTTT_TAES2020}, given the sensor state $\mathbf{s}_k$, where $\tilde{\mathcal{X}}_k = \tilde{\mathcal{X}}_k^\mathrm{U} \uplus \tilde{\mathcal{X}}_k^{1:I_{k-1}} \uplus \tilde{\mathcal{Y}}_k^{1:J_k}$~(see, \eqref{eq:factorizedDensity}).
\begin{lemma}
    The marginals of $\mathbf{s}_k$, $\tilde{\mathcal{X}}_k^\mathrm{U}$, $\tilde{\mathcal{X}}_k^i$, $\tilde{\mathcal{Y}}_k^j$ of~\eqref{eq:factorizedDensity} represent a \ac{pmb} that is an optimal approximation of the corresponding \ac{pmbm} posterior in the sense that it minimizes the \ac{kld}~$D(\tilde{f}_\mathsf{u}^\mathrm{PMBM}(\tilde{\mathcal{X}}_k|\mathbf{s}_k)f_\mathsf{u}(\mathbf{s}_k)||\tilde{q}_\mathtt{u}^\mathrm{PMB}(\tilde{\mathcal{X}})q_\mathtt{u}(\mathbf{s}_k))$.
    The marginal densities of $\mathbf{s}_k$, $\tilde{\mathcal{X}}_k^\mathrm{U}$, $\tilde{\mathcal{X}}_k^i$, $\tilde{\mathcal{Y}}_k^j$ are given by
\begin{align}
    \tilde{q}_\mathtt{u}^\mathrm{PMB}(\tilde{\mathcal{X}}) &= \tilde{q}_\mathtt{u}^{\mathrm{U}}(\tilde{\mathcal{X}}^\mathrm{U})\prod_{i=1}^{n}\tilde{q}_\mathtt{u}^{i}(\tilde{\mathcal{X}}^i), \\ 
    q_\mathtt{u}(\mathbf{s}_k) &= f_\mathtt{u}(\mathbf{s}_k), \\
    q_\mathtt{u}(\tilde{\mathcal{X}}_k^\mathrm{U}) &= \int f_\mathtt{u}(\tilde{\mathcal{X}}_k^\mathrm{U}|\mathbf{s}_k) f_\mathtt{u}(\mathbf{s}_k) \mathrm{d}\mathbf{s}_k,\\
    q_\mathtt{u}(\tilde{\mathcal{X}}_k^i) &= \int f_\mathtt{u}(\tilde{\mathcal{X}}_k^i|\mathbf{s}_k) f_\mathtt{u}(\mathbf{s}_k) \mathrm{d}\mathbf{s}_k,\\
    q_\mathtt{u}(\tilde{\mathcal{Y}}_k^j) &= \int f_\mathtt{u}(\tilde{\mathcal{Y}}_k^j|\mathbf{s}_k) f_\mathtt{u}(\mathbf{s}_k) \mathrm{d}\mathbf{s}_k.
\end{align}
    Then, upon convergence, the set-type \ac{bp} applied to \eqref{eq:factorizedDensity} approximates these marginal densities by producing the interior points of the constrained Bethe free energy on the factor graph, see Theorem~\ref{theo:DerSetBP}.
\end{lemma}
\begin{proof}
    The proof of the \ac{kld} minimization is an extension of Proposition~1 of~\cite{Angel_Trajectory_TSP2020}, showing the optimality of the solution. The Bethe free minimization is a result of Theorem~\ref{theo:DerSetBP}.
\end{proof}

\subsection{Exploring the Relationships: Set-Type and Vector-Type BP PMB-SLAM Filters}
\label{sec:Connections}
    
    We now reveal connections between the proposed set-type \ac{bp} filter to the vector-type \ac{bp} filters~\cite{Erik_AOABPSLAM_ICC2019}.
    The proposed set-type BP PMB-SLAM filter turned out to be algorithmically identical to the vector-type BP-SLAM filters~\cite{Erik_AOABPSLAM_ICC2019}, except for minor details.\footnote{Note that the vector-type BP-SLAM filter in~\cite{Erik_BPSLAM_TWC2019} mentions the possibility of using an external PHD filter, but it is not explained in detail, so we focus on the comparison with~\cite{Erik_AOABPSLAM_ICC2019}, as it has an explicit description of the external PHD filter.}
    We note that the connection between vector-type and RFS-based method was partially discussed from the perspective of the expression of target densities and data association in~\cite[Sec.~XIII-A]{Florian_BP-MTT_Proc2018} and \cite[Sec.~IV-E]{Hyowon_MPMB_TVT2022}.
    

\subsubsection{Previously Detected Targets}
    Bernoulli densities in the proposed set-type BP-SLAM filter can be parameterized using a state $\ensuremath{\mathbf{y}=[\mathbf{x}^{\top},\epsilon]^{\top}}$, where $\ensuremath{\epsilon}$ is the existence variable such that $\ensuremath{\epsilon}=1$ if the target exists and $\ensuremath{\epsilon}=0$ if the target does not exist. A Bernoulli density in vector representation is $f\left(\mathbf{x},\epsilon\right)=f\left(\epsilon\right)f\left(\mathbf{x}|\epsilon\right)$ where $f\left(\mathbf{x},\epsilon=1\right)=f\left(\left\{ \mathbf{x}\right\} \right)$, and $f\left(\mathbf{x},\epsilon=0\right)=f\left(\emptyset\right)$.
    The vector representation has the theoretical drawback of requiring a dummy density $f\left(\mathbf{x}|\epsilon=0\right)$ which does not exist, though this does not affect practical results. The same procedure can be followed to define equivalent conditional densities (factors) that consider sets of at maximum one element.
    With this in mind, we can write the part of the factor graph in Fig.~\ref{Fig:FG} which has Bernoulli densities using the existence variable representation.
    It is also straightforward to express the parts that involve variables that have vector densities (such as $\mathbf{s}_{k}$). 
    
\subsubsection{Undetected Targets}
    
    As the undetected targets are PPP distributed, this part of the factor graph, including the factors $f^{B}$ (for prediction) and $f^{H}$ (for update) cannot be written with existence variables. 
    Therefore, the analogous factor graph with existence variables has to drop this source of information.
    One solution to compensate for this loss of information in vector-type BP is to use an external PHD filter to model undetected features, as in~\cite{Erik_AOABPSLAM_ICC2019}. The prediction step of the external PHD filter is designed to match the solution obtained by the (prediction) factor graph in Fig.~\ref{Fig:FG}.
    
\subsubsection{Newly Detected Targets}
    In the update, the external PHD filter can be used to describe the density of the newly detected Bernoullis.
    In order to link this PHD filter with the factor graph, reference~\cite{Erik_AOABPSLAM_ICC2019} requires direct modeling of the density of newly detected features. This density corresponds to the predicted intensity of the PHD filter multiplied by the probability of detection, normalized to integrate to one in~\cite[eq.~(16)]{Erik_AOABPSLAM_ICC2019}.

\subsubsection{Detection Probability}
    Vector-type \ac{bp} \ac{slam} requires that the detection probability for undetected targets does not depend on the sensor state, such that $\ensuremath{p_{\mathrm{D}}(\mathbf{s}_{k},\mathbf{x}_{k})=p_{\mathrm{D}}(\mathbf{x}_{k})}$, as in~\cite{Erik_AOABPSLAM_ICC2019}. The reason why this is required is that if $p_{\mathrm{D}}(\mathbf{s}_{k},\mathbf{x}_{k})$ depends on $\mathbf{s}_{k}$ and $\mathbf{x}_{k}$, the detection probability should be part of the factor that connects each newly detected feature and the sensor state, as in Fig.~\ref{Fig:FG}. This could be fixed in vector-type BP by defining the density of undetected features (as the normalized PHD intensity) and then considering $p_{\mathrm{D}}(\mathbf{s}_{k},\mathbf{x}_{k})$ in the factor, as in Fig.~\ref{Fig:FG}.

\subsubsection{Sensor State Update}
    It is also relevant to mention that, while the factor graphs in~\cite{Erik_BPSLAM_TWC2019,Erik_AOABPSLAM_ICC2019} consider messages from the sensor state to the newly detected features, the belief update equation of the sensor state does not have these messages, see~\cite[eq.~(37)]{Erik_BPSLAM_TWC2019} and ~\cite[Sec.~III-G]{Erik_AOABPSLAM_ICC2019}.
    The equations with the messages from newly detected features to the sensor state are not considered in~\cite{Erik_BPSLAM_TWC2019,Erik_AOABPSLAM_ICC2019} either. 
    These equations can be corrected to obtain a vector-type BP-SLAM  algorithm similar to the set-type BP-SLAM algorithm presented in this paper.

\subsubsection{Derivations}
    It is also important to realize that this paper provides an alternative derivation of the (slightly modified) vector-type BP-SLAM algorithm in \cite{Erik_AOABPSLAM_ICC2019} obtained from first principles using RFSs. That is, vector-type BP-SLAM is obtained assuming models for single-feature dynamics, single-feature death, sensor dynamics, and measurement models, but no feature births. Instead, it is assumed that, at the current time step, we know the multi-Bernoulli density of previous features and the multi-Bernoulli density of newly detected features. Then, the factor graph is obtained. Afterwards, an external PHD filter can be used in vector-based SLAM to model undetected features.

    Instead, the set-type BP PMB-SLAM derivation works as follows. We require models for feature birth, single-feature dynamics, single-feature death, sensor dynamics, and measurements, similar to Bayesian multi-target tracking algorithms. Then, auxiliary variables are used to obtain the associated factor graph, and then the factor graph is solved using set-BP. The resulting algorithm, with the minor modifications explained above, is similar to vector-BP SLAM with the external PHD filter in~\cite{Erik_AOABPSLAM_ICC2019}.
    The connection is clearer by writing each Bernoulli component in terms of existence variables. This shows that the vector-BP SLAM algorithm with the external PHD filter, with the minor modifications above, can be derived from RFS theory, which is an important result for the literature.
    It is also relevant to note that the set-BP MB-SLAM algorithm that we also propose could also be written with existence variables.

\begin{figure}[t!]
\begin{centering}
	\includegraphics[width=.6\columnwidth]{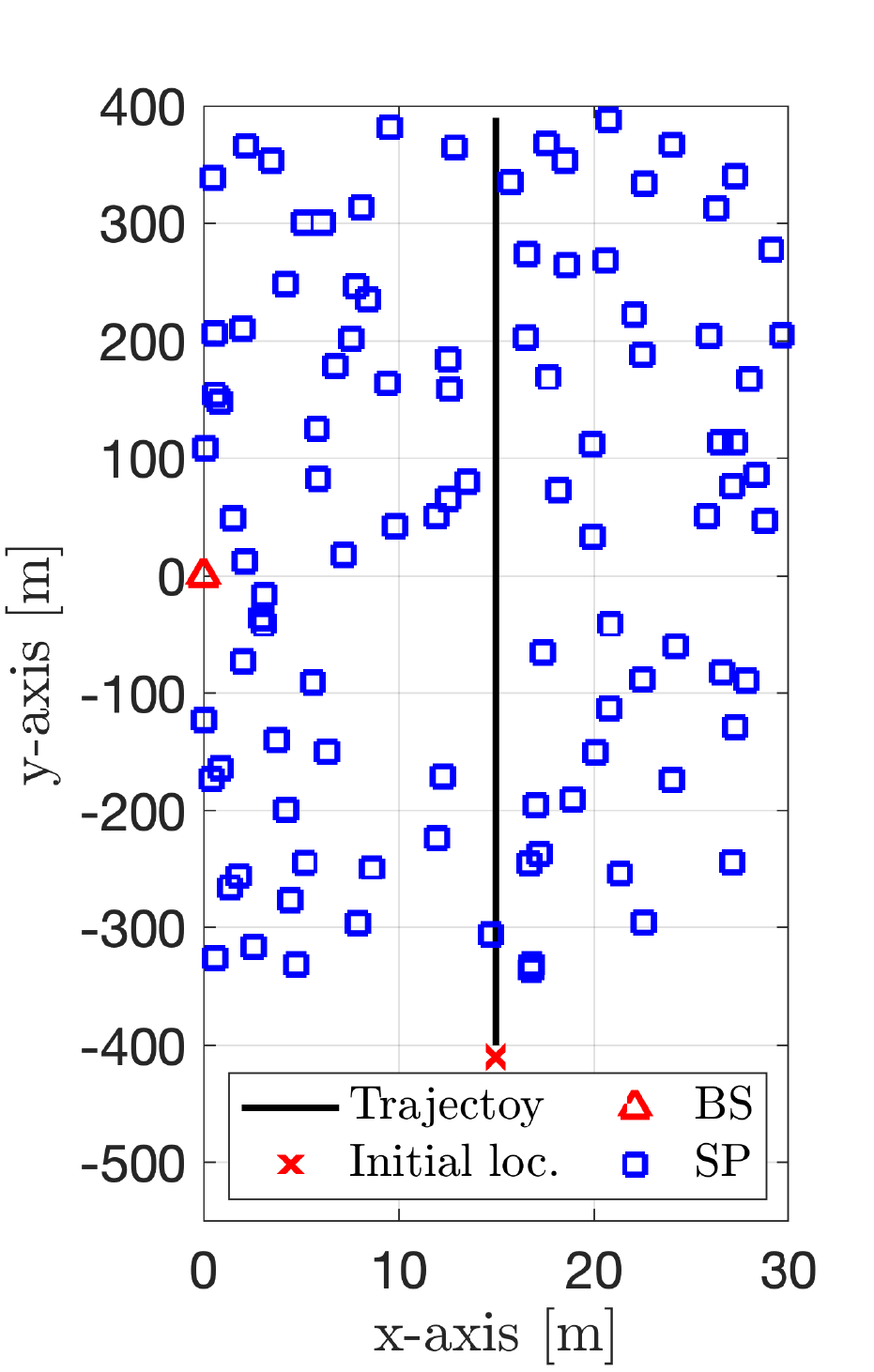}
	\caption{Bistatic radio SLAM scenario, where 176 SPs are uniformly distributed. A single BS transmits the signal, and the sensor receives the two types of measurements: BS-sensor; and BS-SPs-sensor.}
	\label{Fig:Deployment}
\end{centering}
\end{figure}
\section{Numerical Results} \label{sec:NumericalResults}

    In this section, we analyze the proposed set-type \ac{bp} \ac{pmb}-\ac{slam} and \ac{bp} \ac{mb}-\ac{slam} filters, abbreviated as SBP-PMB-SLAM and SBP-MB-SLAM, respectively, in comparison with \ac{bp}-\ac{slam}~\cite{Erik_AOABPSLAM_ICC2019}, abbreviated as VBP-SLAM1.
    Note that the vector-type BP-SLAM filters~\cite{Erik_BPSLAM_TWC2019,Erik_AOABPSLAM_ICC2019} with the minor modifications, abbreviated as VBP-SLAM2, is equivalent to the SBP-PMB-SLAM, as discussed in Sec.~\ref{sec:Connections}.
    We introduce the simulation setup for evaluating the \ac{slam} filters, and subsequently the results are discussed.

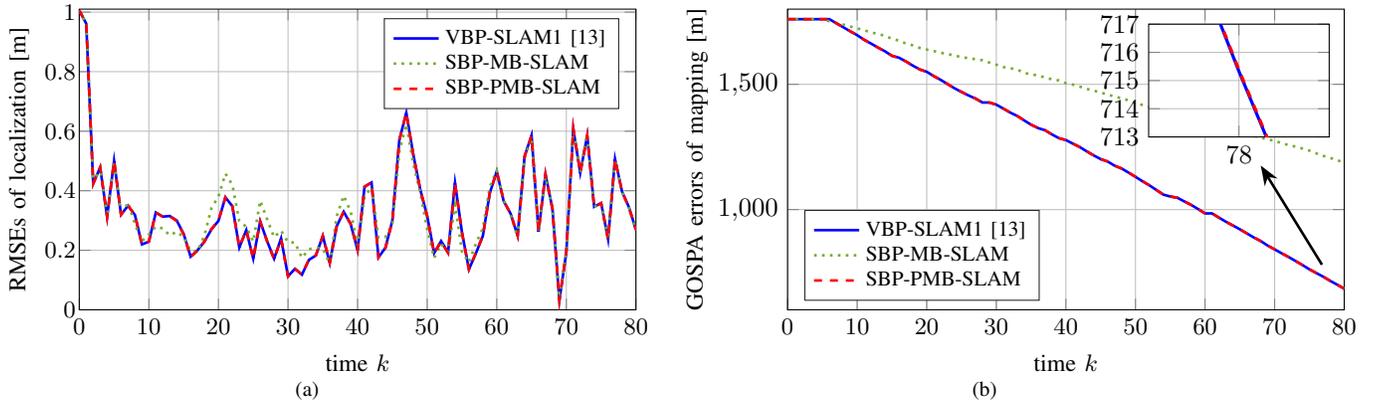
\begin{figure*}[t!]
\begin{center}
\captionsetup[subfigure]{aboveskip=0mm,belowskip=0mm}
\hspace{-10mm}
\begin{subfigure}{0.45\textwidth}
    {
%
%
\definecolor{mycolor1}{rgb}{0.46667,0.67451,0.18824}%

\begin{tikzpicture}

\begin{axis}[%
width=74mm,
height=40mm,
at={(0mm,0mm)},
scale only axis,
xmin=0,
xmax=80,
ymin=0,
ymax=1.01,
yticklabel style = {font=\small,xshift=0.5ex},
xticklabel style = {font=\small,yshift=0ex},
axis background/.style={fill=white},
axis background/.style={fill=white},
xmajorgrids,
ymajorgrids,
legend style={legend pos=north east, legend cell align=left, align=left, draw=white!15!black,style={row sep=-0.05cm}}
]

\addplot [color=blue, line width=1.0pt,  mark options={solid, blue}]
  table[row sep=crcr]{%
0	1.00553531435301\\
1	0.961277566477467\\
2	0.423923885835415\\
3	0.478329489394731\\
4	0.312553605415685\\
5	0.500062465349316\\
6	0.319540387296896\\
7	0.349697417650881\\
8	0.318244635072366\\
9	0.219597613240075\\
10	0.22895741232891\\
11	0.326817653814621\\
12	0.312935926588618\\
13	0.315125148575708\\
14	0.300042510155965\\
15	0.254932058291517\\
16	0.178569577889995\\
17	0.200464163135892\\
18	0.229704753533739\\
19	0.270464478366302\\
20	0.298803528839905\\
21	0.377922532407106\\
22	0.348103420224772\\
23	0.209207914565441\\
24	0.26764023906411\\
25	0.172004721951046\\
26	0.29681879380084\\
27	0.230294037910688\\
28	0.171944575260809\\
29	0.239830004061695\\
30	0.112804643832594\\
31	0.13774278452507\\
32	0.118026803639434\\
33	0.167738891854655\\
34	0.182382531819105\\
35	0.247783398392162\\
36	0.157606537950515\\
37	0.281111317962166\\
38	0.32936593623638\\
39	0.288067639862117\\
40	0.200866656153624\\
41	0.413572279312537\\
42	0.427481589243929\\
43	0.176216032431621\\
44	0.208541995742455\\
45	0.300165220399171\\
46	0.570910634325096\\
47	0.661363173100724\\
48	0.521530767110602\\
49	0.404619936755221\\
50	0.314574035801129\\
51	0.190723833373086\\
52	0.230895021361062\\
53	0.192818887975526\\
54	0.426015617321368\\
55	0.260428132915663\\
56	0.136049939236155\\
57	0.189489420709035\\
58	0.248270433501162\\
59	0.392378846123567\\
60	0.464433778418254\\
61	0.364667277316558\\
62	0.322482705494246\\
63	0.249281887626648\\
64	0.515884556188497\\
65	0.58597689995556\\
66	0.262346142492536\\
67	0.455185680729238\\
68	0.340224203068349\\
69	0.0317129107827933\\
70	0.193799066382464\\
71	0.60648392036811\\
72	0.460233775628275\\
73	0.584591687802313\\
74	0.346048275182309\\
75	0.358624381881545\\
76	0.23780004197784\\
77	0.506026261353326\\
78	0.397777915465984\\
79	0.3445776368089\\
80	0.26941730632728\\
};
\addlegendentry{\footnotesize VBP-SLAM1~\cite{Erik_AOABPSLAM_ICC2019}}

\addplot [color=mycolor1, dotted, line width=1.0pt, mark options={dotted, mycolor1}]
  table[row sep=crcr]{%
0	1.00553531435301\\
1	0.961061845725181\\
2	0.424303398894065\\
3	0.47835326679005\\
4	0.313240648008065\\
5	0.500413673828897\\
6	0.314275025856269\\
7	0.361393026834775\\
8	0.287376870055599\\
9	0.230043107584014\\
10	0.25313542002803\\
11	0.279647319725172\\
12	0.273167270492224\\
13	0.250244856661792\\
14	0.262155500235376\\
15	0.24440390739693\\
16	0.197756038729636\\
17	0.19677335685799\\
18	0.25056922363879\\
19	0.342227710217231\\
20	0.380048468063534\\
21	0.45738047059912\\
22	0.419860513439284\\
23	0.29733782001985\\
24	0.278986480632057\\
25	0.229130341640535\\
26	0.36638157100086\\
27	0.304052530889497\\
28	0.243828466017482\\
29	0.26614304612457\\
30	0.226957363405428\\
31	0.225828586163168\\
32	0.174609369852786\\
33	0.203721863636231\\
34	0.209145573322593\\
35	0.205334908044768\\
36	0.154279821248489\\
37	0.314506187240558\\
38	0.383337493927816\\
39	0.330379167410671\\
40	0.21758765885469\\
41	0.380266864062257\\
42	0.41173416298775\\
43	0.243332619890271\\
44	0.244408398036692\\
45	0.289565885079505\\
46	0.532153829081494\\
47	0.623374706925198\\
48	0.490368332492808\\
49	0.394020991504765\\
50	0.271126008189715\\
51	0.173935758599702\\
52	0.220988291480129\\
53	0.194705650263885\\
54	0.352605371697423\\
55	0.195612568915037\\
56	0.15871822995536\\
57	0.222575561347296\\
58	0.272944425570537\\
59	0.413419986186264\\
60	0.467976653044477\\
61	0.367658179188161\\
62	0.318348355633784\\
63	0.243508997985291\\
64	0.508727036826356\\
65	0.577017232611046\\
66	0.262097130024309\\
67	0.445744811577465\\
68	0.338942132895044\\
69	0.0396847682898542\\
70	0.197157588711055\\
71	0.610687456011515\\
72	0.462860536399896\\
73	0.588771258738897\\
74	0.345466388952141\\
75	0.358149391431518\\
76	0.237747818435077\\
77	0.505626801290799\\
78	0.396341253539036\\
79	0.343163172536987\\
80	0.267228115064162\\
};
\addlegendentry{\footnotesize SBP-MB-SLAM}

\addplot [color=red, dashed, line width=1.0pt, mark options={dash, red}]
  table[row sep=crcr]{%
0	1.00553531435301\\
1	0.961277566477467\\
2	0.423923885835415\\
3	0.478329489394731\\
4	0.312553605415685\\
5	0.500062465349316\\
6	0.320243594218261\\
7	0.350371910199738\\
8	0.318404202708449\\
9	0.219094994323702\\
10	0.229574084181025\\
11	0.326593492057591\\
12	0.312788770226479\\
13	0.314628963209573\\
14	0.2981765680299\\
15	0.253435396104839\\
16	0.176953627526775\\
17	0.198907598879422\\
18	0.229544211988996\\
19	0.271535941081153\\
20	0.299418743928993\\
21	0.377647837129194\\
22	0.348787951240992\\
23	0.210504392410569\\
24	0.267819347631334\\
25	0.172302407559483\\
26	0.296411019095685\\
27	0.22935785132872\\
28	0.169558255212385\\
29	0.237698394503109\\
30	0.111607391365352\\
31	0.136711428725855\\
32	0.118987308544998\\
33	0.16847289358831\\
34	0.183726518933981\\
35	0.248467991600938\\
36	0.158842448296847\\
37	0.280868256619607\\
38	0.329795958786118\\
39	0.288391326246912\\
40	0.201971494004942\\
41	0.413320279855906\\
42	0.428003267345312\\
43	0.1755642304252\\
44	0.207824464020366\\
45	0.298957349004609\\
46	0.57055601987251\\
47	0.661799921046075\\
48	0.522137601204153\\
49	0.40529004555374\\
50	0.315432407880001\\
51	0.189742989257281\\
52	0.228311972499229\\
53	0.191009493965249\\
54	0.425241508915692\\
55	0.259004774900963\\
56	0.136149855353067\\
57	0.191086014302375\\
58	0.248283445350075\\
59	0.393168075414247\\
60	0.465191283150759\\
61	0.365680132937151\\
62	0.32273218203714\\
63	0.249179963527899\\
64	0.515253441383536\\
65	0.58507467365952\\
66	0.263804877455077\\
67	0.454848224648121\\
68	0.340729663642764\\
69	0.0331380974882232\\
70	0.195526691771278\\
71	0.609103102455761\\
72	0.460766214762738\\
73	0.584821469625411\\
74	0.345423329978415\\
75	0.35789313396628\\
76	0.236726436915178\\
77	0.505697715321933\\
78	0.397132214650123\\
79	0.34371656613828\\
80	0.268556727669662\\
};
\addlegendentry{\footnotesize SBP-PMB-SLAM}
\end{axis}

\node[rotate=0,fill=white] (BOC6) at (3.7cm,-.7cm){\small time $k$};
\node[rotate=90] at (-8mm,20mm){\small RMSEs of localization [m]};
\end{tikzpicture}%
}
    \subcaption{}
    \label{Fig:UnInfLoc}
\end{subfigure}
\hspace{6mm}
\begin{subfigure}{0.45\textwidth}
    {
%
%
\definecolor{mycolor1}{rgb}{0.46667,0.67451,0.18824}%

\begin{tikzpicture}

\begin{axis}[%
width=74mm,
height=40mm,
at={(0mm,0mm)},
scale only axis,
xmin=0,
xmax=80,
ymin=600,
ymax=1800,
xticklabel style = {font=\small,yshift=0ex},
axis background/.style={fill=white},
axis background/.style={fill=white},
xmajorgrids,
ymajorgrids,
legend style={legend pos=south west, legend cell align=left, align=left, draw=white!15!black,style={row sep=-0.05cm}}
]

\addplot [color=blue, line width=1.0pt, mark options={solid, blue}]
  table[row sep=crcr]{%
0	1760\\
1	1760\\
2	1760\\
3	1760\\
4	1760\\
5	1760\\
6	1760\\
7	1743.82823318986\\
8	1727.83367715586\\
9	1711.67259557143\\
10	1696.1865472578\\
11	1677.73009870767\\
12	1662.5340844606\\
13	1646.45919485194\\
14	1631.66389576973\\
15	1613.70396221633\\
16	1606.22214175336\\
17	1591.41477117943\\
18	1574.51048661281\\
19	1559.2672857542\\
20	1549.99508471961\\
21	1532.87441407618\\
22	1517.68849865117\\
23	1500.81120008573\\
24	1483.99701443364\\
25	1469.40876277089\\
26	1455.53783912785\\
27	1443.30290814497\\
28	1427.2847550933\\
29	1427.2847550933\\
30	1418.73961443551\\
31	1403.16943601421\\
32	1387.23401046497\\
33	1373.7967847372\\
34	1355.94464415377\\
35	1338.71042682233\\
36	1325.30854001838\\
37	1316.71419160283\\
38	1299.91151531437\\
39	1284.95505863635\\
40	1277.47139513197\\
41	1263.13131388809\\
42	1249.47987958\\
43	1232.34524724366\\
44	1216.1179820315\\
45	1201.43132514449\\
46	1191.80082818639\\
47	1175.18560960332\\
48	1164.06279105431\\
49	1147.64129858923\\
50	1130.71669280104\\
51	1113.22849564605\\
52	1096.72813579785\\
53	1078.38683239812\\
54	1059.82270355648\\
55	1052.07481941358\\
56	1047.40994996445\\
57	1032.93355202788\\
58	1018.68468630519\\
59	1002.07251360262\\
60	984.756619883428\\
61	984.756619883428\\
62	967.815799140841\\
63	951.374245513098\\
64	936.445105915641\\
65	921.430061810625\\
66	904.942884113375\\
67	888.832110254754\\
68	874.381553060751\\
69	856.291712041512\\
70	841.177540625868\\
71	825.569245502176\\
72	810.849111027445\\
73	794.853702875364\\
74	778.62540689506\\
75	761.393317973661\\
76	746.609278508881\\
77	731.675938729291\\
78	715.326729019646\\
79	699.995713878304\\
80	684.387199416106\\
};
\addlegendentry{\footnotesize VBP-SLAM1~\cite{Erik_AOABPSLAM_ICC2019}}

\addplot [color=mycolor1, dotted, line width=1.0pt, mark options={dotted, mycolor1}]
  table[row sep=crcr]{%
0	1760\\
1	1760\\
2	1760\\
3	1760\\
4	1760\\
5	1760\\
6	1751.57768830396\\
7	1744.82879174527\\
8	1738.45104718906\\
9	1730.60910767022\\
10	1723.05274505677\\
11	1716.45880367089\\
12	1706.01314349481\\
13	1702.11599587732\\
14	1692.60319696511\\
15	1682.21800057681\\
16	1674.10613145687\\
17	1662.91762700323\\
18	1655.43837323027\\
19	1645.03895821773\\
20	1638.67039678038\\
21	1632.94800981819\\
22	1625.3142655547\\
23	1621.70919538996\\
24	1613.40849389134\\
25	1609.94737842203\\
26	1603.38804930623\\
27	1600.62602559204\\
28	1596.03817608033\\
29	1586.94506758997\\
30	1578.26344130584\\
31	1568.92258206307\\
32	1564.13993345897\\
33	1558.5639279493\\
34	1546.24156247827\\
35	1538.920338723\\
36	1536.09009585055\\
37	1526.61039563996\\
38	1520.12776045352\\
39	1515.39959071511\\
40	1504.3553545611\\
41	1498.85425333725\\
42	1490.65223619877\\
43	1481.62185580058\\
44	1474.81642951698\\
45	1467.26538846149\\
46	1458.75162200993\\
47	1451.52682805466\\
48	1444.15898070973\\
49	1436.01846885415\\
50	1425.79277569991\\
51	1417.57451270542\\
52	1409.83761759107\\
53	1405.98732502766\\
54	1401.38157595507\\
55	1393.85653914474\\
56	1383.47799574275\\
57	1372.21207319878\\
58	1361.54020643102\\
59	1355.873228509\\
60	1350.31135100917\\
61	1340.32460806143\\
62	1331.08024623636\\
63	1321.1393840789\\
64	1314.78111015122\\
65	1304.93702349699\\
66	1298.50032550775\\
67	1293.84971578598\\
68	1288.12302187983\\
69	1278.80053115849\\
70	1276.0100825471\\
71	1263.9077002929\\
72	1258.23501019058\\
73	1251.06981231503\\
74	1243.69928101378\\
75	1235.32575548666\\
76	1224.33096577622\\
77	1214.3990436558\\
78	1205.8959329708\\
79	1198.14147158701\\
80	1185.33152518742\\
};
\addlegendentry{\footnotesize SBP-MB-SLAM}

\addplot [color=red, dashed, line width=1.0pt, mark options={dash, red}]
  table[row sep=crcr]{%
0	1760\\
1	1760\\
2	1760\\
3	1760\\
4	1760\\
5	1760\\
6	1760\\
7	1743.82842838185\\
8	1727.83410874815\\
9	1711.67494010507\\
10	1696.18702665835\\
11	1677.72590017167\\
12	1662.52623422171\\
13	1646.44880382808\\
14	1631.65515623425\\
15	1613.70677927306\\
16	1606.22503084755\\
17	1591.40200028917\\
18	1574.50900784794\\
19	1559.27929111\\
20	1550.0038683309\\
21	1532.89785900911\\
22	1517.72697119553\\
23	1500.86051102947\\
24	1484.05804331136\\
25	1469.49171115995\\
26	1455.61785563904\\
27	1443.38571923025\\
28	1427.37855140121\\
29	1427.37855140121\\
30	1418.83799146243\\
31	1403.26427636392\\
32	1387.34105890524\\
33	1373.90318272031\\
34	1356.03573269412\\
35	1338.79533755754\\
36	1325.3932033067\\
37	1316.79565092076\\
38	1299.9900099834\\
39	1285.04391030376\\
40	1277.56384102001\\
41	1263.22469692886\\
42	1249.56281167769\\
43	1232.42391892904\\
44	1216.19532807074\\
45	1201.5030580949\\
46	1191.86685127851\\
47	1175.25013001702\\
48	1164.11685763219\\
49	1147.69297413996\\
50	1130.76450620967\\
51	1113.26895438177\\
52	1096.75901974365\\
53	1078.41176365819\\
54	1059.86250854363\\
55	1052.11626385696\\
56	1047.44566142306\\
57	1032.98990315498\\
58	1018.7170846606\\
59	1002.11569777623\\
60	984.791056853641\\
61	984.791056853641\\
62	967.871055206244\\
63	951.429318773339\\
64	936.510486087306\\
65	921.500293691692\\
66	904.994995601012\\
67	888.874213571586\\
68	874.422300394362\\
69	856.320635541354\\
70	841.201111189923\\
71	825.598383052689\\
72	810.872594820298\\
73	794.877000645362\\
74	778.654597919522\\
75	761.433968650088\\
76	746.65915138774\\
77	731.715505648981\\
78	715.381705063313\\
79	700.046028096467\\
80	684.456700676998\\
};
\addlegendentry{\footnotesize SBP-PMB-SLAM}
\end{axis}

\begin{axis}[%
width=24mm,
height=15mm,
at={(48mm,23mm)},
scale only axis,
xmin=77.5,
xmax=78.5,
ymin=713,
ymax=717,
xtick={78},
xticklabels={$78$},
axis background/.style={fill=white},
xmajorgrids,
ymajorgrids
]
\addplot [color=blue, line width=1.2pt, mark options={solid, blue}, forget plot]
  table[row sep=crcr]{%
71	825.569245502176\\
72	810.849111027445\\
73	794.853702875364\\
74	778.62540689506\\
75	761.393317973661\\
76	746.609278508881\\
77	731.675938729291\\
78	715.326729019646\\
79	699.995713878304\\
80	684.387199416106\\
};
\addplot [color=red, dashed, line width=1.2pt, mark options={dash, red}, forget plot]
  table[row sep=crcr]{%
71	825.598383052689\\
72	810.872594820298\\
73	794.877000645362\\
74	778.654597919522\\
75	761.433968650088\\
76	746.65915138774\\
77	731.715505648981\\
78	715.381705063313\\
79	700.046028096467\\
80	684.456700676998\\
};
\end{axis}

\begin{axis}[%
width=74mm,
height=40mm,
at={(0mm,0mm)},
scale only axis,
xmin=0,
xmax=1,
ymin=0,
ymax=1,
axis line style={draw=none},
ticks=none,
axis x line*=bottom,
axis y line*=left
]

\draw [black, line width=1.0pt] (axis cs:0.96,0.1) ellipse [x radius=4, y radius=6];
\draw[-{Stealth}, color=black, line width=1.0pt] (axis cs:0.96,0.15) -- (axis cs:0.85,0.48);
\end{axis}

\node[rotate=0,fill=white] (BOC6) at (3.7cm,-.7cm){\small time $k$};
\node[rotate=90] at (-12mm,20mm){\small GOSPA errors of mapping [m]};
\end{tikzpicture}%
}
    \subcaption{}
    \label{Fig:UnInfMap}
\end{subfigure}
\caption{Uninformative birth case. Comparison of the proposed 
set-type BP PMB-SLAM 
SBP-PMB-SLAM~(i.e., VBP-SLAM2)
with 
VBP-SLAM1~\cite{Erik_AOABPSLAM_ICC2019} and SBP-MB-SLAM:
(a) RMSEs of sensor localization and (b) GOSPA errors of landmark mapping.
}
\label{Fig:UnInf}
\par
\end{center}
\end{figure*}

\begin{figure*}[t!]
\begin{center}
\captionsetup[subfigure]{aboveskip=0mm,belowskip=0mm}
\hspace{-10mm}
\begin{subfigure}{0.45\textwidth}
    {
%
%
\definecolor{mycolor1}{rgb}{0.46667,0.67451,0.18824}%

\begin{tikzpicture}

\begin{axis}[%
width=74mm,
height=40mm,
at={(0mm,0mm)},
scale only axis,
xmin=0,
xmax=80,
ymin=0,
ymax=1,
yticklabel style = {font=\small,xshift=0.5ex},
xticklabel style = {font=\small,yshift=0ex},
axis background/.style={fill=white},
axis background/.style={fill=white},
xmajorgrids,
ymajorgrids,
legend style={legend pos=north west, legend cell align=left, align=left, draw=white!15!black,style={row sep=-0.05cm}}
]

\addplot [color=blue, line width=1.0pt,  mark options={solid, blue}]
  table[row sep=crcr]{%
0	0.981469839530703\\
1	0.966045971889366\\
2	0.425171294994916\\
3	0.482534075600437\\
4	0.311300067849925\\
5	0.502593385993497\\
6	0.271705805742475\\
7	0.269275614893491\\
8	0.118735766948632\\
9	0.129988224327035\\
10	0.273163884423714\\
11	0.272324477394183\\
12	0.216968474438064\\
13	0.198423226665183\\
14	0.153051989119172\\
15	0.102209961572484\\
16	0.0960434571669851\\
17	0.118323351850631\\
18	0.206345654808074\\
19	0.263413026151794\\
20	0.205064295404622\\
21	0.138204963526044\\
22	0.0948102545589977\\
23	0.197856398430636\\
24	0.234937401477526\\
25	0.273576489159725\\
26	0.201716608929077\\
27	0.167154880850377\\
28	0.365048865160184\\
29	0.358391597201572\\
30	0.257642313482724\\
31	0.226748214796236\\
32	0.250916167434391\\
33	0.158806069244227\\
34	0.0717780922496104\\
35	0.0536780248008144\\
36	0.139209825685165\\
37	0.314912611480303\\
38	0.286314366059122\\
39	0.241312772222002\\
40	0.132294199832931\\
41	0.249893994020203\\
42	0.234770958026451\\
43	0.24774873274658\\
44	0.227780302665765\\
45	0.218997460841421\\
46	0.117697274634343\\
47	0.244490613566682\\
48	0.17490031813038\\
49	0.0778366850843074\\
50	0.114513546479122\\
51	0.243740463373389\\
52	0.262765897534805\\
53	0.0993731675554777\\
54	0.11025909676547\\
55	0.147395509088341\\
56	0.167114823738376\\
57	0.249673374937917\\
58	0.122226186966901\\
59	0.168124785890646\\
60	0.123077160568778\\
61	0.136683021049357\\
62	0.100080679347129\\
63	0.135060934995968\\
64	0.259839822627696\\
65	0.266200131119923\\
66	0.189459753884121\\
67	0.196330100422132\\
68	0.179684553006343\\
69	0.112784730816039\\
70	0.121882209456976\\
71	0.155390944953208\\
72	0.214933387872515\\
73	0.200299756607877\\
74	0.09078235406645\\
75	0.117546896663338\\
76	0.123826058326627\\
77	0.15488956141445\\
78	0.230113171675962\\
79	0.308161842293122\\
80	0.229735902517688\\
};
\addlegendentry{\footnotesize VBP-SLAM1~\cite{Erik_AOABPSLAM_ICC2019}}

\addplot [color=red, dashed, line width=1.0pt, mark options={dash, red}]
  table[row sep=crcr]{%
0	0.981469839530703\\
1	0.966045971889366\\
2	0.425171294994916\\
3	0.482534075600437\\
4	0.311300067849925\\
5	0.522593501501085\\
6	0.318516180979774\\
7	0.267558600717772\\
8	0.0748953006569883\\
9	0.0715502032889025\\
10	0.164169890566883\\
11	0.242647282295889\\
12	0.173050580796239\\
13	0.149300986987941\\
14	0.134951550456464\\
15	0.123069418434771\\
16	0.143158242133675\\
17	0.166050541840461\\
18	0.175741513593063\\
19	0.219069831079362\\
20	0.191979436131335\\
21	0.0834571894081313\\
22	0.0497999291753846\\
23	0.155464787615493\\
24	0.129149737708998\\
25	0.210959696618775\\
26	0.179485601106511\\
27	0.118329192806788\\
28	0.268344470354512\\
29	0.354506441639885\\
30	0.318267135931403\\
31	0.213029558540469\\
32	0.15964271793405\\
33	0.13820500199997\\
34	0.200402247483238\\
35	0.143797964401856\\
36	0.193305858337961\\
37	0.282754399744565\\
38	0.191307185840951\\
39	0.111560157267883\\
40	0.153250303274179\\
41	0.14397072043095\\
42	0.096815848996119\\
43	0.230175147570465\\
44	0.223438926448914\\
45	0.226467009532699\\
46	0.0527508436235256\\
47	0.20206313818777\\
48	0.162825432719704\\
49	0.0789412415560053\\
50	0.169751287527788\\
51	0.294706430365799\\
52	0.18876083749006\\
53	0.0910880004828219\\
54	0.0819750588379939\\
55	0.0672339098834129\\
56	0.0853938675581538\\
57	0.135899515355453\\
58	0.0884458437782099\\
59	0.136084051275522\\
60	0.155917041226006\\
61	0.0652056733610316\\
62	0.195577456296822\\
63	0.255213146299211\\
64	0.248725839201369\\
65	0.262224709311571\\
66	0.203995041407862\\
67	0.0691609439406263\\
68	0.0823164963004939\\
69	0.160816462433826\\
70	0.0575166130632097\\
71	0.0468561984621275\\
72	0.122295445995125\\
73	0.107589416996902\\
74	0.0680830007284918\\
75	0.0679427593601219\\
76	0.0793020270465203\\
77	0.0919151529954032\\
78	0.102994317388714\\
79	0.108607554398509\\
80	0.0463790333968483\\
};
\addlegendentry{\footnotesize SBP-PMB-SLAM}
\end{axis}

\begin{axis}[%
width=24mm,
height=15mm,
at={(48mm,23mm)},
scale only axis,
xmin=76,
xmax=80,
ymin=0.04,
ymax=0.31,
axis background/.style={fill=white},
xmajorgrids,
ymajorgrids
]
\addplot [color=blue, line width=1.2pt, mark options={solid, blue}, forget plot]
  table[row sep=crcr]{%
71	0.155390944953208\\
72	0.214933387872515\\
73	0.200299756607877\\
74	0.09078235406645\\
75	0.117546896663338\\
76	0.123826058326627\\
77	0.15488956141445\\
78	0.230113171675962\\
79	0.308161842293122\\
80	0.229735902517688\\
};
\addplot [color=red, dashed, line width=1.2pt, mark options={dash, red}, forget plot]
  table[row sep=crcr]{%
71	0.0468561984621275\\
72	0.122295445995125\\
73	0.107589416996902\\
74	0.0680830007284918\\
75	0.0679427593601219\\
76	0.0793020270465203\\
77	0.0919151529954032\\
78	0.102994317388714\\
79	0.108607554398509\\
80	0.0463790333968483\\
};
\end{axis}

\begin{axis}[%
width=74mm,
height=40mm,
at={(0mm,0mm)},
scale only axis,
xmin=0,
xmax=1,
ymin=0,
ymax=1,
axis line style={draw=none},
ticks=none,
axis x line*=bottom,
axis y line*=left
]

\draw [black, line width=1.0pt] (axis cs:0.98,0.2) ellipse [x radius=5, y radius=16];
\draw[-{Stealth}, color=black, line width=1.0pt] (axis cs:0.94,0.3) -- (axis cs:0.85,0.45);
\end{axis}

\node[rotate=0,fill=white] (BOC6) at (3.7cm,-.7cm){\small time $k$};
\node[rotate=90] at (-8mm,20mm){\small RMSEs of localization [m]};
\end{tikzpicture}%
}
    \subcaption{}
    \label{Fig:InfLoc}
\end{subfigure}
\hspace{6mm}
\begin{subfigure}{0.45\textwidth}
    {
%
%
\definecolor{mycolor1}{rgb}{0.46667,0.67451,0.18824}%

\begin{tikzpicture}

\begin{axis}[%
width=74mm,
height=40mm,
at={(0mm,0mm)},
scale only axis,
xmin=0,
xmax=80,
ymin=240,
ymax=1800,
xticklabel style = {font=\small,yshift=0ex},
axis background/.style={fill=white},
axis background/.style={fill=white},
xmajorgrids,
ymajorgrids,
legend style={legend pos=south west, legend cell align=left, align=left, draw=white!15!black,style={row sep=-0.05cm}}
]

\addplot [color=blue, line width=1.0pt, mark options={solid, blue}]
  table[row sep=crcr]{%
0	1760\\
1	1760\\
2	1760\\
3	1760\\
4	1760\\
5	1702.86\\
6	1682.2866511152\\
7	1672.28079897387\\
8	1653.00374123126\\
9	1633.15858930612\\
10	1613.61367063487\\
11	1593.49848544072\\
12	1573.88355138378\\
13	1553.86669998388\\
14	1534.07595761948\\
15	1514.45757250804\\
16	1494.52131975116\\
17	1474.71746376537\\
18	1454.73814294823\\
19	1444.5896479068\\
20	1425.15612743605\\
21	1405.41481811181\\
22	1385.49285703521\\
23	1365.52322663459\\
24	1345.84260818152\\
25	1326.16868729404\\
26	1306.12555039382\\
27	1286.43440457881\\
28	1266.79765408124\\
29	1247.18432338323\\
30	1226.98350566752\\
31	1207.38904052048\\
32	1187.5326002446\\
33	1167.90791674596\\
34	1147.78294491223\\
35	1128.29697058073\\
36	1108.11356237125\\
37	1088.64805292983\\
38	1068.71039114156\\
39	1048.6113279889\\
40	1028.78012055309\\
41	1009.39222964741\\
42	989.287227886738\\
43	969.528723206857\\
44	949.822206081434\\
45	929.95292818398\\
46	910.056228418402\\
47	890.551531878688\\
48	870.62721447106\\
49	850.885150268838\\
50	830.859803126872\\
51	811.162826357165\\
52	791.237175667466\\
53	771.535263073964\\
54	751.47927018231\\
55	731.886828853743\\
56	711.835327932156\\
57	692.387037166804\\
58	672.208860126631\\
59	652.358722739029\\
60	633.058162229945\\
61	612.79321161197\\
62	592.877470477061\\
63	573.069121678142\\
64	553.168601370967\\
65	533.573510075093\\
66	514.03256394331\\
67	494.027142500152\\
68	474.319295254383\\
69	463.283524759805\\
70	444.538890388037\\
71	425.013409367948\\
72	404.90597955547\\
73	385.432473530191\\
74	365.632591175268\\
75	345.496194011727\\
76	325.918538002852\\
77	306.216646104442\\
78	286.174804404933\\
79	266.564126869158\\
80	246.421591647771\\
};
\addlegendentry{\footnotesize VBP-SLAM1~\cite{Erik_AOABPSLAM_ICC2019}}

\addplot [color=red, dashed, line width=1.0pt, mark options={dash, red}]
  table[row sep=crcr]{%
0	1760\\
1	1760\\
2	1760\\
3	1760\\
4	1760\\
5	1702.86\\
6	1682.3374673982\\
7	1672.18146445788\\
8	1652.81957251883\\
9	1632.92264289021\\
10	1613.29372324903\\
11	1593.13959155919\\
12	1573.44029003214\\
13	1553.43889742309\\
14	1533.64084014696\\
15	1513.96795264291\\
16	1493.96580832811\\
17	1474.13889790614\\
18	1454.13213700899\\
19	1443.93805267106\\
20	1424.4935512907\\
21	1404.71057586365\\
22	1384.76680417614\\
23	1364.75370496321\\
24	1345.04220993736\\
25	1325.3041302017\\
26	1305.24958039518\\
27	1285.52022353062\\
28	1265.83390413458\\
29	1246.18192317819\\
30	1226.01380310212\\
31	1206.40943927877\\
32	1186.56708893649\\
33	1166.9020057916\\
34	1146.73927619632\\
35	1127.24418964199\\
36	1107.03637643635\\
37	1087.51501506998\\
38	1067.50979942111\\
39	1047.39310333418\\
40	1027.47718969312\\
41	1008.06727219632\\
42	987.929180341962\\
43	968.092068986482\\
44	948.355230419941\\
45	928.436358130708\\
46	908.532591624317\\
47	888.939641383999\\
48	868.984411534112\\
49	849.189103932047\\
50	829.145299833944\\
51	809.401527926538\\
52	789.435030339699\\
53	769.668577961043\\
54	749.589863153746\\
55	729.967067131492\\
56	709.922734244983\\
57	690.452336876452\\
58	670.233257883064\\
59	650.365096627082\\
60	630.989160749924\\
61	610.704540386297\\
62	590.762089961701\\
63	570.926358069898\\
64	550.978082226819\\
65	531.288933264628\\
66	511.736126128092\\
67	491.703759012554\\
68	471.981059155577\\
69	460.908724875634\\
70	442.132126639741\\
71	422.5329526963\\
72	402.400470620362\\
73	382.891717558434\\
74	363.071436788275\\
75	342.8671722266\\
76	323.295473249516\\
77	303.512493829995\\
78	283.456248534176\\
79	263.824724617317\\
80	243.635357575542\\
};
\addlegendentry{\footnotesize SBP-PMB-SLAM}
\end{axis}

\begin{axis}[%
width=27mm,
height=15mm,
at={(45mm,23mm)},
scale only axis,
xmin=77.5,
xmax=78.5,
ymin=283,
ymax=287,
xtick={78},
xticklabels={$78$},
axis background/.style={fill=white},
xmajorgrids,
ymajorgrids
]

\addplot [color=blue, line width=1.2pt, mark options={solid, blue}, forget plot]
  table[row sep=crcr]{%
71	425.013409367948\\
72	404.90597955547\\
73	385.432473530191\\
74	365.632591175268\\
75	345.496194011727\\
76	325.918538002852\\
77	306.216646104442\\
78	286.174804404933\\
79	266.564126869158\\
80	246.421591647771\\
};

\addplot [color=red, dashed, line width=1.2pt, mark options={dash, red}, forget plot]
  table[row sep=crcr]{%
71	422.5329526963\\
72	402.400470620362\\
73	382.891717558434\\
74	363.071436788275\\
75	342.8671722266\\
76	323.295473249516\\
77	303.512493829995\\
78	283.456248534176\\
79	263.824724617317\\
80	243.635357575542\\
};
\end{axis}

\begin{axis}[%
width=74mm,
height=40mm,
at={(0mm,0mm)},
scale only axis,
xmin=0,
xmax=1,
ymin=0,
ymax=1,
axis line style={draw=none},
ticks=none,
axis x line*=bottom,
axis y line*=left
]

\draw [black, line width=1.0pt] (axis cs:0.96,0.05) ellipse [x radius=4, y radius=6];
\draw[-{Stealth}, color=black, line width=1.0pt] (axis cs:0.96,0.11) -- (axis cs:0.8,0.45);
\end{axis}

\node[rotate=0,fill=white] (BOC6) at (3.7cm,-.7cm){\small time $k$};
\node[rotate=90] at (-12mm,20mm){\small GOSPA errors of mapping [m]};
\end{tikzpicture}%
}
    \subcaption{}
    \label{Fig:InfMap}
\end{subfigure}
\caption{Informative birth case.
Comparison of the proposed 
SBP-PMB-SLAM~(i.e., VBP-SLAM2)
and 
VBP-SLAM1~\cite{Erik_AOABPSLAM_ICC2019}:
(a) RMSEs of sensor localization and (b) GOSPA errors of landmark mapping. The performance of 
SBP-MB-SLAM
is identical to 
SBP-PMB-SLAM,
and the plots are omitted.
}
\label{Fig:Inf}
\par
\end{center}
\end{figure*}
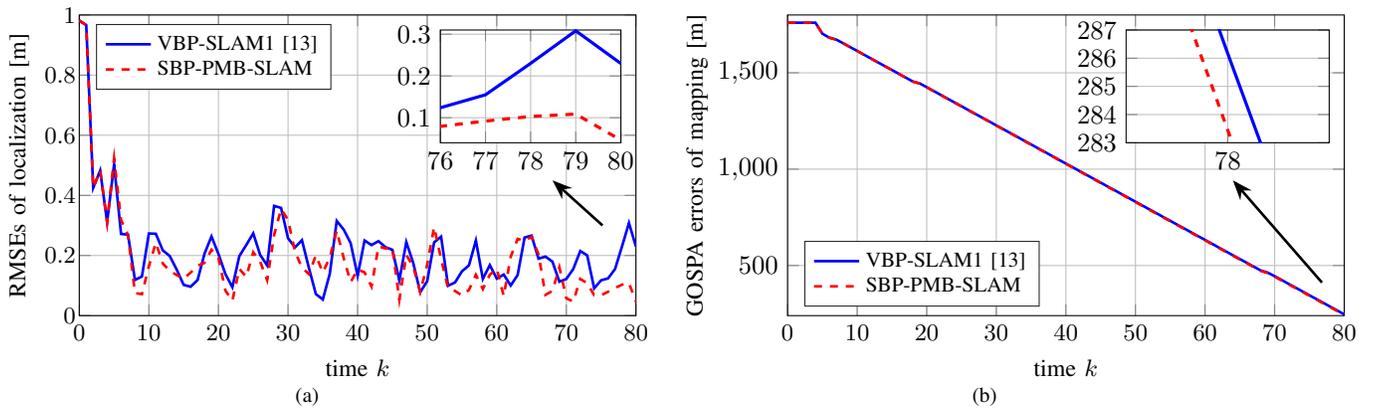

\subsection{Simulation Setup}
\subsubsection{Sensor and Landmarks}
    We consider a propagation environment of bistatic radio \ac{slam}, where a single \ac{bs} transmits the pilot signals, and \acp{sp}~(i.e., landmarks) are uniformly distributed, as shown in Fig.~\ref{Fig:Deployment}.
    A single sensor can receive two types of measurements: one from  \ac{bs}-sensor path; and others from \ac{bs}-\acp{sp}-sensor paths.
    We denote the locations of the \ac{bs} and \ac{sp} $j$ by $\mathbf{x}_\mathrm{BS}$ and $\mathbf{x}_\text{SP}^j$. The sensor state at time $k$ is denoted by $\mathbf{s}_k=[\mathbf{x}_{k,s}^\top,\dot{\mathbf{x}}_{k,s}^\top]^\top$, where 
    $\mathbf{x}_{k,s}= [x_{k,s},y_{k,s}]^\top$ denotes the location, and
    $\dot{\mathbf{x}}_{k,s}=[\dot{x}_{k,s},\dot{y}_{k,s}]^\top$ denotes the velocity.

\subsubsection{Dynamics}
    For the sensor dynamics during $K$ time steps, we adopt constant-velocity motion model~\cite[Sec.~6.3.2]{Yaakov_Tracking_Book2002} with the known transition density $f(\mathbf{s}_k|\mathbf{s}_{k-1})$, and the sensor state evolution follows 
\begin{align}
    \mathbf{s}_k = \mathbf{F}\mathbf{s}_{k-1} + \mathbf{B}\mathbf{q}_{k},
\end{align}
    $\mathbf{F}$ and $\mathbf{B}$ are defined as~\cite[Sec.~6.3.2]{Yaakov_Tracking_Book2002}
\begin{align}
    \mathbf{F} =
    \begin{bmatrix}
        \mathbf{I}_{2\times2} & \Delta t \mathbf{I}_{2\times2} \\
        \mathbf{0}_{2\times2} & \mathbf{I}_{2\times2}
    \end{bmatrix}, &&
    \mathbf{B} =
    \begin{bmatrix}
        0.5 \Delta t^2  \mathbf{I}_{2\times2}  \\
        \Delta t  \mathbf{I}_{2\times2}
    \end{bmatrix}.
\end{align}
    Here $\Delta t$ is the sampling time, and $\mathbf{q}_{k}$
    is the driving process that follows zero-mean Gaussian distribution with the covariance matrix $\sigma_w^2\mathbf{I}_{2\times2}$, where $\sigma_w$ is the standard deviation, and $\mathbf{I}_{2\times2}$ is the 2 by 2 identity matrix. 
    
    We set the parameters as follows: $K=80$ time steps, $\sigma_w=0.1~\text{m/s}^2$, and $\Delta t = 0.5$~s.
    We set the prior density of the sensor state to $f(\mathbf{s}_0) = \mathsf{N}(\mathbf{s}_0;\mathbf{s}_{0|0},\mathbf{P}_{0|0})$, where $\mathbf{s}_{0|0}$ is sampled from $\mathsf{N}(\bar{\mathbf{s}}_0,\mathbf{P}_{0|0})$ for each simulation run. Here $\bar{\mathbf{s}}_0 = [15,-420,0,20]^\top$ is the ground truth of the initial sensor state, with the units of m, m, m/s, and m/s, and $\mathbf{P}_{0|0} = \mathrm{blkdiag}[0.5 \mathbf{I}_{2 \times 2}, 0.005 \mathbf{I}_{2 \times 2}]$, and the units of the diagonal term of $\mathbf{P}$ are
    m$^2$, m$^2$, m$^2$/s$^2$, and m$^2$/s$^2$.


\subsubsection{Measurements}
    The sensor obtains the measurement set $\mathcal{Z}_k=\{\mathbf{z}_k^0,\mathbf{z}_k^1,\dots,\mathbf{z}_k^{J_k}\}$, where $\mathbf{z}_k^0$ and $\mathbf{z}_k^j$ for $j>0$ are the measurements corresponding to the \ac{bs} and different \acp{sp}, respectively.
    Let $\mathbf{t}_k^j = [({\mathbf{x}_\text{SP}^j})^\top,\mathbf{x}_{k,s}^\top]^\top$ be the augmented vector of \ac{sp} location $\mathbf{x}_\text{SP}^j$ and the sensor location $\mathbf{x}_{k,s}$. With the known \ac{bs} location, the measurements $\mathbf{z}_k^0$ and $\mathbf{z}_k^j$ for $j>0$
    are modeled as
\begin{align}
    \mathbf{z}_k^0 & = \mathbf{x}_{k,s} + \mathbf{r}_k, \\
    \mathbf{z}_k^j & = \mathbf{H} \mathbf{t}_k^j + \mathbf{r}_k.
\end{align}
    Here $\mathbf{H}= [\mathbf{I}_{2\times2}, -\mathbf{I}_{2\times2}]$, and $\mathbf{r}_k$ is the measurement noise that follows zero-mean Gaussian distribution with the covariance $\sigma_r^2 \mathbf{I}_{2\times2}$, where $\sigma_r$ is the standard deviation.
    We regard the false alarms and shortly visible \acp{sp} as clutter, modeled as $\mathbf{z}_k^\mathrm{C} \in \mathcal{Z}_k$.
    The number of clutter measurements, which are observed by the sensor, is modeled by the Poisson distribution with mean $\mu_\mathrm{c}$~\cite{Angel_Trajectory_TSP2020}, and 
    the clutter intensity is denoted by $c(\mathbf{z})$.

    
    The parameters are set as follows: $\mu_\mathrm{c}=1$, $c(\mathbf{z})=1.6 \cdot 10^{-4}$ in the area of interest, and $\sigma_r=0.707$~m.
    The number of SPs that are uniformly distributed in a network is denoted by $N_\mathrm{SP}$.
    The network size is $[0, 30]$~m $\times  [-400, 470]$~m, $N_\mathrm{SP}=176$, and the \ac{bs} location is $\mathbf{x}_\mathrm{BS} = [0,0]^\top$. 
    The  field-of-view of the sensor is determined by  
\begin{align}
\ensuremath{p_{\mathrm{D}}(\mathbf{s}_{k},\mathbf{x}_{k})} & =\begin{cases}
p_{\mathrm{D}} & \left\Vert \mathbf{x}_{k}-\mathbf{s}_{k}\right\Vert <r\\
0 & \mathrm{otherwise}
\end{cases}
\end{align}
where $r=20$ m and $p_\mathrm{D}=0.95$. 
    The landmarks are observable starting from $k=5$, and additional landmarks can be observable at each subsequent time step.
    For sensor localization in both scenarios, multipath is employed from $k=5$, whereas for $k<5$ the measurement for \ac{bs}-sensor path is exploited.\footnote{By the ellipsoidal gating method \cite{Panta_GatingPHD_TAES207}, the measurement for \ac{bs}-sensor path is determined as $\hat{\mathbf{z}}=\min_{\mathbf{z}_k^j \in \mathcal{Z}_k} (\mathbf{z}_k^j - \mathbf{x}_\mathrm{BS})^\top\mathbf{R}^{-1}(\mathbf{z}_k^j - \mathbf{x}_\mathrm{BS})$.}

\subsubsection{SLAM Filter}

    To investigate the influence of newly detected landmarks on sensor state updates, we consider $K$ time steps with two scenarios: 
one with an uninformative birth model, and the other with an informative birth model in modeling undetected landmarks. 
    In realistic environments, the uninformative birth model represents cases where we lack prior map knowledge, whereas under the informative birth model the assumption is that we have a previously available map.
    

    
    Undetected landmarks and the birth for both 
    SBP-PMB-SLAM and VBP-PMB1~\cite{Erik_AOABPSLAM_ICC2019}
    are modeled by the \ac{ppp}, implemented by the intensity functions.
    The intensities for undetected targets and the birth are 
    represented by 
    $\lambda(\mathbf{x}) = \sum_q \eta^q \mathsf{N}(\mathbf{x};\mathbf{x}^q,\mathbf{U}^q)$ and $\lambda^\mathrm{B}(\mathbf{x}) = \sum_q \eta^{\mathrm{B},q} \mathsf{N}(\mathbf{x};\mathbf{x}^{\mathrm{B},q},\mathbf{U}^{\mathrm{B},q})$,
    where $\mathsf{N}(\mathbf{x};\mathbf{x},\mathbf{U})$ is the Gaussian density, and $\eta$ is the Gaussian weight.
    For SBP-MB-SLAM, the birth is modeled by the \ac{mb} process. The \ac{mb} birth density is the form of~\eqref{eq:MBdensity}, where the $q$-th Bernoulli has an existence probability $r^{\text{B},q}$ and Gaussian density $\mathsf{N}(\mathbf{x};\mathbf{x}^{\text{B},q},\mathbf{U}^{\text{B},q})$.
    For the informative birth case, we set the number of births to the number of newly detected landmarks, $\mathbf{x}^{\text{B},q}\sim \mathsf{N}(\mathbf{x};\mathbf{x}_\text{SP}^{j(q)},\mathbf{U}^{\text{B},q})$, $\eta^{\text{B},q} = 1$, and the small Gaussian covariance $\mathbf{U}^{\text{B},q}=0.01 \mathbf{I}_{2 \times 2}$, where $j(q)$ is the \ac{sp} index corresponding to the $q$-th Gaussian.
    For the uninformative birth case, we set the number of Gaussians for births to the same as the number of measurements, birth weight $\eta^{\text{B},q} = 10^{-3}$, and the large Gaussian covariance $\mathbf{U}^{\text{B},q}=10^6\times \mathbf{I}_{2 \times 2}$.
    
    Using the Kalman filter~\cite{Simon2013}, we implement the messages and beliefs corresponding to landmarks and data associations. 
    The belief of the sensor state of \fbox{\footnotesize{27}} is intractable by the Kalman filter, and thus implemented by the particle filter~\cite{ParticleFilter_TSP2002,Henk_IRD_2007} with $N_\mathrm{P}$ samples.
    After the belief computation at each time step, we declare that a landmark is detected when $r_{k|k}^i> \Gamma_\mathrm{D}$. The previously detected landmarks with $r_{k|k}^i< \Gamma_\mathrm{Ber}$ for $i=1,\dots,I_{k-1}+J_{k}$ are removed, and the Gaussians in the \ac{ppp} with $\eta_{k|k}^q < \Gamma_\mathrm{Poi}$ are eliminated.
 
    
    The rest of the algorithm parameters are set as follows:
    $N_p=10^4$,
    $\Gamma_\mathrm{D}=0.4$;
    $\Gamma_\mathrm{Ber}=10^{-5}$;
    $\Gamma_\mathrm{Poi}=5\cdot 10^{-10}$; $p_\mathrm{S}=0.99$.
    The simulation results are averaged over 500 Monte Carlo trials.
    The performance of sensor localization and landmark mapping are evaluated by \ac{rmse} and \ac{gospa}~\cite[eq.~(1)]{RahmathullahGFS:2017} with parameters $\mathrm{G}_p=1$, $\mathrm{G}_c=2$, and $\mathrm{G}_\alpha=2$, respectively.



\subsection{Results and Discussions}
    
\subsubsection{Uninformative Birth}
    Fig.~\ref{Fig:UnInf} shows the \ac{slam} performance against time, under the scenario that the \ac{ppp} for undetected landmarks and birth in SBP-PMB-SLAM 
    and the \ac{mb} for birth in SBP-MB-SLAM
    are uninformative.
    During $K=80$, the \ac{slam} performance of the 
    SBP-PMB-SLAM
    is identical to that of the 
    VBP-BP-SLAM
    filter even though the messages \fbox{\footnotesize{23}} are employed for sensor belief computation.
    This happens because the messages \fbox{\footnotesize{23}} corresponding to newly detected landmarks or clutter are not as informative as the messages \fbox{\footnotesize{20}} corresponding to previously detected landmarks.
    The mapping performance of the 
    SBP-MB-SLAM
    filter is inferior to the other two filters. This is because the \ac{mb} birth model is limited in the number of Bernoulli sets, where each is available for modeling a single landmark, whereas the set that follows \ac{ppp} captures multiple landmarks.
    

\subsubsection{Informative Birth}
    Fig.~\ref{Fig:Inf} shows the \ac{slam} performance against time, under the scenario that the \ac{ppp} for undetected landmarks and birth in 
    SBP-PMB-SLAM 
    and MB birth in 
    SBP-MB-SLAM
    are informative.
    When using the informative PPP, we achieve performance improvement in 
    set-type BP PMB-SLAM 
    SBP-PMB-SLAM
    compared to that of 
    VBP-SLAM1.
    This is because new landmarks appear and the \ac{ppp} for undetected landmarks is informative. 
    The sensor localization gap is clearly visible, and the landmarks are determined, starting from $k=5$ since
    the landmarks begin to be observable from $k=5$.
    The gap of \acp{gospa} increases over time steps because there exist newly detected landmarks
    at every time step.
    We obtain that the performance of 
    SBP-MB-SLAM
    is identical to that of 
    SBP-PMB-SLAM,
    and thus the results are omitted.
    This is because the clutter part of the messages \fbox{\footnotesize{23}}\footnote{This message comprises the sum of two parts: clutter and newly detected landmark.} is much smaller than the newly detected landmark part, and the implementation of \ac{mb} birth for 
    SBP-MB-SLAM
    is equivalent to that of 
    SBP-PMB-SLAM,
    under the informative birth case.
    
    Table~\ref{tab:SLAMPerform}
    present the \acp{rmse} of sensor localization and \ac{gospa} errors of landmark mapping respectively, with the different setup: $\mu_c=1$ with $c(\mathbf{z})=1.6\cdot10^{-4}$, $\mu_c=5$ with $c(\mathbf{z})=8\cdot10^{-4}$, and $\mu_c=10$ with $c(\mathbf{z})=1.6\cdot10^{-3}$, for
    $N_\mathrm{SP}= 176$, $N_\mathrm{SP}= 352$.
    The results are averaged over all Monte Carlo simulation runs during the steady-state operation after $k>40$. 
    Table~\ref{tab:SLAMPerform}
    show the performance of both methods deteriorates progressively as the clutter Poisson mean $\mu_c$ increases~($\mu_c=1,\mu_c=5,\mu_c=10$). 
    The sensor localization accuracy with $N_\mathrm{D}=4$ is improved, compared to with $N_\mathrm{D}=2$. 
    This improvement is attributed to the fact that sensor localization accuracy increases as either mapping accuracy or the number of detected targets increases. 
    The performance gap between 
    SBP-PMB-SLAM and VBP-SLAM1
    arises from the usage of newly detected targets in sensor state updates.

\begin{table}[t]
    \centering
    \caption{Informative birth case: RMSEs of sensor localization and GOSPA errors of mapping with different $N_\mathrm{SP}$ and $\mu_c$.}
    \resizebox{1\columnwidth}{!} 
    {
    \begin{tabular}{|P{1cm} P{1cm} P{1cm}|P{2cm}|P{2cm}|}
    \hline
     \textbf{}
     & & & VBP-SLAM1~\cite{Erik_AOABPSLAM_ICC2019} & SBP-PMB-SLAM\\
    \hline
    \hline
    \multicolumn{1}{|l|}{\multirow{6}{*}{RMSEs [m]}} & \multicolumn{1}{l|}{\multirow{3}{*}{$N_\mathrm{SP}=176$}} & \multicolumn{1}{l|}{$\mu_c=1$} & \multicolumn{1}{c|}{0.1651} & \multicolumn{1}{c|}{0.1306}\\
    \cline{3-5}
    \multicolumn{1}{|l|}{}& \multicolumn{1}{l|}{} & \multicolumn{1}{l|}{$\mu_c=5$} & \multicolumn{1}{c|}{0.1881} & \multicolumn{1}{c|}{0.1334}\\
    \cline{3-5}
    \multicolumn{1}{|l|}{}& \multicolumn{1}{l|}{} & \multicolumn{1}{l|}{$\mu_c=10$} & \multicolumn{1}{c|}{0.2014} & \multicolumn{1}{c|}{0.1401}\\
    \cline{2-5}
    \multicolumn{1}{|l|}{}& \multicolumn{1}{l|}{\multirow{3}{*}{$N_\mathrm{SP}=352$}} & \multicolumn{1}{l|}{$\mu_c=1$} & \multicolumn{1}{c|}{0.1321} & \multicolumn{1}{c|}{0.1125}\\
    \cline{3-5}
    \multicolumn{1}{|l|}{}& \multicolumn{1}{l|}{} & \multicolumn{1}{l|}{$\mu_c=5$} & \multicolumn{1}{c|}{0.1361} & \multicolumn{1}{c|}{0.1137}\\
    \cline{3-5}
    \multicolumn{1}{|l|}{}& \multicolumn{1}{l|}{} & \multicolumn{1}{l|}{$\mu_c=10$} & \multicolumn{1}{c|}{0.1559} & \multicolumn{1}{c|}{0.1195}\\
    \cline{1-5}
    \multicolumn{1}{|l|}{\multirow{6}{*}{GOSPA Errors [m]}} & \multicolumn{1}{l|}{\multirow{3}{*}{$N_\mathrm{SP}=176$}} & \multicolumn{1}{l|}{$\mu_c=1$} & \multicolumn{1}{c|}{638.48} & \multicolumn{1}{c|}{636.42}\\
    \cline{3-5}
    \multicolumn{1}{|l|}{}& \multicolumn{1}{l|}{} & \multicolumn{1}{l|}{$\mu_c=5$} & \multicolumn{1}{c|}{979.74} & \multicolumn{1}{c|}{977.31}\\
    \cline{3-5}
    \multicolumn{1}{|l|}{}& \multicolumn{1}{l|}{} & \multicolumn{1}{l|}{$\mu_c=10$} & \multicolumn{1}{c|}{1162.34} & \multicolumn{1}{c|}{1160.52}\\
    \cline{2-5}
    \multicolumn{1}{|l|}{}& \multicolumn{1}{l|}{\multirow{3}{*}{$N_\mathrm{SP}=352$}} & \multicolumn{1}{l|}{$\mu_c=1$} & \multicolumn{1}{c|}{622.52} & \multicolumn{1}{c|}{620.99}\\
    \cline{3-5}
    \multicolumn{1}{|l|}{}& \multicolumn{1}{l|}{} & \multicolumn{1}{l|}{$\mu_c=5$} & \multicolumn{1}{c|}{794.65} & \multicolumn{1}{c|}{793.17}\\
    \cline{3-5}
    \multicolumn{1}{|l|}{}& \multicolumn{1}{l|}{} & \multicolumn{1}{l|}{$\mu_c=10$} & \multicolumn{1}{c|}{881.29} & \multicolumn{1}{c|}{880.04}\\
    \cline{2-5}
    \hline
    \end{tabular}
                    }
    \label{tab:SLAMPerform}
\end{table}

\section{Conclusions} \label{sec:Conclusions}
    
    In this paper, we have developed the general framework of set-type \ac{bp}, which can serve as a fundamental tool for computing either
    the marginal (or its approximate density) of an \ac{rfs}.
    From the framework, we derived the \ac{pmb} and \ac{mb} filters with the developed set-type \ac{bp}, applicable to the related problems of mapping, \ac{mtt}, \ac{slam}, and \ac{slat}.
    Under the densities that follow the Bernoulli process, we demonstrated that vector-type \ac{bp} is the special case of set-type \ac{bp}.
    To handle the unknown set cardinality in the factor graph, we developed the set-factor nodes for set partitioning, set merging, and shifting set auxiliary variables.
    We applied the proposed set-type \ac{bp} to \ac{pmb}-\ac{slam} filter and explored the relations between the set-type \ac{bp} \ac{pmb} and vector-type \ac{bp}-\ac{slam}~\cite{Erik_BPSLAM_TWC2019,Erik_AOABPSLAM_ICC2019} filters.
    Our results demonstrated that the proposed set-type \ac{bp}-\ac{slam} filter (and the equivalent vector type BP algorithm) outperforms 
    vector-type \ac{bp} from \cite{Erik_BPSLAM_TWC2019,Erik_AOABPSLAM_ICC2019}, under the informative \ac{ppp} for undetected landmarks, and equivalent for the uninformative \ac{ppp}. Vector-type BP can be readily modified, through various heuristics, to match the performance and operation of set-type BP.

    Other applications where set-type BP can be applied include smoothing (obtained by computing messages backward in time), extended targets (where each target may generate a set of measurements), cooperative processing (obtained by linking the factor graphs related to two different sensors), general target-generated measurements and arbitrary clutter, and positioning (by adding factors related to fixed anchors in the environment). 
    Applications
    of the set-type \ac{bp} framework to other families of \ac{rfs} filters,
    along with their \ac{bp} counterparts,
    deserve further study. 
    Finally, we hope that the proposed framework can find application in problems outside the sphere of mapping and tracking. 
    


\begin{appendices}

\section{Proof of Theorem~\ref{theo:DerSetBP}}
\label{app:Optimal_SetBP}
    We prove Theorem~\ref{theo:DerSetBP} in a similar fashion as in the vector-type \ac{bp}~\cite[Theorem~2]{Yedidia2005BetheFree}. 

\begin{proof}
    Using the set beliefs of \eqref{eq:setVbp} and \eqref{eq:setFbp}, we introduce the Bethe free energy~\cite{Yedidia2005BetheFree}: $F_\mathrm{Bethe} = U_\mathrm{Bethe}- H_\mathrm{Bethe}$, where $U_\mathrm{Bethe}$ is the Bethe average energy, given by\footnote{The set integral such as $\int \mathtt{b}(\mathcal{X}) \ln \mathtt{b}(\mathcal{X}) \delta \mathcal{X}$ and $\int \mathtt{b}(\mathcal{X}) \ln f(\mathcal{X}) \delta \mathcal{X}$ requires the use of the measure-theoretic integral~\cite{Hoang_PoissonCSD_TIT2015}, due to the units of the standard set integral and densities~\cite[Sec.~3.2.4]{Mahler_book2014}.
    In this sense, the integral is then 
    $\int \mathtt{b}(\mathcal{X}) \ln (K^{\lvert \mathcal{X} \rvert}\mathtt{b}(\mathcal{X})) \delta \mathcal{X}$ and $\int \mathtt{b}(\mathcal{X}) \ln (K^{\lvert \mathcal{X} \rvert}
    f(\mathcal{X})) \delta \mathcal{X}$, where $K$ is the unit of the hypervolume of the single state $\mathbf{x} \in \mathcal{X}$. For notational simplicity, we omit the unit $K^{\lvert \mathcal{X} \rvert}$ in the integration.}
\begin{align}
    U_\mathrm{Bethe} = - \sum_{a}\int  \mathtt{b}(\underline{\mathcal{X}}^a) \ln f_a(\underline{\mathcal{X}}^a) \delta \underline{\mathcal{X}}^a,
\end{align}
    and $H_\mathrm{Bethe}$ is the Bethe entropy, given by
\begin{align}
    H_\mathrm{Bethe} 
    =&
    \, -\sum_a \int 
     \mathtt{b}(\underline{\mathcal{X}}^a) \ln \mathtt{b}(\underline{\mathcal{X}}^a) \delta \underline{\mathcal{X}}^a \notag\\
     &+
     \sum_i(\lvert \mathcal{M}(i) \rvert-1)\int \mathtt{b}(\mathcal{X}^i) \ln \mathtt{b}(\mathcal{X}^i) \delta \mathcal{X}^i.
\end{align}
    The constraints of the Bethe free energy, the average energy, and the entropy are all functions of set-beliefs. They are introduced as follows.
    The normalization constraints are
    $\int \mathtt{b}(\mathcal{X}^i) \delta \mathcal{X}^i = 1$ for all set-variable $i$, $\int \mathtt{b}(\underline{\mathcal{X}}^a) \delta \underline{\mathcal{X}}^a = 1$ for all set-factor $a$ with $\lvert \mathcal{M}(i) \rvert \geq 2$. 
    The marginalization constraints are
    $\int \mathtt{b}(\underline{\mathcal{X}}^a) \delta \mathcal{X}^{\sim i} = \mathtt{b}(\mathcal{X}^i)$ such that $i \in \mathcal{N}(a)$.
    The inequality constraints are $0 \leq \mathtt{b}(\underline{\mathcal{X}}^a) \leq 1$ for all set-factor $a$ and $0 \leq \mathtt{b}(\mathcal{X}^i) \leq 1$ for all set-variable $i$.
    Due to the assumption of the interior stationary point, the inequality constraints will be inactive.
    Thus, we enforce the equality constraints with the Lagrange multipliers,  denoted by $\gamma_a$, $\gamma_i$, and $\lambda_{a,i}(\mathcal{X}^i)$, respectively, and the Lagrangian is formulated as
\begin{align}
    L =
    &
    - \sum_{a}\int  \mathtt{b}(\underline{\mathcal{X}}^a) \ln f_a(\underline{\mathcal{X}}^a) \delta \underline{\mathcal{X}}^a \\
    &
    +\sum_a \int 
     \mathtt{b}(\underline{\mathcal{X}}^a) \ln \mathtt{b}(\underline{\mathcal{X}}^a) \delta \underline{\mathcal{X}}^a \notag \\
    &
    - \sum_i(\lvert \mathcal{M}(i) \rvert-1)\int \mathtt{b}(\mathcal{X}^i) \ln \mathtt{b}(\mathcal{X}^i) \delta \mathcal{X}^i \notag \\
    &
    + \sum_a \gamma_a ( \int  \mathtt{b}(\underline{\mathcal{X}}^a)\delta \underline{\mathcal{X}}^a -1)
    + \sum_i \gamma_i ( \int  \mathtt{b}(\mathcal{X}^i) \delta \mathcal{X}^i -1) \notag \\
    &
    + \sum_i \sum_{a\in\mathcal{M}(i)}\int \lambda_{a,i}(\mathcal{X}^i) (\mathtt{b}(\mathcal{X}^i) - \int \mathtt{b}(\underline{\mathcal{X}}^a) \delta \mathcal{X}^{\sim i}) \delta \mathcal{X}^{i}. \notag
\end{align}

    By the derivatives of the Lagrangian with respect to $\mathtt{b}(\underline{\mathcal{X}}^a)$ and $\mathtt{b}(\mathcal{X}^i)$, we can obtain the interior stationary points as follows:
\begin{align}
    \hat{\mathtt{b}}(\mathcal{X}^i) &=\frac{1}{Z_i}\exp \Big (\frac{1}{\lvert \mathcal{M}(i) \rvert-1}\sum_{a\in \mathcal{M}(i) }\lambda_{a,i}(\mathcal{X}^i) \Big ),\label{eq:der_b_i}\\
    \hat{\mathtt{b}}(\underline{\mathcal{X}}^a) &= \frac{1}{Z_a}f_a(\underline{\mathcal{X}}^a)\exp \Big(\sum_{i\in \mathcal{N}(a)}\lambda_{a,i}(\mathcal{X}^i) \Big),\label{eq:der_b_a}
\end{align}
    where $Z_a$ and $Z_i$ are the normalization constants.
    Making the identification
\begin{align}
    \lambda_{a,i}(\mathcal{X}^i) = \ln n_{i \rightarrow a}(\mathcal{X}^i) = \ln \prod_{b \in \mathcal{M}(i)\setminus \{a\}} m_{b \rightarrow i} (\mathcal{X}^i),
    \label{eq:lambdaM}
\end{align}
    and substitute \eqref{eq:lambdaM} into \eqref{eq:der_b_i} and \eqref{eq:der_b_a}, then we recover the set-type \ac{bp} fixed points of \eqref{eq:setVbp} and \eqref{eq:setFbp} as follows:
\begin{align}
    \hat{\mathtt{b}}(\mathcal{X}^i) &\propto \prod_{a \in \mathcal{M}(i)} \mathtt{m}_{a \to i}(\mathcal{X}^i),\\
    \hat{\mathtt{b}}(\underline{\mathcal{X}}^a)  &\propto 
    f_a (\underline{\mathcal{X}}^a)
    \prod_{i \in \mathcal{N}(a)} \prod_{b \in \mathcal{M}(i) \setminus \{a\}} \mathtt{m}_{b \to i}(\mathcal{X}^{j}).
\end{align}

    To derive this theorem conversely, we introduce $\mathtt{m}_{a\rightarrow i}$
\begin{align}
    \mathtt{m}_{a\rightarrow i} 
    &= \exp \bigg(
    \frac{2- \lvert \mathcal{M}(i) \rvert}{\lvert \mathcal{M}(i) \rvert-1} \lambda_{a,i}(\mathcal{X}^i) 
    + \frac{1}{\lvert \mathcal{M}(i) \rvert-1}\lambda_{b,i}(\mathcal{X}^i)\bigg),\label{eq:m_ai}
\end{align}
    where $\lambda_{b,i}(\mathcal{X}^i)= \ln n_{i \rightarrow b}(\mathcal{X}^i)$, which can be obtained from \eqref{eq:lambdaM} and set-message update rules of~\eqref{eq:setM_fv} and \eqref{eq:setM_vf}.
    Substituting $\prod_{b\in \mathcal{M}(i)\setminus \{a\}} \mathtt{m}_{b \rightarrow i}$ of \eqref{eq:lambdaM} and $\mathtt{m}_{a \rightarrow i}$ of \eqref{eq:m_ai} into the set-type \ac{bp} fixed points of \eqref{eq:setVbp} and \eqref{eq:setFbp}, the reverse of this theorem can be shown.
    
    We omitted a single variable that is only connected to a single factor (i.e., $\lvert \mathcal{M}(i) \rvert = 1$), called a dead-end variable, in the Lagrangian since the dead-end variable does not contribute to the Bethe free energy and the beliefs. The beliefs at dead-end variables are not required but it can be easily computed from the belief $\int \mathtt{b}(\underline{\mathcal{X}}^a) \delta \mathcal{X}^{\sim i}$ as needed.

\end{proof}

\section{Proof of Corollary~\ref{coro:setBelief_nocycle}}
\label{app:setBelief_nocycle}
    We prove 
    Corollary~\ref{coro:setBelief_nocycle}
    by showing that the set-belief $\mathtt{b}(\mathcal{X}^i)$ obtained after running set-type \ac{bp} in a factor graph with no cycles corresponds to the marginal probability density of $\mathcal{X}^i$.
\begin{proof}
    Note that $\mathtt{m}_{a \rightarrow i}(\mathcal{X}^i) = f_a(\underline{\mathcal{X}}^a)$ if $\lvert \mathcal{N}(a) \rvert = 1$.
    By the chain rule with \eqref{eq:setM_fv} and \eqref{eq:setM_vf},
    the set-belief $\mathtt{b}(\mathcal{X}^i)$ of~\eqref{eq:setM_fv} is represented by the integration of the product of all set-factors $f_a(\underline{\mathcal{X}}^a)$.
    It corresponds to $\int f({\mathcal{X}}^1,\dots,{\mathcal{X}}^n) \delta \mathcal{X}^{\sim i}$ of~\eqref{eq:SetFactor}.
\end{proof}

\section{Proof of Proposition~\ref{prop:Poisson}}
\label{app:ProofPoi}
    We prove Proposition~\ref{prop:Poisson} with the partitioning and merging factor of Definition~\ref{def:PartitionFactor}.
\begin{proof}
    Using the partitioning and merging factor of Definition~\ref{def:PartitionFactor}, we can then partition the Poisson set $\mathcal{X}^j$ with the incoming Poisson message $\mathtt{n}_{j \rightarrow a}(\mathcal{X}^j)$ into $\lvert \mathcal{N}(a) \rvert-1$ Poisson sets $\mathcal{X}^i$ with the Poisson outgoing messages $\mathtt{m}_{a \rightarrow i}(\mathcal{X}^i)$ for $i \in \mathcal{N}(a) \setminus \{j\}$. 
    Suppose we have incoming messages $\mathtt{n}_{j \rightarrow a}(\mathcal{X}^j)=f^{\mathrm{PPP}}(\mathcal{X}^{j})$ that follows the \ac{ppp} density, $\mathtt{n}_{q \rightarrow a}(\mathcal{X}^{q})=1$, for $q \in \mathcal{N}(a) \setminus \{j,i\}$. 
    From~\eqref{eq:SMFactor_ForOut}, the partitioning messages from $f_a(\underline{\mathcal{X}}^a)$ to $\mathcal{X}^i$ with $i \in \mathcal{N}(a) \setminus \{j\}$ are
\begin{align}
    \mathtt{m}_{a \rightarrow i}(\mathcal{X}^i) 
    &= \int \mathtt{n}_{j \rightarrow a}(\uplus_{q \in \mathcal{N}(a)\setminus \{j\}} \mathcal{X}^{q})
    \delta \mathcal{X}^{\sim i} \label{eq:PartitionPPP_cup}\\
    & = \int \mathtt{n}_{j \rightarrow a}(\mathcal{X}^i \uplus \mathcal{X})
    \delta \mathcal{X}\\
    & \propto f^{\mathrm{PPP}}(\mathcal{X}^{i}),
\end{align}
    the same \ac{ppp} intensity function.
    It indicates that the Poisson set is partitioned into multiple Poisson sets with the same \ac{ppp} intensity function, 
    which is derived in Appendix~\ref{app:PPPPartition}.
    
    Conversely, the $\lvert \mathcal{N}(a) \rvert-1$ Poisson sets $\mathcal{X}^i$ with the incoming Poisson messages $\mathtt{n}_{i \rightarrow a}(\mathcal{X}^i)$ for $i \in \mathcal{N}(a) \setminus \{j\}$ are merged into a single Poisson set $\mathcal{X}^j$ with the Poisson outgoing message $\mathtt{m}_{a \rightarrow j}(\mathcal{X}^j)$.  
    Suppose we have incoming messages $\mathtt{n}_{i \rightarrow a}(\mathcal{X}^{i}) = f^{\mathrm{PPP}}(\mathcal{X}^{i})$, for $i \in \mathcal{N}(a) \setminus \{j\}$ that follow \ac{ppp} densities.
    The  message of~\eqref{eq:SMFactor_BackOut} is
\begin{align}
    \mathtt{m}_{a \rightarrow j}(\mathcal{X}^j) 
    & = 
    \sum_{\uplus_{i \in \mathcal{N}(a) \setminus \{j\}} \mathcal{W}^{i}=\mathcal{X}^j}
    \prod_{i \in \mathcal{N}(a) \setminus \{j\}} 
    f^\mathrm{PPP}(\mathcal{W}^{i}). 
\end{align}
    It indicates that the outgoing message is the convolution of all incoming PPP messages, representing the union of multiple Poisson sets.
\end{proof}

\section{Proof of Poisson Set Partition}
\label{app:PPPPartition}
    We find that $f^\mathrm{PPP}(\mathcal{X}) = e^{-\bar{\lambda}} \prod_{\mathbf{x} \in \mathcal{X}} \lambda(\mathbf{x})$ and $\bar{\lambda} = \int \lambda(\mathbf{x}) \mathrm{d}\mathbf{x}$ from~\eqref{eq:PoissonDensity}, and then~\eqref{eq:PartitionPPP_cup} is expressed as
\begin{align}
    &\mathtt{m}_{a \rightarrow i}(\mathcal{X}^i) \\
    &= \int 
    e^{-\bar{\lambda}}
    \prod_{q\in \mathcal{N}(a) \setminus \{j\} }
    \prod_{\mathbf{x}\in\mathcal{X}^{q}} \lambda({\mathbf{x}})
    \delta \mathcal{X}^{\sim i} \\
    & \propto 
    e^{-\bar{\lambda}}
    \prod_{\mathbf{x}\in\mathcal{X}^{i}} \lambda({\mathbf{x}})
    \prod_{q\in \mathcal{N}(a) \setminus \{j\} } 
    \int 
    e^{-\bar{\lambda}}
    \prod_{\mathbf{x}\in\mathcal{X}^{q}} \lambda({\mathbf{x}})
    \delta \mathcal{X}^{\sim (i,j)}\\
    &=
    e^{-\bar{\lambda}}
    \prod_{\mathbf{x}\in\mathcal{X}^{i}} \lambda({\mathbf{x}})=f^\mathrm{PPP}(\mathcal{X}^i).
\end{align}
    Finally, we find that $ \mathtt{m}_{a \rightarrow i}(\mathcal{X}^i) \propto f^\mathrm{PPP}(\mathcal{X}^i)$.

\section{Joint Update Density}
\label{SM:JointUp}

We prove the joint update density of~\eqref{eq:factorizedDensity} with introducing $\mathcal{X}_k^1\uplus \cdots \uplus \mathcal{X}_k^{I_{k-1}} = \overline{\mathcal{X}}_k^{I_{k-1}}$,  $\mathcal{Z}_k^1\uplus \cdots \uplus \mathcal{Z}_k^{I_{k-1}} = \overline{\mathcal{Z}}_k^{I_{k-1}}$, and $\mathcal{Y}_k^1\uplus \cdots \uplus \mathcal{Y}_k^{J_k} = \overline{\mathcal{Y}}_k^{J_k}$. 
We start with the prior at time step $k$ (without auxiliary variables)
such that
\begin{align}
    f_\mathtt{p}(\mathbf{s}_{k},{\mathcal{X}}_{k}) & =f_\mathtt{p}(\mathbf{s}_{k} )f_\mathtt{p}({\mathcal{X}}_{k})\\
    & =f_\mathtt{p}(\mathbf{s}_{k} )
    \sum_{{\mathcal{P}}_k^\mathrm{U}\uplus \overline{\mathcal{X}}_k^{I_{k-1}} = \mathcal{X}_k}
    f_\mathtt{p}^\mathrm{U} ({\mathcal{P}}_k^\mathrm{U}) \prod_{i=1}^{I_{k-1}} f_{\mathtt{p}}^{i} ({\mathcal{X}}_k^i ),
\end{align}
where the prior assumes that the sensor state and the set of targets
are independent. The set of measurements received at time step $k$ is $\mathcal{Z}_{k}=\{ z_{k}^{1},...,z_{k}^{J_{k}}\}$.
\subsection{Likelihood}

For any sets $\mathcal{P}_k^\mathrm{U},\mathcal{X}_k^{1},...,\mathcal{X}_k^{I_{k-1}}$ such that $\lvert \mathcal{X}_k^{i} \rvert \leq 1$
for $i=1,...,I_{k-1}$, we introduce the likelihood functions~\cite[Eq.~(25)]{Angel_PMBM_TAES2018}
\begin{align}
    &l_{o}(\mathcal{Z}_k|{\mathcal{P}}_k^\mathrm{U},{\mathcal{X}}_k^{1},...,{\mathcal{X}}_k^{I_{k-1}},\mathbf{s}_k) \nonumber\\
    &
    = \sum_{\mathcal{Z}_k^\mathrm{U}\uplus \overline{\mathcal{Z}}_k^{I_{k-1}}=\mathcal{Z}_k}
    l(\mathcal{Z}_k^\mathrm{U}|\mathcal{P}_k^\mathrm{U},\mathbf{s}_k) 
    \prod_{i=1}^{I_{k-1}} t(\mathcal{Z}_k^i|\mathcal{X}_k^i,\mathbf{s}_k ),\label{eq:likelihood_partitioning_conjugate_prior}
\end{align}
where $\mathcal{Z}_k^\mathrm{U}$ represents a measurement set including measurement elements that are generated from both targets in $\mathcal{X}_k^\mathrm{U}$ and clutter, and $t(\mathcal{Z}_k^i|\mathcal{X}_k^i)$ is the likelihood for a set with zero or one measurement element without clutter, given by
\begin{align}
    &t(\mathcal{Z}_k^{i}|\mathcal{X}_k^{i},\mathbf{s}_k) \notag\\
    &=
    \begin{cases}
        p_\mathrm{D}(\mathbf{x}_k,\mathbf{s}_k) l(\mathbf{z}_k|\mathbf{x}_k,\mathbf{s}_k), & \mathcal{Z}_k^{i}=\{ \mathbf{z}_k\} ,\mathcal{X}_k^{i}=\{ \mathbf{x}_k\} \\
        1-p_\mathrm{D}(\mathbf{x}_k,\mathbf{s}_k), & \mathcal{Z}_k^{i}=\emptyset,\mathcal{X}_k^{i}=\{ \mathbf{x}_k\} \\
        1, & \mathcal{Z}_k^{i}=\emptyset,\mathcal{X}_k^{i}=\emptyset \\
        0, & \mathrm{otherwise}.
    \end{cases}\label{eq:t_z_x}
\end{align}
In addition, we know that, for $\mathcal{P}_k^\mathrm{U}=\mathcal{X}_k^\mathrm{U}\uplus \mathcal{Y}_k^{1}\uplus\dots\uplus \mathcal{Y}_k^{J_{k}}$,
the multi-target likelihood meets 
\begin{align}
l(\mathcal{Z}_k|\mathcal{P}_k^\mathrm{U},\mathbf{s}_k) & =l_{o}(\mathcal{Z}_k|\mathcal{X}_k^\mathrm{U},\mathcal{Y}_k^{1},\dots,\mathcal{Y}_k^{J_{k}},\mathbf{s}_k).
\label{eq:likelihood_equivalence_proof}
\end{align}

\subsection{Posterior}
    From~\cite[Eq.~(34)]{Angel_PMBM_TAES2018}, the posterior density is then
\begin{align}
&f_\mathtt{u}(\mathbf{s}_k,\mathcal{X}_k) \propto l(\mathcal{Z}_k|\mathcal{X}_k,\mathbf{s}_k)f_\mathtt{p}(\mathbf{s}_k,\mathcal{X}_k)\\
& =f_\mathtt{p}(\mathbf{s}_k)
    \sum_{\mathcal{Z}_k^\mathrm{U}\uplus \overline{\mathcal{Z}}_k^{I_{k-1}}=\mathcal{Z}_k}
    \sum_{\mathcal{P}_k^\mathrm{U}\uplus \overline{\mathcal{X}}_k^{I_{k-1}}=\mathcal{X}_k}
    l(\mathcal{Z}_k^\mathrm{U}|\mathcal{P}_k^\mathrm{U},\mathbf{s}_k) f_\mathtt{p}^\mathrm{U}(\mathcal{P}_k^\mathrm{U}) \notag \\
& ~~~ \times \prod_{i=1}^{I_{k-1}}t(\mathcal{Z}_k^{i}|\mathcal{X}_k^{i},\mathbf{s}_k)f_\mathtt{p}^{i}(\mathcal{X}_k^{i})\\
& =f_\mathtt{p}(\mathbf{s}_k)
    \sum_{\mathcal{Z}_k^\mathrm{U}\uplus \overline{\mathcal{Z}}_k^{I_{k-1}}=\mathcal{Z}_k}
    \sum_{\mathcal{P}_k^\mathrm{U}\uplus \overline{\mathcal{X}}_k^{I_{k-1}}=\mathcal{X}_k}
    \sum_{\mathcal{X}_k^\mathrm{U}\uplus \overline{\mathcal{Y}}_k^{J_{k}}=\mathcal{P}_k^\mathrm{U}}
    \notag\\
& ~~~\times f_\mathtt{p}^\mathrm{U}(\mathcal{X}_k^\mathrm{U}) [1-p_\mathrm{D}(\cdot,\mathbf{s}_k)]^{\mathcal{X}_k^\mathrm{U}}
    \notag\\
    &~~~\times \prod_{j=1}^{J_k}\big[\chi_{\mathcal{Z}_k^\mathrm{U}}(\mathbf{z}_k^{j})\tilde{l}(\mathbf{z}_k^{j}|\mathcal{Y}_k^{j},\mathbf{s}_k)f_\mathtt{p}^\mathrm{U}(\mathcal{Y}_k^{j}) \notag\\
&~~~+(1-\chi_{\mathcal{Z}_k^\mathrm{U}}(\mathbf{z}_k^{j}))\delta_{\emptyset}(\mathcal{Y}_k^{j})\big] \prod_{i=1}^{I_{k-1}}t(\mathcal{Z}_k^{i}|\mathcal{X}_k^{i},\mathbf{s}_k)f_\mathtt{p}^{i}(\mathcal{X}_k^{i}),
\end{align}
    where $\chi_{\mathcal{Z}_k^\mathrm{U}}(\mathbf{z}_k^{j})$ is defined to be 1 if $\mathbf{z}_k^{j} \in \mathcal{Z}_k^\mathrm{U}$, and to be 0 otherwise.
Combining the last two convolution sums into one, we have
\begin{align}
&f_\mathtt{u}(\mathbf{s}_k,\mathcal{X}_k) \notag\\
& \propto f_\mathtt{p}(\mathbf{s}_k)
    \sum_{\mathcal{Z}_k^\mathrm{U}\uplus \overline{\mathcal{Z}}_k^{I_{k-1}}=\mathcal{Z}_k}
    \sum_{\mathcal{X}_k^\mathrm{U}\uplus  
    \overline{\mathcal{Y}}_k^{J_{k}}
    \uplus \overline{\mathcal{X}}_k^{I_{k-1}}=\mathcal{X}_k}
    f_\mathtt{p}^\mathrm{U}(\mathcal{X}_k^\mathrm{U}) \notag\\
& ~~~\times[1-p_\mathrm{D}(\cdot,\mathbf{s}_k)]^{\mathcal{X}_k^\mathrm{U}}
    \prod_{j=1}^{J_k}\big[\chi_{\mathcal{Z}_k^\mathrm{U}}(\mathbf{z}_k^{j})\tilde{l}(\mathbf{z}_k^{j}|\mathcal{Y}_k^{j},\mathbf{s}_k)f_\mathtt{p}^\mathrm{U}(\mathcal{Y}_k^{j})\notag\\
&~~~+(1-\chi_{\mathcal{Z}_k^\mathrm{U}}(\mathbf{z}_k^{j}))\delta_{\emptyset}(\mathcal{Y}_k^{j})\big] \prod_{i=1}^{I_{k-1}}t(\mathcal{Z}_k^{i}|\mathcal{X}_k^{i},\mathbf{s}_k)f_\mathtt{p}^{i}(\mathcal{X}_k^{i}).
\end{align}

We can now add auxiliary variables for the previous targets $\mathcal{X}_k^{1},\dots,\mathcal{X}_k^{I_{k-1}}$
and also for the newly detected targets $\mathcal{Y}_k^{1},\dots,\mathcal{Y}_k^{J_k}$.
The auxiliary variables for the newly detected targets will start
from $I_{k-1}+1$. Then, we have
\begin{align}
\tilde{f}_\mathtt{u}(\mathbf{s}_k,\mathcal{\tilde{X}}_{k}) 
\propto & f_\mathtt{p}(\mathbf{s})
    \sum_{\mathcal{Z}_k^\mathrm{U}\uplus \overline{\mathcal{Z}}_k^{I_{k-1}}=Z_{k}}
    \tilde{f}_\mathtt{p}^\mathrm{U}(\tilde{\mathcal{X}}_k^\mathrm{U})[1-p_\mathrm{D}(\cdot,\mathbf{s}_k)]^{\tilde{\mathcal{X}}_k^\mathrm{U}} \notag\\
& \times
    \prod_{j=1}^{J_k}\big[\chi_{\mathcal{Z}_k^\mathrm{U}}(\mathbf{z}_k^{j})\tilde{l}(\mathbf{z}_k^{j}|\tilde{\mathcal{Y}}_k^{j},\mathbf{s}_k)\tilde{f}_\mathtt{p}^\mathrm{U}(\tilde{\mathcal{Y}}_k^{j})\notag\\
    &+(1-\chi_{\mathcal{Z}_k^\mathrm{U}}(\mathbf{z}_k^{j}))\delta_{\emptyset}(\tilde{\mathcal{Y}}_k^{j})\big]\notag\\ 
    & \times \prod_{i=1}^{I_{k-1}}t(\mathcal{Z}_k^{i}|\tilde{\mathcal{X}}_k^{i},\mathbf{s}_k)\tilde{f}_\mathtt{p}^{i}(\tilde{\mathcal{X}}_k^{i}).
    \label{eq:App_Joint_3}
\end{align}
    By jointly considering the data associations $\mathbf{c}_k$ and $\mathbf{d}_k$ introduced in~\eqref{eq:factorized3}, we can introduce the measurement sets as follows: 
\begin{align}
    \mathcal{Z}_k^i =& 
    \begin{cases}
        \{\mathbf{z}_k^{c_k^i}\} & c_k^i >0\\
        \emptyset & c_k^i =0
    \end{cases},\\
    \mathcal{Z}_k^\mathrm{U} =&
    \{\mathbf{z}_k^j: d_k^j = 0\}.
\end{align}
    With the association function $\psi(\mathbf{c}_{k}^i,\mathbf{d}_{k}^j)$ introduced in~\eqref{eq:factorized3}, we can rewrite \eqref{eq:App_Joint_3} as 
\begin{align}
    \tilde{f}_\mathtt{u}(\mathbf{s}_k,\mathcal{\tilde{X}}_{k}) 
    \propto & f_\mathtt{p}(\mathbf{s})\sum_{\mathbf{c}_k,\mathbf{d}_k} \tilde{f}_\mathtt{p}^\mathrm{U}(\tilde{\mathcal{X}}_k^\mathrm{U})[1-p_\mathrm{D}(\cdot,\mathbf{s}_k)]^{\tilde{\mathcal{X}}_k^\mathrm{U}} \label{eq:App_Joint_4}\\
    & \times
    \prod_{j=1}^{J_k}\big[\chi_{\mathcal{Z}_k^\mathrm{U}}(\mathbf{z}_k^{j})\tilde{l}(\mathbf{z}_k^{j}|\tilde{\mathcal{Y}}_k^{j},\mathbf{s}_k,d_k^j)\tilde{f}_\mathtt{p}^\mathrm{U}(\tilde{\mathcal{Y}}_k^{j})\notag\\
    &+(1-\chi_{\mathcal{Z}_k^\mathrm{U}}(\mathbf{z}_k^{j}))\delta_{\emptyset}(\tilde{\mathcal{Y}}_k^{j})\big]\notag\\ 
    & \times \prod_{i=1}^{I_{k-1}}t(\mathcal{Z}_k^{i}|\tilde{\mathcal{X}}_k^{i},\mathbf{s}_k,c_k^i)\tilde{f}_\mathtt{p}^{i}(\tilde{\mathcal{X}}_k^{i}),\nonumber 
\end{align}
    where $\tilde{l}(\mathbf{z}_k^{j}|\tilde{\mathcal{Y}}_k^{j},\mathbf{s}_k,d_k^j)$ and $t(\mathcal{Z}_k^{i}|\tilde{\mathcal{X}}_k^{i},\mathbf{s}_k,c_k^i)$ were defined in~\eqref{eq:likelihoodPoi} and~\eqref{eq:likelihoodBer}, respectively.
    We now make association variables $\mathbf{c}_k$ and $\mathbf{d}_k$ explicit in the posterior to define the density.
    We can then define the joint update density as
\begin{align}
    \tilde{f}_\mathtt{u}(\mathbf{s}_k,\tilde{\mathcal{X}}_k,\mathbf{c}_k,\mathbf{d}_k)  \propto & f_\mathtt{p}(\mathbf{s}_k)\tilde{f}_\mathtt{p}^\mathrm{U}(\tilde{\mathcal{X}}_k^\mathrm{U})[1-p_\mathrm{D}(\cdot,\mathbf{s}_k)]^{\tilde{\mathcal{X}}_k^\mathrm{U}}\label{eq:Factor_graph_update} \\
    & \times
    \prod_{j=1}^{J_k}\tilde{l}(\mathbf{z}_k^{j}|\tilde{\mathcal{Y}}_k^{j},\mathbf{s}_k,d_k^j)\tilde{f}_\mathtt{p}^\mathrm{U}(\tilde{\mathcal{Y}}_k^{j})\notag\\
    & \times\prod_{i=1}^{I_{k-1}}t(\mathcal{Z}_k^{i}|\tilde{\mathcal{X}}_k^{i},\mathbf{s}_k,c_k^i)\tilde{f}_\mathtt{p}^{i}(\tilde{\mathcal{X}}_k^{i})\psi(c_{k}^{i},d_{k}^{j}),\nonumber 
\end{align}
\bibliographystyle{IEEEtran}
\bibliography{bibliography}

\begin{IEEEbiography}
[{\includegraphics[width=1in,height=1.25in,clip,keepaspectratio]{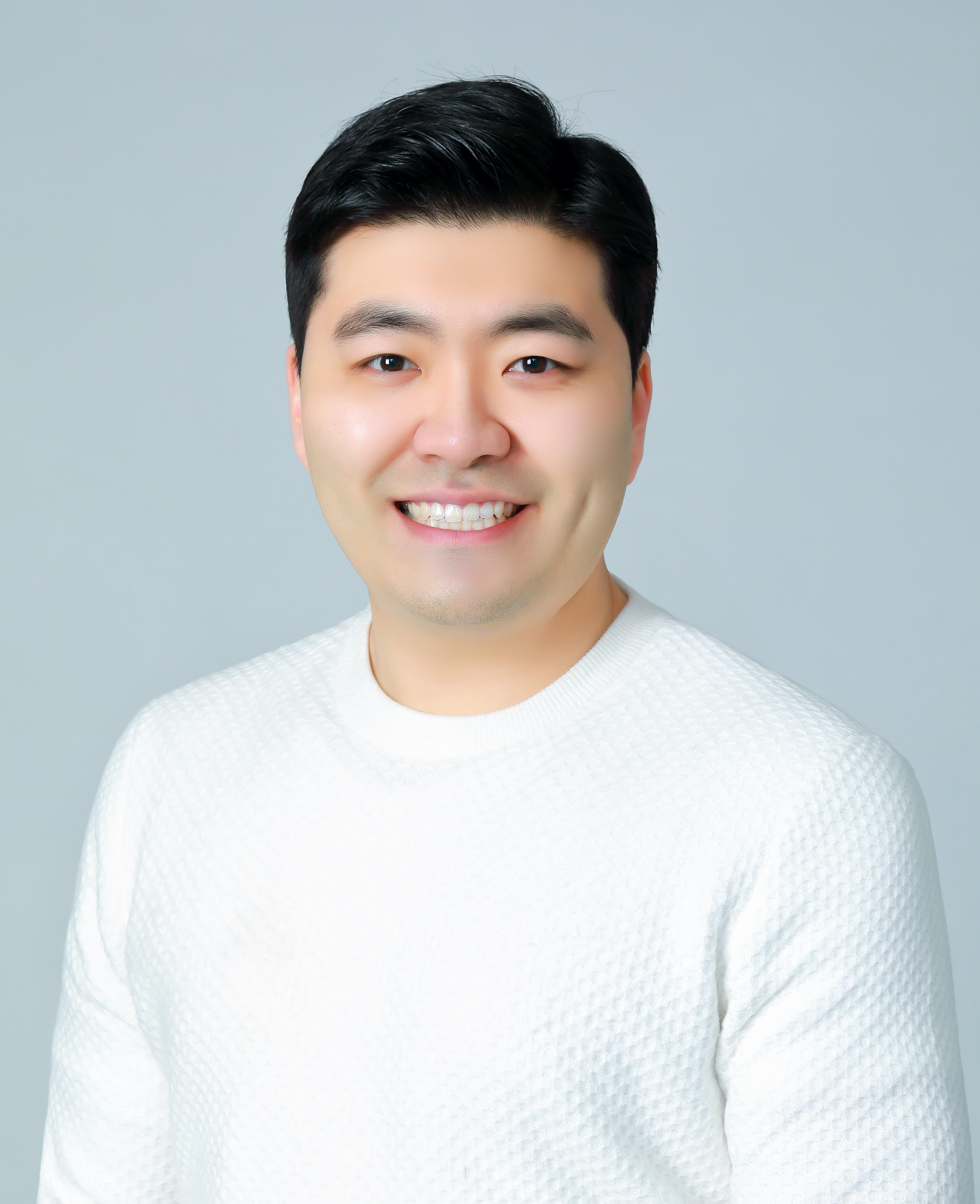}}]{Hyowon Kim}
received the Ph.D. degree from the Department of Electronic Engineering, Hanyang University, Seoul, Korea in 2021. 
He is currently an Assistant Professor in the Department of Electronics Engineering at Chungnam National University, Daejeon, Korea.
He was a Marie Skłodowska-Curie Fellow/Postdoctoral Researcher in the Department of Electrical Engineering at Chalmers University of Technology, Sweden, from 2021 to 2023. He was a Visiting Researcher with the Department of Electrical Engineering, Chalmers University of Technology, Sweden, from 2019 to 2020. His main research interests include wireless communications and integrated sensing/localization/communications, in 5G and Beyond 5G communication systems.
\end{IEEEbiography}


\begin{IEEEbiography}
[{\includegraphics[width=1in,height=1.25in,clip,keepaspectratio]{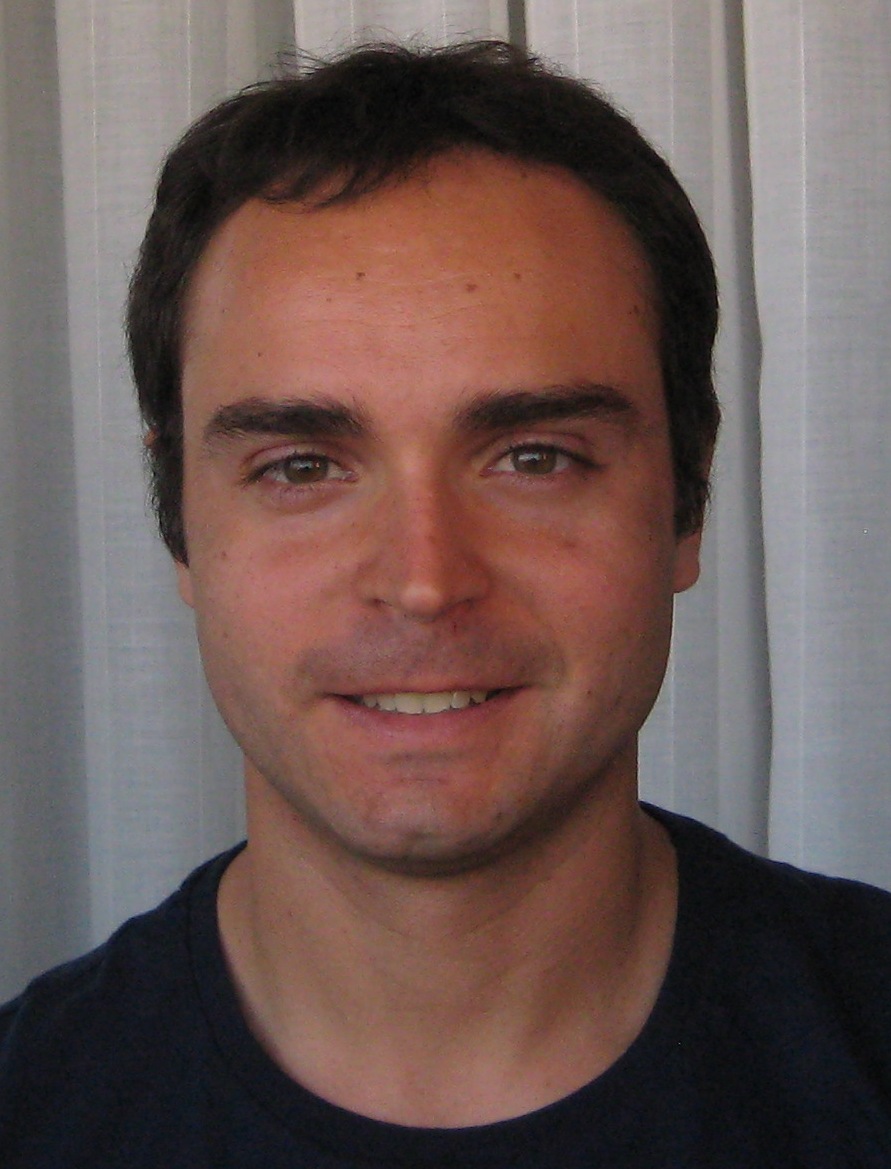}}]{{\'{A}}ngel F. Garc{\'{i}}a-Fern{\'{a}}ndez} received the telecommunication engineering degree and the Ph.D. degree from Universidad Politécnica de Madrid, Madrid, Spain, in 2007 and 2011, respectively. 
He is currently a Senior Lecturer in the Department of Electrical Engineering and Electronics at the University of Liverpool, Liverpool, UK. He previously held postdoctoral positions at Universidad Politécnica de Madrid, Chalmers University of Technology, Gothenburg, Sweden, Curtin University, Perth, Australia, and Aalto University, Espoo, Finland. His main research activities and interests are in the area of Bayesian inference, with emphasis on dynamic systems and multiple target tracking.
\end{IEEEbiography}


\begin{IEEEbiography}
[{\includegraphics[width=1in,height=1.25in,clip,keepaspectratio]{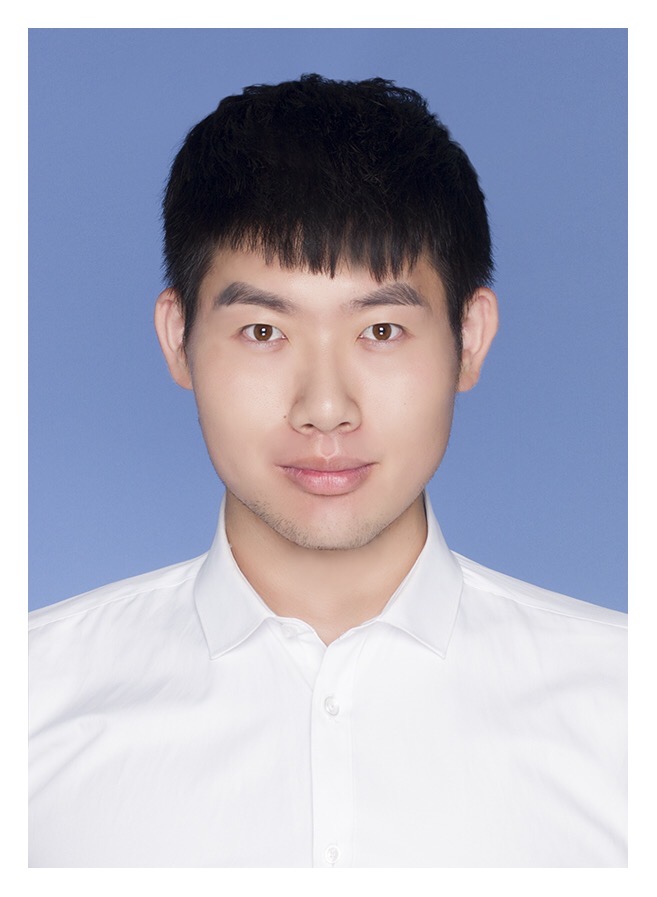}}]{Yu Ge}
(S'20) received his B.E. degree from Zhejiang University, Hangzhou, China, in 2017, and M.Sc. degree from the KTH Royal Institute of Technology, Stockholm, Sweden, in 2019. He is currently a Ph.D. candidate in the Department of Electrical and Engineering at Chalmers University of Technology, Sweden. His research interests include integrated communication and sensing, wireless positioning systems, simultaneous localization and mapping, and multi-object tracking, particularly in 5G and Beyond 5G scenarios.
\end{IEEEbiography}


\begin{IEEEbiography}
[{\includegraphics[width=1in,height=1.25in,clip,keepaspectratio]{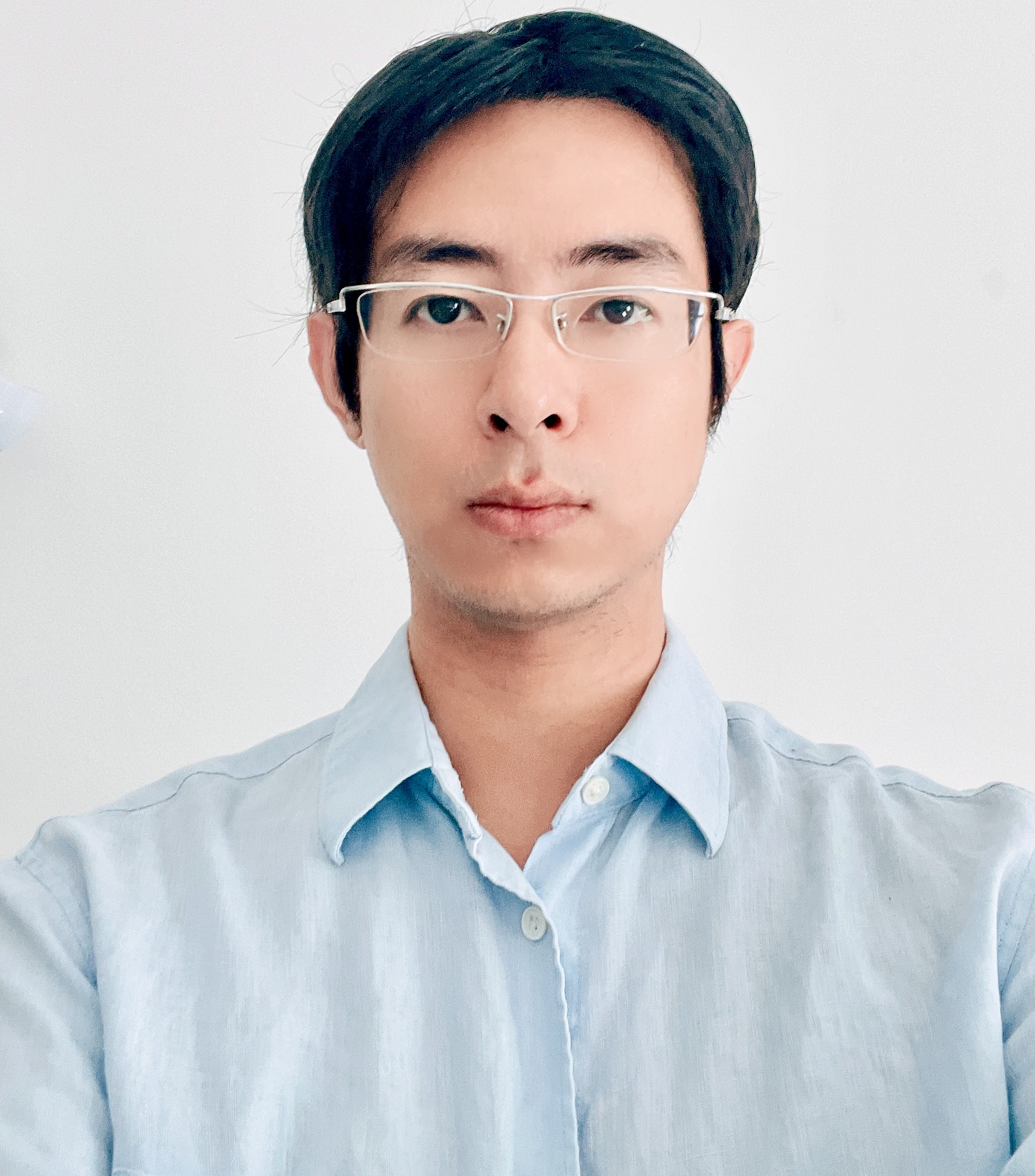}}]{Yuxuan Xia} received his M.Sc. degree in communication engineering and Ph.D. degree in signal processing both from Chalmers University of Technology, Gothenburg, Sweden, in 2017 and 2022, respectively. After obtaining his Ph.D. degree, he stayed at the same research group as a postdoctoral researcher for a year. He is currently an industrial postdoctoral researcher with Zenseact AB, Gothenburg, Sweden and also affiliated with the Division of Automatic Control, Linköping University, Linköping, Sweden. His main research interests include sensor fusion, multi-object tracking and SLAM, especially for automotive applications. 
\end{IEEEbiography}

\begin{IEEEbiography}
[{\includegraphics[width=1in,height=1.25in,clip,keepaspectratio]{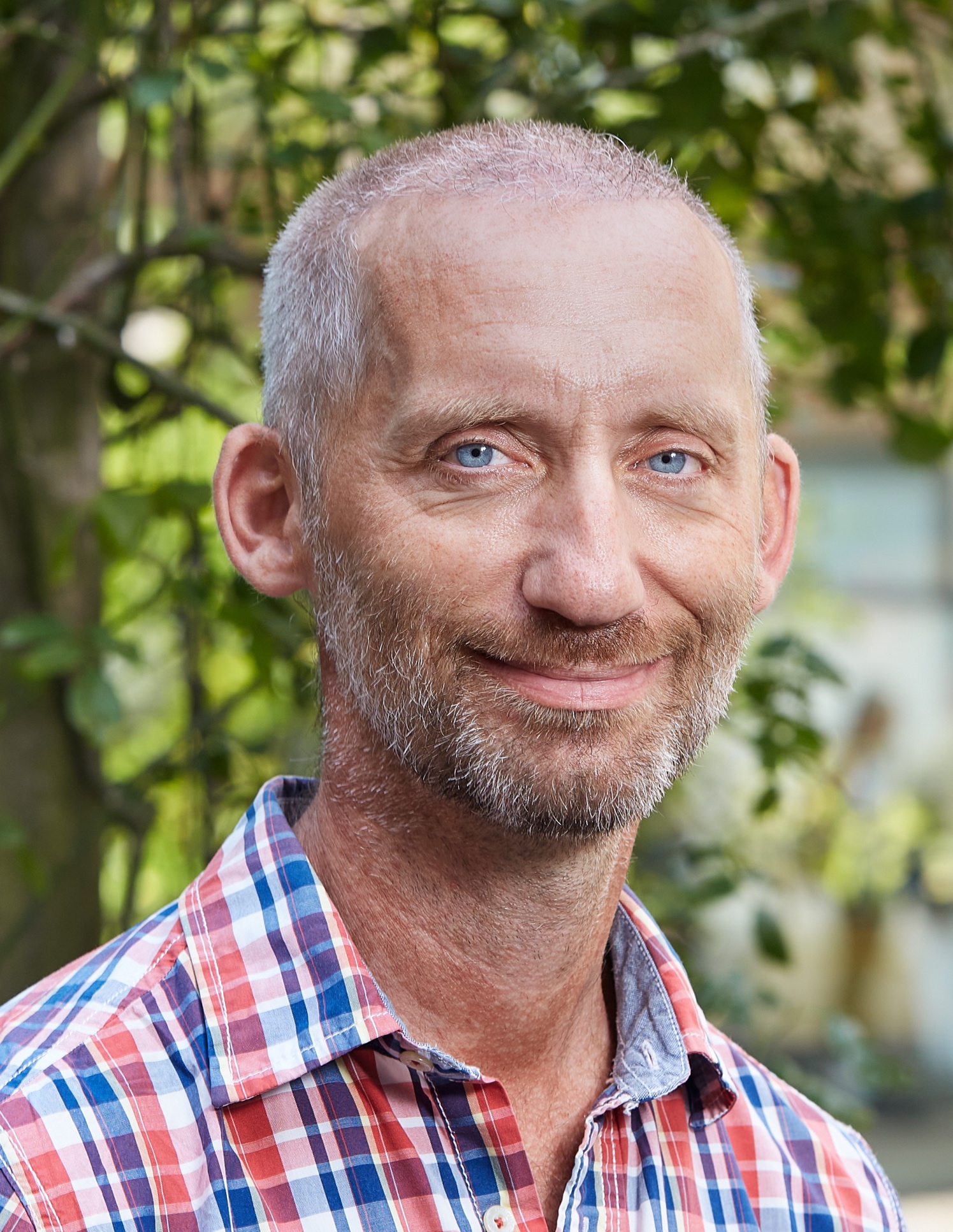}}]{Lennart Svensson} (M'09-SM'10)
is a Professor of Signal Processing with the Chalmers University of Technology. His main research interests include machine learning and Bayesian inference in general, and nonlinear filtering, deep learning, and tracking in particular. He has organized a massive open online course on multiple object tracking, available on edX and YouTube, and received paper awards at the International Conference on Information Fusion in 2009, 2010, 2017, and 2019.
\end{IEEEbiography}

\begin{IEEEbiography}
[{\includegraphics[width=1in,height=1.25in,clip,keepaspectratio]{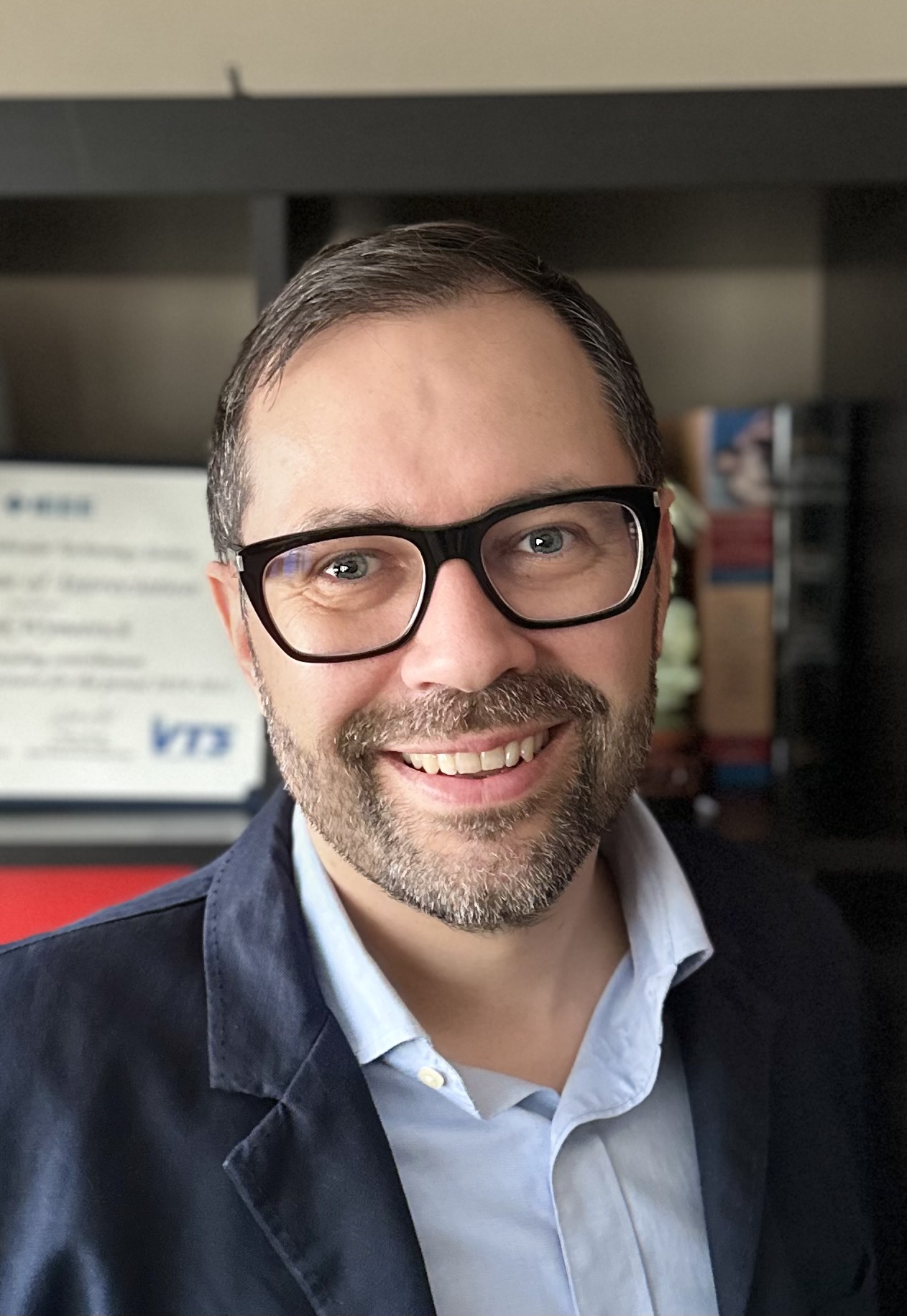}}]{Henk Wymeersch}
(S'01, M'05, SM'19, F'24) obtained the Ph.D. degree in Electrical Engineering/Applied Sciences in 2005 from Ghent University, Belgium. He is currently a Professor of Communication Systems with the Department of Electrical Engineering at Chalmers University of Technology, Sweden. Prior to joining Chalmers, he was a postdoctoral researcher from 2005 until 2009 with the Laboratory for Information and Decision Systems at the Massachusetts Institute of Technology. Prof. Wymeersch served as Associate Editor for IEEE Communication Letters (2009-2013), IEEE Transactions on Wireless Communications (since 2013), and IEEE Transactions on Communications (2016-2018) and is currently Senior Member of the IEEE Signal Processing Magazine Editorial Board.  During 2019-2021, he was an IEEE Distinguished Lecturer with the Vehicular Technology Society.  His current research interests include the convergence of communication and sensing, in a 5G and Beyond 5G context. 
\end{IEEEbiography}

\newpage
\pagestyle{empty}
\beginsupplement
\makeatletter

\noindent \large{\textbf{Set-Type Belief Propagation with Applications to Poisson Multi-Bernoulli SLAM: Supplementary Material}}
\normalsize

\section{Set-Type BP Messages and Beliefs}
\label{SM:Messages}
    The messages and beliefs on the factor graph (Fig.~\ref{Fig:FG}) are computed as follows.

\subsection{Prediction}
    The prediction step is cascaded to the updated step of the previous time step, where
    we have a sensor state vector and $I_{k-1} +1$ sets: 1 for undetected targets that have never been detected; $I_{k-1}$ for detected targets that have been previously detected.
    
\subsubsection{Sensor} 
    We compute the message of the predicted density for the sensor state. The belief of the sensor variable $\mathbf{s}_{k-1}$ is denoted by $\mathtt{b}(\mathbf{s}_{k-1})$, obtained from message \fbox{\footnotesize{27}} at the previous time step.
\begin{enumerate}
\item[\fbox{\footnotesize{1}}]
    A message from the the sensor variable $\mathbf{s}_{k-1}$ to linked factor $A$, i.e., $f^A(\mathbf{s}_{k},\mathbf{s}_{k-1})=f(\mathbf{s}_{k}|\mathbf{s}_{k-1})$, is denoted by $\mathtt{n}_{\mathsf{s}_{k-1} \to A}(\mathbf{s}_{k-1})$, and $\mathtt{n}_{\mathsf{s}_{k-1} \to A}(\mathbf{s}_{k-1})=\mathtt{b}(\mathbf{s}_{k-1})$.
\item[\fbox{\footnotesize{2}}]
    A message from factor $A$
    to the variable $\mathbf{s}_{k}$ is denoted by $\mathtt{m}_{A \to \mathsf{s}_k}(\mathbf{s}_{k})$, given by
\begin{align}
    \mathtt{m}_{A \to \mathsf{s}_k}(\mathbf{s}_{k}) = \int \mathtt{n}_{\mathsf{s}_{k-1} \to A}(\mathbf{s}_{k-1}) f(\mathbf{s}_{k}|\mathbf{s}_{k-1}) \mathrm{d}\mathbf{s}_{k-1}.
\end{align}
\end{enumerate}

\subsubsection{Undetected Targets}
    The undetected targets will either remain undetected again or be detected for the first time.
    Thus, we compute the messages of the predicted densities for the undetected targets and newly detected targets.
    A belief of undetected target set ${\mathcal{X}}_{k-1}^\mathrm{U}$ is denoted by $\mathtt{b}({\mathcal{X}}_{k-1}^\mathrm{U})$, obtained from message \fbox{\footnotesize{22}} at the previous time step. 
    
\begin{enumerate}
\item[\fbox{\footnotesize{3}}]
    A message from the set-variable ${\mathcal{X}}_{k-1}^\mathrm{U}$ to the linked factor $B$, i.e., $f^B({\mathcal{X}}_{k}^\mathrm{S},{\mathcal{X}}_{k-1}^\mathrm{U}) = f({\mathcal{X}}_{k}^\mathrm{S}|{\mathcal{X}}_{k-1}^\mathrm{U})$, is denoted by $\mathtt{n}_{\mathsf{X}_{k-1}^\mathrm{U} \to B}({\mathcal{X}}_{k-1}^\mathrm{U})=\mathtt{b}({\mathcal{X}}_{k-1}^\mathrm{U})$ that follows the \ac{ppp} distribution.
\item[\fbox{\footnotesize{4}}]    
    A message from the factor $B$ to ${\mathcal{X}}_{k}^\mathrm{S}$ is denoted by $\mathtt{m}_{B \to{\mathsf{X}}_k^\mathrm{S} }({\mathcal{X}}_k^\mathrm{S})$, given by
\begin{align}
    \mathtt{m}_{B \to{\mathsf{X}}_k^\mathrm{S} }
    &({\mathcal{X}}_k^\mathrm{S})
    = \int \mathtt{n}_{\mathsf{X}_{k-1}^\mathrm{U} \to B}({\mathcal{X}}_{k-1}^\mathrm{U}) 
    f({\mathcal{X}}_{k}^\mathrm{S}|{\mathcal{X}}_{k-1}^\mathrm{U})  \delta {\mathcal{X}}_{k-1}^\mathrm{U},
\end{align}
    which follows the \ac{ppp} distribution.
\item[\fbox{\footnotesize{5}}]
    Messages from ${\mathcal{X}}_k^\mathrm{S}$
    and ${\mathcal{X}}_{k}^\mathrm{B}$ to 
    the linked factor $D$, 
    i.e., $f^D({\mathcal{X}}_{k}^\mathrm{S},{\mathcal{X}}_{k}^\mathrm{B},{\mathcal{P}}_{k}^\mathrm{U}) =
    \delta_{{\mathcal{X}}_{k}^\mathrm{S} \uplus {\mathcal{X}}_{k}^\mathrm{B}}({\mathcal{P}}_{k}^\mathrm{U})$,
    are denoted by $\mathtt{n}_{\mathsf{X}_{k}^\mathrm{S} \to D}({\mathcal{X}}_{k}^\mathrm{S})$, respectively, and $\mathtt{n}_{\mathsf{X}_{k}^\mathrm{B} \to D}({\mathcal{X}}_{k}^\mathrm{B})$, where $\mathtt{n}_{\mathsf{X}_{k}^\mathrm{B} \to D}({\mathcal{X}}_{k}^\mathrm{B})=f^C({\mathcal{X}}_{k}^\mathrm{B})
    =f^\mathrm{Poi}({\mathcal{X}}_{k}^\mathrm{B})$ 
    follows the \ac{ppp} distribution. 
\item[\fbox{\footnotesize{6}}]
    A message from the 
    factor $D$
    to ${\mathcal{P}}_{k}^\mathrm{U}$ is denoted by $\mathtt{m}_{D \to{\mathsf{P}}_k^\mathrm{U} }({\mathcal{P}}_k^\mathrm{U})$, given by
\begin{align}
    \mathtt{m}_{D \to{\mathsf{P}}_k^\mathrm{U} }
    ({\mathcal{P}}_k^\mathrm{U})
    =& \iint \mathtt{n}_{\mathsf{P}_{k}^\mathrm{S} \to D}({\mathcal{X}}_{k}^\mathrm{S}) 
    \mathtt{n}_{\mathsf{X}_{k}^\mathrm{B} \to D}({\mathcal{X}}_{k}^\mathrm{B})\notag \\
    &\times
    \delta_{{\mathcal{X}}_{k}^\mathrm{S} \uplus {\mathcal{X}}_{k}^\mathrm{B}}({\mathcal{P}}_{k}^\mathrm{U}) 
    \delta {\mathcal{X}}_{k}^\mathrm{S}
    \delta {\mathcal{X}}_{k}^\mathrm{B},
\end{align}
    which follows the \ac{ppp} distribution.
\end{enumerate}

\subsubsection{Previously Detected Targets} 
    We compute the messages of the predicted densities for the previously detected targets. Beliefs of previously detected target set-variables $\mathcal{X}_{k-1}^i$ are denoted by $\mathtt{b}(\mathcal{X}_{k-1}^i)$, for $i=\{1,\dots,I_{k-1}\}$, obtained from message \fbox{\footnotesize{20}} at the previous time step.
\begin{enumerate}
\item[\fbox{\footnotesize{7}}]
    Messages from the set-variable $\mathcal{X}_{k-1}^i$ to the the linked factors $E^i$, i.e., $f_i^E(\mathcal{X}_{k}^i,\mathcal{X}_{k-1}^i) =f(\mathcal{X}_{k}^i|\mathcal{X}_{k-1}^i)$, are denoted by $\mathtt{n}_{\mathsf{X}_{k-1}^i \to E^i}(\mathcal{X}_{k-1}^i)$, and $\mathtt{n}_{\mathsf{X}_{k-1}^i \to E^i}(\mathcal{X}_{k-1}^i)=\mathtt{b}(\mathcal{X}_{k-1}^i)$ for $i \in \mathcal{I}_{k-1}$.
\item[\fbox{\footnotesize{8}}]
    Messages from factors $E^i$
    to the linked detected set-variables $\mathcal{X}_{k}^i$ are denoted by $\mathtt{m}_{E^i \to \mathsf{X}_{k}^i}(\mathcal{X}_{k}^i)$, given by
\begin{align}
    \mathtt{m}_{E^i \to \mathsf{X}_{k}^i}(\mathcal{X}_{k}^i) = \int \mathtt{n}_{\mathsf{X}_{k-1}^i \to E^i}(\mathcal{X}_{k-1}^i) f(\mathcal{X}_{k}^i|\mathcal{X}_{k-1}^i) \delta \mathcal{X}_{k-1}^i
\end{align}
\end{enumerate}

\subsection{Update}
    The update step is cascaded to the prediction step from which the messages \fbox{\footnotesize{2}}, \fbox{\footnotesize{6}}, and \fbox{\footnotesize{8}} are obtained.
\subsubsection{Separation of Undetected Targets}
    The set of undetected targets is partitioned into $1$ set of targets that remains undetected and $J_k$ sets of newly detected targets.

\begin{enumerate}
\item[\fbox{\footnotesize{9}}]
    A message from the set-variable ${\mathcal{P}}_k^\mathrm{U}$ to factor $F$, i.e., $f^F({\mathcal{P}}_{k}^\mathrm{U},{\mathcal{X}}_{k}^\mathrm{U},  {\mathcal{P}}_k^1, \dots, {\mathcal{P}}_k^{J_k})=\delta_{ {\mathcal{X}}_{k}^\mathrm{U} \uplus {\mathcal{P}}_k }({\mathcal{P}}_k^\mathrm{U})$, is denoted by $\mathtt{n}_{{\mathsf{P}}_k^\mathrm{U} \to F}({\mathcal{P}}_k^\mathrm{U})$, where 
    $\mathtt{n}_{{\mathsf{P}}_k^\mathrm{U} \to F}({\mathcal{P}}_k^\mathrm{U})= \mathtt{m}_{D \to {\mathsf{P}}_k^\mathrm{U}}({\mathcal{P}}_k^\mathrm{U})$.
    Messages from the sets $\mathcal{X}_k^\mathrm{U}$, ${\mathcal{P}}_k^1,\dots,{\mathcal{P}}_k^{J_k}$ to factor $F$
    are $\mathtt{n}_{\mathsf{X}_k^\mathrm{U} \to F}(\mathcal{X}_k^\mathrm{U})=1, \mathtt{n}_{{\mathsf{P}}_k^j \to F}({\mathcal{P}}_k^j)=1$, for $j\in \{1,\dots,J_k\}$.

\item[\fbox{\footnotesize{10}}]
    A message from factor $F$ to $\mathcal{X}_k^\mathrm{U}$ is denoted by $\mathtt{m}_{F \to \mathsf{X}_k^\mathrm{U}}(\mathcal{X}_k^\mathrm{U})$. 
    With the Proposition~\ref{prop:Poisson}, $\mathtt{m}_{F \to \mathsf{X}_k^\mathrm{U}}(\mathcal{X}_k^\mathrm{U})$ is given by
    
\begin{align}
    \mathtt{m}_{F \to \mathsf{X}_k^\mathrm{U}}(\mathcal{X}_k^\mathrm{U}) & = 
    \int \mathtt{n}_{{\mathsf{P}}_k^\mathrm{U} \to F}({\mathcal{P}}_k^\mathrm{U})
    \delta_{ {\mathcal{X}}_{k}^\mathrm{U} \uplus {\mathcal{P}}_k }({\mathcal{P}}_k^\mathrm{U})
    \delta \mathcal{X}_k^{\sim \mathrm{U}}\\
    & \propto f^\mathrm{Poi}(\mathcal{X}_k^\mathrm{U}),
\end{align}
    proportional to the \ac{ppp} distribution.
    In similarly, messages from factor $F$ to  ${\mathcal{P}}_k^j$, for $j \in \{1, \dots, J_k\}$, are denoted by $\mathtt{m}_{F \to {\mathsf{P}}_k^j}({\mathcal{P}}_k^j)$, and with Proposition~\ref{prop:Poisson}, $\mathtt{m}_{F \to {\mathsf{P}}_k^j}({\mathcal{P}}_k^j)$ for $j \in \{1, \dots, J_k\}$ are given by
\begin{align}
    \mathtt{m}_{F \to {\mathsf{P}}_k^j}({\mathcal{P}}_k^j) 
    & = \int \mathtt{n}_{{\mathsf{P}}_k^\mathrm{U} \to F }({\mathcal{P}}_k^\mathrm{U}) 
    \delta_{ {\mathcal{X}}_{k}^\mathrm{U} \uplus {\mathcal{P}}_k }({\mathcal{P}}_k^\mathrm{U})
    \delta {\mathcal{P}}_k^{\sim j},\\
    & \propto f^\mathrm{Poi}({\mathcal{P}}_k^{j})
\end{align}
    also proportional to the \ac{ppp} distribution.
\item[\fbox{\footnotesize{11}}]
    Messages from the set-variable ${\mathcal{P}}_k^j$ to linked factor $G^j$, i.e., $f_j^G({\mathcal{P}}_k^j, {\mathcal{Y}}_k^{j}) =
    \delta_{\mathtt{h}_{-(I_{k-1}+j)}({\mathcal{Y}}_k^j)}({\mathcal{P}}_k^j)$, are denoted by $\mathtt{n}_{{\mathsf{P}}_k^j \to G^j}({\mathcal{P}}_k^j) = \mathtt{m}_{F \to {\mathsf{P}}_k^j}({\mathcal{P}}_k^j)$, for $j=\{1,\dots,J_k\}$.
\item[\fbox{\footnotesize{12}}]
    Messages from factor $G^j$
    to the linked set-variable ${\mathcal{Y}}_k^{j}$ are denoted by $\mathtt{m}_{G^j \to \mathsf{Y}_k^{j}}({\mathcal{Y}}_k^j)$.
    With the Definition~\ref{def:ConvertFactor}, the messages $\mathtt{m}_{G^j \to \mathsf{Y}_k^{j}}({\mathcal{Y}}_k^j)$ are given by
\begin{align}
    \mathtt{m}_{G^j \to \mathsf{Y}_k^{j}}({\mathcal{Y}}_k^j) 
    &= 
    \int  \mathtt{n}_{F \to {\mathsf{P}}_k^j}({\mathcal{P}}_k^j) \delta_{\mathtt{h}_{-(I_{k-1}+j)}({\mathcal{Y}}_k^j)}({\mathcal{P}}_k^j),
\end{align}
    which follows the \ac{ppp} distribution with auxiliary variables $u=I_{k-1}+j$.
\end{enumerate}

\subsubsection{Data Association}
    We compute the messages of marginal probabilities for $c_k^i$ and $d_k^j$ by running loopy \ac{bp}~\cite{Jason_PMB_TAES2015} on the factor graph with cycle. Initial association probabilities are determined by the predicted messages and their linked likelihood factors.
\begin{enumerate}
\item[\fbox{\footnotesize{13}}]
    Messages from the sensor state variable $\mathbf{s}_k$ to linked factors $I^j$, i.e., $f_j^I(\mathbf{s}_k,\mathcal{Y}_k^j,d_k^j) = \tilde{l}(\mathbf{z}_k^j|\mathbf{s}_k,\mathcal{Y}_k^j,d_k^j)$ for $j \in \mathcal{J}_{k}$, and $J^i$, i.e., $f_i^J(\mathbf{s}_k,\mathcal{X}_k^i,c_k^i) = t(\mathcal{Z}_k^i|\mathbf{s}_k,\mathcal{X}_k^i,c_k^i)$ for $i
    \in \mathcal{I}_{k-1}$, are respectively denoted by $\mathtt{n}_{\mathsf{s}_{k} \to I^j}(\mathbf{s}_{k})$, and $\mathtt{n}_{\mathsf{s}_{k} \to J^i}(\mathbf{s}_{k})$, given by $\mathtt{n}_{\mathsf{s}_{k} \to I^j}(\mathbf{s}_{k}) = \mathtt{n}_{\mathsf{s}_{k} \to J^i}(\mathbf{s}_{k}) = \mathtt{m}_{A \to \mathsf{s}_{k}}(\mathbf{s}_{k})$.
    Messages from the set-variables $\mathcal{Y}_k^j$ to factor $I^j$ 
    are denoted by $\mathtt{n}_{\mathsf{Y}_k^j \to I^j}(\mathcal{Y}_k^j)$, and $\mathtt{n}_{\mathsf{Y}_k^j \to I^j}(\mathcal{Y}_k^j)=\mathtt{m}_{G^j \to \mathsf{Y}_k^j}(\mathcal{Y}_k^j)$, for $j \in \{1,\dots,J_k\}$.
    Messages from the set-variable $\mathcal{X}_k^i$ to linked factors $J^i$
    are denoted by $\mathtt{n}_{\mathsf{X}_k^i \to J^i}(\mathcal{X}_k^i)$, and $\mathtt{n}_{\mathsf{X}_k^i \to J^i}(\mathcal{X}_k^i)=\mathtt{m}_{E^i \to \mathsf{X}_k^i}(\mathcal{X}_k^i)$, for $i
    \in \mathcal{I}_{k-1}$.

\item[\fbox{\footnotesize{14}}]
    Messages from the linked factors $I^j$
    to the the linked variables $d_{k}^j$ are denoted by $\mathtt{m}_{I^j \to \mathsf{d}_{k}^j}(d_{k}^j)$.
    The messages $\mathtt{m}_{I^j \to \mathsf{d}_{k}^j}(d_{k}^j)$, for $j
    \in \mathcal{J}_{k}$, are given by
\begin{align}
    &\mathtt{m}_{I^j \to \mathsf{d}_{k}^j}(d_{k}^j) \notag
    \\
    &= \iint \mathtt{n}_{\mathsf{s}_{k} \to F^j}(\mathbf{s}_{k}) \mathtt{n}_{\mathsf{Y}_k^j \to I^j}(\mathcal{Y}_k^j)
    \tilde{l}(\mathbf{z}_k^j|\mathbf{s}_k,\mathcal{Y}_k^j,d_k^j) \mathrm{d}\mathbf{s}_{k} \delta \mathcal{Y}_{k}^j.
\end{align}
    Messages from the linked factors $J^i$
    to the linked variables $c_{k}^i$ are denoted by 
    $\mathtt{m}_{J^i \to \mathsf{c}_{k}^i}(c_{k}^i)$. The messages $\mathtt{m}_{J^i \to \mathsf{c}_{k}^i}(c_{k}^i)$, for $i
    \in \mathcal{I}_{k-1}$,
    are given by
\begin{align}
    &\mathtt{m}_{J^i \to \mathsf{c}_{k}^i}(c_{k}^i) \notag \\
    &= \iint \mathtt{n}_{\mathsf{s}_{k} \to J^i}(\mathbf{s}_{k}) \mathtt{n}_{\mathsf{X}_k^i \to J^i}(\mathcal{X}_k^i) t(\mathcal{Z}_k^i|\mathbf{s}_k,\mathcal{X}_k^i,c_k^i) \mathrm{d}\mathbf{s}_{k} \delta \mathcal{X}_{k}^i.
\end{align}
\item[\fbox{\footnotesize{15}}]
    During $L$ iteration, loopy \ac{bp}~\cite{Jason_PMB_TAES2015} between target-oriented data association variables $c_{k}^i$ and measurement-oriented data association variables $d_{k}^j$ with the factors $K^{i,j}$, i.e., $f_{i,j}^K(c_{k}^i,d_{k}^j) = \psi(c_{k}^i,d_{k}^j)$, is performed. 
    Messages from the 
    variables $c_{k}^i$ to $d_{k}^j$ and from the variables $d_{k}^j$ to $c_{k}^i$ at the $l$-th iteration are respectively denoted by $\mathtt{m}_{\mathsf{d}_{k}^j \to \mathsf{c}_{k}^i}^{(l)}(c_{k}^i)$ and $\mathtt{m}_{\mathsf{c}_{k}^i \to \mathsf{d}_{k}^j}^{(l)}(d_{k}^j)$, given by
\begin{align}
    &\mathtt{m}_{\mathsf{d}_{k}^j \to \mathsf{c}_{k}^i}^{(l)}(c_{k}^i) \notag \\
    &= \sum_{j=0}^{J_k} \mathtt{m}_{J^i \to \mathsf{c}_{k}^i}(c_{k}^i) \psi(c_k^i,d_k^j)\prod_{i'\in \mathcal{I}_{k-1}\setminus\{i\}} \mathtt{m}_{\mathsf{c}_{k}^{i'} \to \mathsf{d}_{k}^j}^{(l-1)}(d_{k}^j)\\
    &\mathtt{m}_{\mathsf{c}_{k}^i \to \mathsf{d}_{k}^j}^{(l)}(d_{k}^j) \notag \\
    &= \sum_{i=0}^{I_{k-1}} \mathtt{m}_{I^j \to \mathsf{d}_{k}^j}(d_{k}^j) \psi(c_k^i,d_k^j)\prod_{j'\in \mathcal{J}_k\setminus\{j\} }\mathtt{m}_{\mathsf{d}_{k}^{j'} \to \mathsf{c}_{k}^i}^{(l-1)}(c_{k}^i).
\end{align}
    Their initial messages are given by
    $\mathtt{m}_{\mathsf{d}_{k}^j \to \mathsf{c}_{k}^i}^{(0)}(c_{k}^i) = \sum_{j=0}^{J_k} \mathtt{m}_{J^i \to \mathsf{c}_{k}^i}(c_{k}^i) \psi(c_k^i,d_k^j)$ and $\mathtt{m}_{\mathsf{c}_{k}^i \to \mathsf{d}_{k}^j}^{(0)}(d_{k}^j) = \sum_{i=0}^{I_{k-1}} \mathtt{m}_{I^j \to \mathsf{d}_{k}^j}(d_{k}^j) \psi(c_k^i,d_k^j)$.
\item[\fbox{\footnotesize{16}}]
    Messages from the variables $c_{k}^i$ to factor $J^i$
    are denoted by $\mathtt{n}_{c_{k}^i \to J^i}(c_{k}^i)$.
    The messages $\mathtt{n}_{c_{k}^i \to J^i}(c_{k}^i)$, for $i
    \in \mathcal{I}_{k-1}$, are given by
    
\begin{align}
    \mathtt{n}_{c_{k}^i \to J^i}(c_{k}^i) = 
    \prod_{j \in \mathcal{J}_{k}}
    \mathtt{m}_{\mathsf{d}_{k}^j \to \mathsf{c}_{k}^i}^{(L)}(c_{k}^i).
\end{align}
    Messages from the variables $d_{k}^j$ to linked factors $I^j$ are denoted by $\mathtt{n}_{\mathsf{d}_{k}^j \to I^j}(d_{k}^j)$.
    The messages $\mathtt{n}_{\mathsf{d}_{k}^j \to I^j}(d_{k}^j)$, for $j
    \in \mathcal{J}_{k}$, are given by
\begin{align}
    \mathtt{n}_{\mathsf{d}_{k}^j \to I^j}(d_{k}^j) = 
    \prod_{i \in \mathcal{I}_{k-1}}
    \mathtt{m}_{\mathsf{c}_{k}^i \to \mathsf{d}_{k}^j }^{(L)}(d_{k}^j).
\end{align}
\end{enumerate}

\subsubsection{Belief Computation}
    We compute beliefs of previously detected targets, newly detected targets, undetected targets, and sensor state. The total number of beliefs is 2+$I_{k-1}+J_k$: 1 for the sensor state vector; 1 for undetected targets; and $I_{k-1}+J_k$ for detected targets.

    \paragraph{Newly Detected Targets}
    We compute $J_{k}$ beliefs of the set-variables $\mathcal{X}_k^j$ for $j \in \mathcal{J}_k$. Each set-variable $\mathcal{X}_k^j$ represents a target, newly detected for the first time or clutter, obtained as follows.
\begin{enumerate}
\item[\fbox{\footnotesize{17}}]
    Messages from factors $I^j$ 
    to the linked set-variables $\mathcal{Y}_k^j$ are denoted by $\mathtt{m}_{I^j \to \mathsf{Y}_k^j}(\mathcal{Y}_k^j)$.
    The messages $\mathtt{m}_{I^j \to \mathsf{Y}_k^j}(\mathcal{Y}_k^j)$, for $j \in \mathcal{J}_{k}$, are
    given by
\begin{align}
    \mathtt{m}_{I^j \to \mathsf{Y}_k^j}(\mathcal{Y}_k^j) 
    = & \int 
    \sum_{i=0}^{I_{k-1}} \mathtt{n}_{\mathsf{d}_{k}^j \to I^j}(d_{k}^j=i)
    \mathtt{n}_{\mathsf{s}_{k} \to I^j}(\mathbf{s}_{k}) \notag \\
    & \times 
    \tilde{l}(\mathbf{z}_k^j|\mathbf{s}_k,\mathcal{Y}_k^j,d_k^j=i) \mathrm{d}\mathbf{s}_{k}.
\end{align}
\item[\fbox{\footnotesize{18}}]
    Beliefs of the set-variables $\mathcal{Y}_k^j$ are denoted by $\mathtt{b}(\mathcal{Y}_k^j)$, for $j \in \mathcal{J}_{k}$.
    The beliefs $\mathtt{b}(\mathcal{Y}_k^j)$ are obtained by
\begin{align}
    \mathtt{b}(\mathcal{Y}_k^j) \propto \mathtt{m}_{G^j \to \mathcal{Y}_k^j}(\mathcal{Y}_k^j)\mathtt{m}_{I^j \to \mathcal{Y}_k^j}(\mathcal{Y}_k^j),
\end{align}
    which follow the Bernoulli distribution. 
\end{enumerate}

\paragraph{Previously Detected targets}
    We compute $I_{k-1}$ beliefs of the set-variables $\mathcal{X}_k^i$ for $i \in \mathcal{I}_{k-1}$. Each set-variable $\mathcal{X}_k^i$ represents the target that had been previously detected, obtained as follows. 
\begin{enumerate}
\item[\fbox{\footnotesize{19}}]
    Messages from factors $J^i$ to the set-variables 
    $\mathcal{X}_k^i$ are denoted by $\mathtt{m}_{J^i \to \mathsf{X}_k^i}(\mathcal{X}_k^i)$.
    The messages $\mathtt{m}_{J^i \to \mathsf{X}_k^i}(\mathcal{X}_k^i)$, for $i\in \mathcal{I}_{k-1}$, are given by
\begin{align}
     \mathtt{m}_{J^i \to \mathsf{X}_k^i}(\mathcal{X}_k^i) = & \int \sum_{j=0}^{J_k} \mathtt{n}_{\mathsf{c}_{k}^i \to J^i}(c_{k}^i=j) \mathtt{n}_{\mathsf{s}_{k} \to J^i}(\mathbf{s}_{k}) \notag \\
     &\times t(\mathcal{Z}_k^i|\mathbf{s}_k,\mathcal{X}_k^i,c_k^i=j)
    \mathrm{d}\mathbf{s}_{k}.
\end{align}
\item[\fbox{\footnotesize{20}}]
    Beliefs of the set-variables are denoted by $\mathcal{X}_k^i$, for $i\in \mathcal{I}_{k-1}$.
    The beliefs $\mathtt{b}(\mathcal{X}_k^i)$ are obtained by
\begin{align}
    \mathtt{b}(\mathcal{X}_k^i) \propto \mathtt{m}_{E^i \to \mathsf{X}_k^i}(\mathcal{X}_k^i)\mathtt{m}_{J^i \to \mathsf{X}_k^i}(\mathcal{X}_k^i),
\end{align}
    which follows the Bernoulli distribution.
\end{enumerate}

\paragraph{Undetected Targets}
    We compute 1 belief of the set-variable $\mathcal{X}_k^\mathrm{U}$ representing the targets that have never been detected and thus remain undetected again, obtained as follows.
\begin{enumerate}
\item[\fbox{\footnotesize{21}}]
    A message from the sensor state $\mathbf{s}_k$ to factor $H$, i.e., $f^H(\mathbf{s}_k,\mathcal{X}_k^\mathrm{U})=[1-\mathsf{p}_\mathrm{D}(\mathbf{s}_k,\mathcal{X}_k^\mathrm{U})]^{\mathcal{X}_k^\mathrm{U}}$, is denoted by $\mathtt{n}_{\mathsf{s}_k \to I}(\mathbf{s}_k)$, given by $\mathtt{n}_{\mathsf{s}_k \to H}(\mathbf{s}_k)=\mathtt{n}_{A \to \mathsf{s}_k}(\mathbf{s}_k)$.
    A message from factor $H$ 
    to the linked set-variable $\mathcal{X}_k^\mathrm{U}$ is denoted by $\mathtt{m}_{H \to \mathsf{X}_k^\mathrm{U}}(\mathcal{X}_k^\mathrm{U})$,
    given by
\begin{align}
    \mathtt{m}_{H \to \mathsf{X}_k^\mathrm{U}}(\mathcal{X}_k^\mathrm{U})
    = & \int \mathtt{n}_{\mathsf{s}_k \to H}(\mathbf{s}_k)
    [1-\mathsf{p}_\mathrm{D}(\mathbf{s}_k,\mathcal{X}_k^\mathrm{U})]^{\mathcal{X}_k^\mathrm{U}} \mathrm{d} \mathbf{s}_k.
\end{align}
\item[\fbox{\footnotesize{22}}]
    The belief of the set-variables $\mathcal{X}_k^\mathrm{U}$ is denoted by $\mathtt{b}(\mathcal{Y}_k^j)$, computed by
\begin{align}
    \mathtt{b}(\mathcal{X}_k^\mathrm{U}) \propto \mathtt{m}_{F \to \mathsf{X}_k^\mathrm{U}}(\mathcal{X}_k^\mathrm{U})
    \mathtt{m}_{H \to \mathsf{X}_k^\mathrm{U}}(\mathcal{X}_k^\mathrm{U}),
\end{align}
    which follows the Bernoulli distribution. 
\end{enumerate}

\paragraph{Sensor}
    We compute 1 belief of the sensor state $\mathbf{s}_k$ using the messages from the predicted sensor state, previously detected targets, and newly detected targets.
\begin{enumerate}
\item[\fbox{\footnotesize{23}}]
    Messages from factor $I^j$ 
    to the linked vector variable $\mathbf{s}_k$ are denoted by $\mathtt{m}_{I^j \to \mathsf{s}_k}(\mathbf{s}_k)$.
    The messages $\mathtt{m}_{I^j \to \mathsf{s}_k}(\mathbf{s}_k)$, for $j\in \mathcal{J}_{k}$, are given by
\begin{align}
    \mathtt{m}_{I^j \to \mathsf{s}_k}(\mathbf{s}_k)
    = & \int 
    \sum_{i=0}^{I_{k-1}} \mathtt{n}_{\mathsf{d}_{k}^j \to I^j}(d_{k}^j=i)
    \mathtt{n}_{\mathsf{Y}_k^j \to I^j}(\mathcal{Y}_k^j) \notag\\ 
    &\times \tilde{l}(\mathbf{z}_k^j|\mathbf{s}_k,\mathcal{Y}_k^j,d_k^j=i) \delta \mathcal{Y}_{k}^j.
\end{align}
\item[\fbox{\footnotesize{24}}]
    Messages from factors $J^i$
    to the linked vector variable $\mathbf{s}_k$ are denoted by $\mathtt{m}_{J^i \to \mathsf{s}_k}(\mathbf{s}_k)$.
    The messages $\mathtt{m}_{J^i \to \mathsf{s}_k}(\mathbf{s}_k)$, for $i\in \mathcal{I}_{k-1}$, are given by
\begin{align}
     \mathtt{m}_{J^i \to \mathsf{s}_k}(\mathbf{s}_k)=& \int \sum_{j=0}^{J_k} \mathtt{n}_{\mathsf{c}_{k}^i \to J^i}(c_{k}^i=j) \mathtt{n}_{\mathsf{X}_k^i \to J^i}(\mathcal{X}_k^i)\notag \\
     & \times t(\mathcal{Z}_k^i|\mathbf{s}_k,\mathcal{X}_k^i,c_k^i=j)
     \delta \mathcal{X}_{k}^i.
\end{align}
    
\item[\fbox{\footnotesize{25}}]
    A message from $\mathcal{X}_k^\mathrm{U}$ to factor $H$ is denoted by $\mathtt{n}_{\mathsf{X}_k^\mathrm{U} \to H}(\mathcal{X}_k^\mathrm{U})$, and $\mathtt{n}_{\mathsf{X}_k^\mathrm{U} \to H}(\mathcal{X}_k^\mathrm{U})=\mathtt{m}_{F \to \mathsf{X}_k^\mathrm{U}}(\mathcal{X}_k^\mathrm{U})$.
\item[\fbox{\footnotesize{26}}]
    A message from factor $H$ to the sensor state $\mathbf{s}_k$ is denoted by $\mathtt{m}_{H \to \mathsf{s}_k}(\mathbf{s}_k)$, computed by
\begin{align}
    \mathtt{m}_{H \to \mathsf{s}_k}(\mathbf{s}_k) 
    = \int \mathtt{n}_{\mathsf{X}_k^\mathrm{U} \to H}(\mathcal{X}_k^\mathrm{U}) [1-\mathsf{p}_\mathrm{D}(\mathbf{s}_k,\mathcal{X}_k^\mathrm{U})]^{ \mathcal{X}_k^\mathrm{U}}
    \delta \mathcal{X}_k^\mathrm{U}
\end{align}
\item[\fbox{\footnotesize{27}}]
    A belief of the sensor variable $\mathbf{s}_k$ is denoted by $\mathtt{b}(\mathbf{s}_k)$, computed by
\begin{align}
    \mathtt{b}(\mathbf{s}_k) \propto
    & \,
    \mathtt{m}_{A \to \mathsf{s}_k}(\mathbf{s}_k) 
    \mathtt{m}_{H \to \mathsf{s}_k}(\mathbf{s}_k)
    \notag
    \\
    & \times \prod_{j\in \mathcal{J}_{k}}
    \mathtt{m}_{I^j \to \mathsf{s}_k}(\mathbf{s}_k) 
    \prod_{i \in \mathcal{I}_{k-1}}
    \mathtt{m}_{J^i \to \mathsf{s}_k}(\mathbf{s}_k).
\end{align}
\end{enumerate}

    Finally, we obtain the  marginal posterior densities, $f_\mathtt{u}(\mathbf{s}_k)=\mathtt{b}(\mathbf{s}_k)$, $f_\mathtt{u}^\mathrm{U}({\mathcal{X}}_{k}^\mathrm{U})=\mathtt{b}(\mathcal{X}_k^\mathrm{U})$, $f_\mathtt{u}^i(\mathcal{X}_{k}^i)=\mathtt{b}(\mathcal{X}_k^i)$ for $i \in \mathcal{I}_{k-1}$, $f_\mathtt{u}^i(\mathcal{X}_{k}^{I_{k-1}+j})=\mathtt{b}(\mathcal{X}_k^j)$ for $j \in \mathcal{J}_k$.
\newpage

\end{appendices}
\end{document}